\documentclass{article}

\usepackage{arxiv}

\usepackage[utf8]{inputenc} 
\usepackage[T1]{fontenc}    
\usepackage{url}            
\usepackage{booktabs}       
\usepackage{amsfonts}       
\usepackage{nicefrac}       
\usepackage{microtype}      
\usepackage{xcolor}         

\usepackage{bm}
\usepackage{amsmath,amsthm,amssymb,amsfonts,amscd,mathrsfs,color}
\usepackage{verbatim}
\usepackage{booktabs}
\usepackage{multirow}
\usepackage{tikz,tikz-cd}
\usepackage{lineno}
\usepackage[colorinlistoftodos,textsize=tiny]{todonotes}

\newcommand{\mypara}[1]{\noindent\textbf{#1}}

\newtheorem{theorem}{Theorem}[section]

\newtheorem{corollary}[theorem]{Corollary}

\newtheorem{lemma}[theorem]{Lemma}
\newtheorem{proposition}[theorem]{Proposition}
\newtheorem*{proposition*}{Proposition}
\newtheorem*{theorem*}{Theorem}
\newtheorem*{conjecture*}{Conjecture}

\newtheorem{claim}[theorem]{Claim}

\theoremstyle{definition}
\newtheorem{definition}[theorem]{Definition}
\newtheorem{proposition-definition}[theorem]{Proposition-Definition}
\newtheorem{remark}[theorem]{Remark}

\newtheorem{example}[theorem]{Example}
\newtheorem{assumption}{Assumption}

\newtheorem{innercustomgeneric}{\customgenericname}
\providecommand{\customgenericname}{}
\newcommand{\newcustomtheorem}[2]{%
  \newenvironment{#1}[1]
  {%
   \renewcommand\customgenericname{#2}%
   \renewcommand\theinnercustomgeneric{##1}%
   \innercustomgeneric
  }
  {\endinnercustomgeneric}
}

\newcustomtheorem{customthm}{Theorem}
\newcustomtheorem{customassump}{Assumption}

\newlength{\widebarargwidth}
\newlength{\widebarargheight}
\newlength{\widebarargdepth}

\newcommand{\Prob}[1]{\mathbb{P}\left[#1\right]}

\newcommand{\real}{\ensuremath{\mathbb{R}}}

\usepackage{listings}
\lstdefinelanguage{Magma}%
  {%
   otherkeywords={:=,+:=,-:=,*:=},%
   procnamekeys={function,func,intrinsic,procedure,proc},%
   morekeywords={true},%
   morekeywords=[2]{adj,and,cat,cmpeq,cmpne,diff,div,eq,ge,gt,in,is,join,le,lt,%
          meet,mod,ne,notadj,notin,notsubset,or,sdiff,subset,xor},%
   morekeywords=[3]{assigned,break,by,case,catch,continue,declare,default,%
          delete,do,elif,else,end,eval,exists,exit,for,forall,fprintf,if,local,%
          not,print,printf,quit,random,read,readi,repeat,restore,save,select,%
          then,time,to,try,until,vprint,vprintf,vtime,when,where,while},%
   morekeywords=[4]{clear,forward,freeze,iload,import,load},%
   morekeywords=[5]{assert,assert2,assert3,error,require,requirege,requirerange},%
   morekeywords=[6]{car,comp,cop,elt,ext,frac,hom,ideal,iso,lideal,loc,map,%
          ncl,pmap,quo,rec,recformat,rep,rideal,sub},%
   morekeywords=[7]{AbelianGroup,AdditiveCode,AffineAlgebra,Algebra,%
          AssociativeAlgebra,Character,CliffordAlgebra,Design,Digraph,%
          ExtensionField,FPAlgebra,FiniteAffinePlane,FiniteProjectivePlane,%
          Graph,Group,GroupAlgebra,IncidenceStructure,LieAlgebra,LinearCode,%
          LinearSpace,MatrixAlgebra,MatrixGroup,MatrixRing,Monoid,%
          MultiDigraph,MultiGraph,NearLinearSpace,Network,PartialMap,%
          PermutationGroup,PolycyclicGroup,QuaternionAlgebra,Semigroup,%
          ZModule,%
          Rationals,Scheme,GroebnerBasis,Dimension,RationalPoints,HasPointsOverExtension,Degree},%
   morekeywords={[8]function,func,intrinsic,procedure,proc,return},%
      sensitive,%
      morecomment=[l]//,%
      morecomment=[s]{/*}{*/},%
      morecomment=[s]{\{}{\}},%
      morestring=[b]",%
    moredelim=**[is][\color{codegreen}]{@}{@},
  }[keywords,procnames,comments,strings]%

\definecolor{codegreen}{rgb}{0,0.6,0}
\definecolor{codegray}{rgb}{0.5,0.5,0.5}
\definecolor{codepurple}{rgb}{0.58,0,0.82}
\definecolor{backcolour}{rgb}{0.95,0.95,0.92}

\lstdefinestyle{mystyle}{
  backgroundcolor=\color{backcolour}, commentstyle=\color{codegreen},
  keywordstyle=\color{magenta},
  numberstyle=\tiny\color{codegray},
  stringstyle=\color{codepurple},
  basicstyle=\ttfamily\footnotesize,
  breakatwhitespace=false,         
  breaklines=true,                 
  captionpos=b,                    
  keepspaces=true,                 
  numbers=left,                    
  numbersep=5pt,                  
  showspaces=false,                
  showstringspaces=false,
  showtabs=false,                  
  tabsize=2
}

\title{On the Identifiability of Mixtures of Ranking Models}

\author{
Xiaomin Zhang \\
Department of Computer Sciences \\
University of Wisconsin-Madison \\
\texttt{xzhang682@wisc.edu}
\And
Xucheng Zhang \\
Fakult\"{a}t f\"{u}r Mathematik \\
Universit\"{a}t Duisburg-Essen \\
\texttt{xucheng.zhang@stud.uni-due.de} 
\And
Po-Ling Loh \\
Department of Pure Mathematics and Mathematical Statistics \\ 
University of Cambridge \\
\texttt{pll28@cam.ac.uk}
\And
Yingyu Liang \\
Department of Computer Sciences \\
University of Wisconsin-Madison \\
\texttt{yliang@cs.wisc.edu} }

\begin{document}

\maketitle

\begin{abstract}
Mixtures of ranking models are standard tools for ranking problems. However, even the fundamental question of parameter identifiability is not fully understood: the identifiability of a mixture model with two Bradley-Terry-Luce (BTL) components has remained open. In this work, we show that popular mixtures of ranking models with two components (BTL, multinomial logistic models with slates of size 3, or Plackett-Luce) are generically identifiable, i.e., the ground-truth parameters can be identified except when they are from a pathological subset of measure zero. We provide a framework for verifying the number of solutions in a general family of polynomial systems using algebraic geometry, and apply it to these mixtures of ranking models to establish generic identifiability. The framework can be applied more broadly to other learning models and may be of independent interest.
\end{abstract}

\section{Introduction}\label{sec:intro}

Ranking is an important topic in machine learning, where the goal is to rank collections of items based on votes involving smaller subsets of items. Three categories are commonly studied by researchers: pointwise ranking \cite{balabanovic1997fab, resnick1994grouplens, yu2009fast}, pairwise ranking \cite{jamieson2011active,tsukida2011analyze,wauthier2013efficient,chen2013pairwise,chen2015spectral}, and listwise ranking \cite{pendergrass1959ranking, khetan2016data, zhao2018composite}. Pointwise ranking collects data based on user ratings; pairwise ranking is based on comparisons between pairs of items; and listwise ranking is based on ordered lists of items. The latter two are preferable due to  better user convenience and improved consistency. Although pointwise rating data are popular in areas such as advertising, recommendation systems, and player rankings~\cite{wauthier2013efficient}, ratings can be inconsistent even when user preferences are consistent: For example, User $A$ might rate movies $i, j$, and $k$ as 1, 4, and 7 out of $10$, respectively, while User $B$ might rate the same movies as 7, 8, and 9.  

In practice, ranking data are typically noisy and may possess inconsistencies. For example, the data may include votes---even from the same user---which involve ``cycles,'' such as $i_1 \succ i_2$ (the notation $\succ$ means the $i_1^{\text{th}}$ item is preferred to the $i_2^{\text{th}}$ item), $i_2 \succ i_3$, and $i_3 \succ i_1$. Various probabilistic ranking models incorporate such inconsistencies as statistical noise. For example, pairwise ranking models may assume a parametric model $\mathbb{P}_{\theta}(i,j)$, which specifies the probability of obtaining a user vote $i \succ j$ for the pair of items $(i,j)$.

Most studies focus on inferring a single parametric model shared by all users, although a \emph{mixture} of ranking models may better fit the ranking data. Imagine a case where users rate movies: Adults may have different preferences from children, and women may have different preferences from men. Therefore, a collection of orderings might better explain user preferences. Single ranking models may cater to a large portion of users or describe an accurate ordering for a large portion of objects, but low accuracy in predicting new votes may hint that multiple user types are actually present in the population. In such cases, using a mixture model to fit the data can lead to higher prediction accuracy.

While mixtures of ranking models are popular tools, even the fundamental question of identifiability is not fully understood for some models such as Bradley-Terry-Luce~\cite{bradley1952rank, luce1959}. Here, identifiability means whether the parameter can be determined by the specified probabilities (e.g., determining $\theta$ from the $\mathbb{P}_{\theta}(i,j)$'s), even ignoring the statistical and optimization challenges. (See Section~\ref{sec:background} for the precise definition of identifiability.) Our work focuses on this identifiability problem. 
 
\subsection{Our results and contributions}
We affirm the identifiability of three popular mixtures of ranking models with two components in all but a subset of measure zero, which we call ``generic identifiability" (cf.\ Section~\ref{SecGenIdent} below). These include \textbf{B}radley-\textbf{T}erry-\textbf{L}uce (\textbf{BTL}) models, \textbf{m}ulti\textbf{n}omial \textbf{l}ogistic (\textbf{MNL}) models with slates of size 3, and \textbf{P}lackett-\textbf{L}uce (\textbf{PL}) models. In particular, we address the aforementioned open question for mixtures of BTL models:

\begin{theorem*}[Informal]
A mixture of BTL models with two components ranking at least five items is generically identifiable (up to reordering). Similarly, generic identifiability holds for a mixture of 3-slate MNL models with two components ranking at least four items, and for a mixture of PL models with two components ranking at least four items.
\end{theorem*}

The proof uses a general framework developed via algebraic geometry:
\begin{theorem*}[Informal]
A polynomial system with variables $\bm{x}$ and parameters $\bm{a}$ has exactly $\ell$ solutions in $\bm{x}$ for all $\bm{a}$ in the parameter domain, except a set of Lebesgue measure zero, assuming:
\begin{customassump}{1}
{\it The polynomial system has at least $\ell$ solutions for almost all $\bm{a}$ in the domain set.}
\end{customassump}
\begin{customassump}{2}
{\it There exists some $\bm{a'}$ in the domain such that $\bm{a'}$ does not nullify any coefficients in the Gr\"obner basis of the polynomial system, and the polynomial system with parameters $\bm{a}'$ has exactly $\ell$ solutions.}
\end{customassump}
\end{theorem*}

In other words, the claim that mixtures of ranking models are identifiable can be deduced by showing that an appropriate polynomial system has a unique solution. Using this general theorem, it suffices to check the two aforementioned assumptions: Assumption 1 is naturally satisfied when realizability of the model is assumed for the data. Assumption 2 is only about solutions at one ``not-too-bad'' parameter, which can be readily verified by proposing a concrete instance manually or using mathematical software. For example, in our analysis for mixtures of ranking models, 
the software Magma~\cite{MR1484478} is used to check Assumption 2. 
More broadly, the theorem gives conditions under which generic identifiability holds for machine learning problems whose models can be transformed into polynomial systems.

In summary, our work has two main contributions:

\begin{enumerate}

\item
We propose a general framework to check the generic identifiability of multiple ranking models. 
Previous work typically focused on one specific model at a time and used specific analysis tools. In contrast, our technique is fairly general, based on a theorem about the number of solutions in a family of polynomial equations using algebraic geometry. 

\item  
We prove generic identifiability of mixtures of ranking models with two components. 
We apply the general framework described above to different ranking models, including BTL models (pairwise comparisons), MNL models with 3-slates (triplet comparisons), and PL models (list orderings). Notably, we prove the generic identifiability of mixtures of BTL models, which had remained an open question despite long-standing interest. Our unified framework also provides alternative proofs for MNL models with 3-slate and PL models, which had been considered in previous work.
\end{enumerate}


\subsection{Literature review and comparisons}\label{Literature-review}
 

\paragraph{Mixture of ranking models.}
Mixtures of Gaussians and mixtures of linear regressions are well-studied~\cite{mclachlan1988mixture,reynolds2009gaussian,yi2014alternating,li2018learning}. 
A popular family of ranking models is the multinomial logistic (MNL) model, also known as the mixed logit model, which includes BTL models (equivalent to MNL models with 2-slate) and MNL models with 3-slate studied in our work. There is a rich literature on mixtures of MNL models: Train~\cite{train2009discrete} introduced the mixed logit model as a method for discrete choices; Ge~\cite{ge2008bayesian} discussed simulation approaches for mixed logit models using the Metropolis-Hastings algorithm; and Arora et al.~\cite{arora2013practical} studied applications of mixed logit models to topic modeling. 

We focus on the identifiability problem of mixtures of ranking models. Our results assume no detailed user type information for votes, unlike previous authors, who assumed access to extra information at the user level. For example, Wu et al.~\cite{wu2015clustering} studied mixtures of BTL models when pairwise comparisons for all pairs are collected from each user. 
Oh and Shah~\cite{oh2014learning} studied mixtures of MNL models, but demanded more information at the user  level, such as collecting a few comparisons that are known to be from the same type of user.
Chierichetti et al.~\cite{chierichetti2018learning} analyzed the identifiability of uniform mixtures of MNL models with both pairwise and triplet comparisons (i.e., 2-slate and 3-slate), while Tang~\cite{tang2020learning} extended the results to mixtures with general mixing probabilities. In contrast, our study of mixtures of BTL models only has pairwise comparisons, while our study of mixtures of MNL models with 3-slate only has triplet comparisons, in which case it is more challenging to affirm identifiability.

Besides mixtures of BTL and MNL models, other mixtures have also been studied. Sturmfels and Welker~\cite{MR2921622} and Zhao et al.~\cite{zhao2016learning} addressed the identifiability of mixtures of Plackett-Luce models. Innario~\cite{iannario2010identifiability} showed the identifiability of mixtures of shifted binomial and uniform discrete models. Awasthi et al.~\cite{awasthi2014learning} proved the uniqueness of Mallows mixture models of two components; Liu and Moitra~\cite{liu2018efficiently} extended the results to any constant number of components; and Chierichetti et al.~\cite{chierichetti2015learning} considered arbitrary numbers of components, but with conditions such as separation between the components. Lu and Boutilier~\cite{lu2014effective} and Mao and Wu~\cite{chengandwu} also studied mixtures of Mallows models, but used groups of pairwise comparisons from the same user.

Finally, we emphasize that our result proving the generic identifiability for mixtures of BTL models is one of our main contributions, since this problem has remained largely open despite the popularity of the model. Indeed, the techniques proposed in the aforementioned studies generally use extra information not available in mixtures of BTL models. (However, besides identifiability, other authors typically also consider the learning problem of whether the parameters can be identified in polynomial time with polynomial samples.) Below, we provide additional details about why earlier techniques are inapplicable to the identifiability of mixtures of BTL models:
  
\begin{enumerate}

\item
Our setting involving mixtures of BTL models has access to less information than in Chierichetti et al.~\cite{chierichetti2018learning} or Tang~\cite{tang2020learning}. 
Although the BTL model is equivalent to an MNL model with 2-slate, Chierichetti et al.~\cite{chierichetti2018learning} and Tang~\cite{tang2020learning} assume access to both 2-slates and 3-slates. (With both 2-slates and 3-slates, our general framework can also establish generic identifiability; see Section~\ref{apn:mnl23}.)

From a technical standpoint, the proof techniques used by previous authors \cite{chierichetti2018learning,tang2020learning} transforms the identifiability problem into finding common roots of univariate quartic equations. Manipulating these equations allows one to apply linear algebraic arguments (to be precise, the resultant of two polynomials) to solve the equations. However, the limited equations in the BTL model do not give rise to the same linearity. 

\item Oh and Shah~\cite{oh2014learning}, in MNL models, assume access to multiple comparisons known to be from the same mixture component, and require additional conditions for successful identification. As pointed out by Chierichetti et al.~\cite{chierichetti2018learning} (page 3, first paragraph in the right column), the data format or ``information oracle" used by previous authors \cite{oh2014learning,zhao2016learning,liu2018efficiently} is stronger than the one usually assumed in BTL models. Specifically, for each observation, Shah and Oh~\cite{oh2014learning} first sample one mixture component and then sample multiple pairwise comparisons from the same component. Multiple comparisons are needed in each observation to build the tensors $M_2$ and $M_3$ in equations (4) \& (5) of their algorithm, even with infinite data. This is not available in the BTL model. Furthermore, conditions C1-C3 are needed for successful parameter identification, so their results do not affirm the generic identifiability of MNL models (including BTL models).

\item Zhao et al.~\cite{zhao2016learning} study mixtures of PL models, while other authors \cite{awasthi2014learning,chierichetti2015learning,liu2018efficiently,lu2014effective,chengandwu} study mixtures of Mallows models. These models involve list orderings, and thus also have stronger information oracles than BTL models. 
Specifically, a list ordering is given in the form of a permutation $(i_1, i_2, \dots, i_n)$, meaning $(i_1 \succ i_2) \wedge (i_2 \succ i_3) \wedge \cdots \wedge (i_{n-1} \succ i_n)$. Note that all these $n-1$ pairwise preferences come from the same user (and thus the same mixture component). In BTL models, however, two pairwise comparisons $i_1 \succ i_2$ and $i_3 \succ i_4$ may not even be from the same component. 

Technically, the analysis in Zhao et al.~\cite{zhao2016learning} partitions the items into three groups and constructs a tensor as the sum of $k$ rank-one tensors for each group. A key assumption is access to at least three pairwise comparisons from the same user type. However, BTL models only have access to independent pairwise comparisons, so to use the tensor decomposition technique, we would need access to more information about the relationship between comparisons. Other approaches \cite{awasthi2014learning,chierichetti2015learning,liu2018efficiently} also require list-ordered data. Lu and Boutilier~\cite{lu2014effective} and Mao and Wu~\cite{chengandwu} consider mixtures of Mallows models using pairwise comparisons; however, their techniques rely on groups of pairwise comparisons from the same mixture component, which are not available in mixtures of BTL models. 


\end{enumerate}
Generally speaking, stronger information oracles used in the aforementioned literature allow one to employ certain linearization procedures to turn the identifiability problem into a linear algebraic problem (e.g., tensor decomposition or resultants). Such linearization may not always be expected in general (in particular, for mixtures of BTL models). New tools such as algebraic geometry are thus required.  

\paragraph{Algebraic geometry for identifiability.}
Tools from algebraic geometry have already been deployed in machine learning. In particular, mixture models and their identifiability have been studied algebraically \cite{MR2921622,MR3300329,MR3873537,MR4269649}, forming a vibrant topic in the field of algebraic statistics \cite{MR2723140,MR3838364}. 
This literature has generally focused on geometric descriptions (and algebraic characterizations) of specific models in the language of algebraic geometry (and commutative algebra)---achieving identifiability as a by-product. In particular, the geometric description presented in Sturmfels and Welker~\cite[Theorem 7.1]{MR2921622} implies, among other things, that the PL model is generically identifiability, and also provides a way to read off the non-identifiable locus explicitly. 


While the entire geometric information given in the above literature furnishes identifiability results, it applies only to specific models.  
In this study, we provide a different and more direct approach to obtain identifiability of more general models. Technically, we use the semi-continuity of Hilbert polynomials in a family of projective varieties, which is a standard result from algebraic geometry and also a natural candidate for proving identifiability. However, these tools cannot be applied directly, and some technical novelty is required. 
This is because the identifiability problem only concerns solutions away from infinity. In terms of algebraic geometry, this means only solutions in the \textit{affine} space (which does not include infinity) are counted, while the semi-continuity tool counts solutions in the \textit{projective} space (which includes infinity) and can lead to incorrect conclusions for the affine case. For example, the polynomial $(1-tx)x=0$, viewed as a family of polynomials in $x$ parametrized by $t$, has a unique solution at $t=0$, but this does not hold even for any other $t$. To fill this gap, we note that the choice of such ``bad'' parameters is rare and can also be avoided by computing the Gr\"obner basis of the polynomial system in question. In the $(1-tx)x=0$ example, any parameter $t \neq 0$ is not ``bad'' and yields exactly two solutions, which does hold generically. As far as we know, this is a novel technical contribution.


 


\section{Background}\label{sec:background}

Here, we formally define generic identifiability and review the mixtures of BTL ranking model. We review the mixtures of MNL and mixtures of PL ranking models in Appendix~\ref{apn:background}.

\paragraph{Notation.} We write $\lambda_m^{\mathbb{C}}$ to denote the Lebesgue measure on $\mathbb{C}^m$, i.e., the measure induced by the standard Euclidean metric on $\mathbb{C}^m$ given by
$\|z\|:=\left(\sum_{i=1}^m |z_i|^2\right)^{1/2}$, for any $z \in \mathbb{C}^m$,
and write $\lambda_m$ to represent the Lebesgue measure on $\mathbb{R}^m$, i.e., the measure induced by the standard Euclidean metric on $\mathbb{R}^m$ given by
$\|z\|:=\left(\sum_{i=1}^m z_i^2\right)^{1/2}$, for any $z \in \mathbb{R}^m$.
We identify $\mathbb{R}^m$ as a subset of $\mathbb{C}^m$ via $\mathbb{R}^m \cong \mathbb{R}^m \times \{\bm{0}_m\} \subseteq \mathbb{R}^{2m} \cong \mathbb{C}^m$, and define the  set zero-set$(f(z)):=\{z: f(z)=0\}$ for any polynomial $f$.
For a vector $v$, we use $v_{i:j}$ to denote the sub-vector $(v_i, v_{i+1}, \ldots, v_j)$.


\subsection{Generic identifiability}
\label{SecGenIdent}

Recall that a probabilistic ranking model (or a mixture of ranking models) is a model $\mathbb{P}_{\theta}$ that specifies the probability of rankings, where the variable $\theta$ is from some set $\Theta$.
Using pairwise rankings as an example, for $\theta^* \in \Theta$ and any pair $(i,j)$, the function $\eta_{i,j}(\theta^*) := \mathbb{P}_{\theta^*}(i, j)$ specifies the probability of observing $i \succ j$ votes from users.\footnote{In machine learning, $\theta$ usually denotes the parameter of the model. In this paper, to be consistent with terminology for polynomial equation systems, we call $\theta$ the \emph{variable} and call $\theta^*$ the parameter.} Given the probability values $\{\eta_{i,j}(\theta^*)\}$, the identifiability problem is whether the values determine the variable of the model, i.e., whether $\theta^*$ is the unique solution for $\theta$ in the equation system $\eta_{i,j}(\theta^*) = \mathbb{P}_{\theta}(i, j) (\forall i \neq j)$. If this is true, we say  $\mathbb{P}_{\theta}$ is \emph{identifiable} at $\theta^*$. If identifiability holds for all $\theta^*$ in the set $\Theta$, we say $\mathbb{P}_{\theta}$ is identifiable on $\Theta$. Here, we ignore the statistical and computational challenges and focus on identifiability: we assume access to the probability values $\eta_{i,j}(\theta^*)$ rather than data samples (equivalently, access to infinite samples)---we consider existence and uniqueness of the solution without regard to computation.

The notion of identifiability can be generalized to equation systems. Note that $\eta_{i,j}(\theta^*) = \mathbb{P}_{\theta}(i, j) (\forall i \neq j)$ is an equation system in the variable $\theta$, with parameter $\theta^* \in \Theta$. 
More generally, consider an equation system $F(\bm{z};\bm{c})= \bm{0}$ on variables $\bm{z}$ from a variable domain set $Z$ with parameters $\bm{c}$ from a parameter domain set $C$. 
\begin{definition}[Identifiability]
For a fixed parameter $\bm{c} \in C$, we call $F(\bm{z};\bm{c})= \bm{0}$ \emph{identifiable} at $\bm{c}$ with respect to $\bm{z}$ if $F(\bm{z}; \bm{c}) = \bm{0}$ has a unique solution for $\bm{z} \in Z$. 
That is, there exists $\bm{z}^*$ satisfying $F(\bm{z}^*; \bm{c}) = \bm{0}$, but no $\bm{z}^{\#} \neq \bm{z}^*$ satisfying $F(\bm{z}^{\#}; \bm{c}) = \bm{0}$.  
\end{definition}
Where there is no ambiguity, we write $F(\bm{z}^*; \bm{c})$ to represent the equation system $F(\bm{z}^*; \bm{c}) = \bm{0}$.

As suggested by Oh and Shah~\cite{oh2014learning} and Chierichetti et al.~\cite{chierichetti2018learning}, examples exist showing that mixtures of rankings may not be identifiable. In Appendix~\ref{apn:non-identifiable-example}, we provide an example which produces infinitely many parameter choices leading to non-identifiability.
Nonetheless, as we will show, such choices of the parameters constitute pathological special cases which might be assumed to be avoided in practice.

\begin{definition}[Generic identifiability]
The \emph{non-identifiable (bad) parameter set} of an equation system $F(\bm{z};\bm{c})$ is the subset of parameters $\bm{c} \in C$ at which $F(\bm{z};\bm{c})= \bm{0}$ is not identifiable.
We call an equation system $F(\bm{z};\bm{c})= \bm{0}$ \emph{generically identifiable} on $C$ if the Lebesgue measure of $C$ is positive and the Lebesgue measure of the non-identifiable parameter set is zero.
\end{definition}

\subsection{The Bradley-Terry-Luce model} \label{sec:btl_model}

\mypara{Probabilistic model.}
The BTL model was introduced by Bradley and Terry~\cite{bradley1952rank} and studied by Luce~\cite{luce1959}. 
Observed pairwise comparisons on pairs of items $(i,j)$ follow the probabilistic model
\begin{align} \label{eqn:btl}
\Prob{i \succ j} = \frac{c_i}{c_i+c_j},
\end{align}
where $c_i$ and $c_j$ are the ranking scores/weights of $i$ and $j$, respectively.

Turning to mixtures of BTL models, suppose we have two user types and $n$ items. Let $U$ be an indicator variable denoting the user type, i.e., $U=1$ or $U =2$. Let $\bm{a}:= (a_1, a_2, \dots, a_n)^\top$ denote the ranking scores for Type 1 users and $\bm{b}=(b_1, b_2, \dots, b_n)^\top$ denote the scores for Type 2 users. We will use $\bm{a}_{i:j}$ and $\bm{b}_{i:j}$ to represent $(a_i, \dots, a_j)$ and $(b_i,\dots,b_j)$, respectively.
Applying BTL on each type, we have the conditional probabilities
$ 
\Prob{i \succ j | U = 1} = \frac{a_i}{a_i + a_j}$ and $\Prob{i \succ j | U = 2} = \frac{b_i}{b_i + b_j}.
$ 
Suppose $U$ follows the Bernoulli distribution $Bernoulli(p_1, \{1,2\})$. Let $p_2 = 1 - p_1$. The observed comparison then follows the distribution 
\begin{equation}\label{eq:etadef}
\begin{split}
\eta_{i,j}(\bm{a}, \bm{b}, p_1) := \Prob{i \succ j} &= \sum_{k=1,2} p_i \Prob{i \succ j | U = k}  =  p_1 \frac{a_i}{a_i + a_j} + p_2 \frac{b_i}{b_i + b_j}.
\end{split}
\end{equation}
We use $\bm{\eta}(\bm{a},\bm{b}, p_1)$ to represent the $n\choose{2}$-dimensional vector with components $\{\eta_{i,j}(\bm{a},\bm{b}, p_1)\}_{i < j}$. When $p_1$ and $p_2$ are assumed to be known, we simplify the notation to $\bm{\eta}(\bm{a},\bm{b})$. When it is clear from the context, we simply write $\bm{\eta}$ or $\eta_{i,j}$.

\mypara{Identifiability.}
We wish to understand when the parameters $(\bm{a},\bm{b})$ can be recovered based on $\bm{\eta}$.
From equation~\eqref{eq:etadef}, we know that scaling the parameters does not change the distribution.
So we can multiply by $1/a_1$ to the original $\bm{a}$ so that $a_1 = 1$. Similarly, we can scale $\bm{b}$ to obtain $b_1 = 1$. 

First consider the case when the mixing probabilities $(p_1,p_2)$ are unknown. We wish to solve the following equation system in the variables $\bm{x}:=x_{1:n}$, $\bm{y}:=y_{1:n}$, and $p$: 
\begin{equation}\label{ex:BTL_2mixtures_pairwise_p}
\begin{cases}
x_1 = y_1 = 1, \\
p \dfrac{x_i}{x_i + x_j} + (1-p) \dfrac{y_i}{y_i + y_j} = \eta_{i,j}, & \forall i < j \in [n].
\end{cases}
\end{equation}
Let $Q^{2n-2}_{BTL} := \mathbb{R}_+^{2n-2}$ be the domain of $(\bm{a}_{2:n},\bm{b}_{2:n})$, where $\mathbb{R}_+$ denotes the positive real numbers. Let $\widetilde{Q}^{2n-2}_{BTL} = \{(\bm{a},\bm{b}): (\bm{a}_{2:n},\bm{b}_{2:n}) \in Q^{2n-2}_{BTL}, a_1 = b_1 = 1\}$ denote the domain of $(\bm{a},\bm{b})$. 
Let $Q^{2n-1}_{BTL,p} := Q^{2n-2}_{BTL} \times (0,1)$ and $\widetilde{Q}^{2n-1}_{BTL,p} := \widetilde{Q}^{2n-2}_{BTL} \times (0,1)$ denote the domains of $(\bm{a}_{2:n},\bm{b}_{2:n}, p_1) $ and $(\bm{a},\bm{b}, p_1)$, respectively. 
Then the non-identifiable parameter set is
\begin{align}
\begin{split}
    N_{BTL,p}^{2n-1} =& \Big\{(\bm{a}_{2:n},\bm{b}_{2:n},p_1) \in Q^{2n-1}_{BTL,p}:
    \exists (\bm{a}^{\#},\bm{b}^{\#},p^{\#}) \in \widetilde{Q}^{2n-1}_{BTL,p}, \text{ s.t.} \\
    & (\bm{a}^{\#},\bm{b}^{\#},p^{\#}) \neq (\bm{a},\bm{b},p_1)  \wedge   \bm{\eta}(\bm{a}^{\#},\bm{b}^{\#}, p^{\#}) = \bm{\eta}(\bm{a},\bm{b}, p_1)    \Big\}.
\end{split}
\end{align}
Note that for the parameter set, the Lebesgue measure is considered in a $(2n-1)$-dimensional space. 

Next, consider the case when the mixing probabilities $(p_1, p_2)$ are known. We wish to solve the following equation system in the variables $(\bm{x},\bm{y})$:
\begin{equation}\label{ex:BTL_2mixtures_pairwise}
\begin{cases}
x_1 = y_1 = 1, \\
p_1 \dfrac{x_i}{x_i + x_j} + p_2 \dfrac{y_i}{y_i + y_j} = \eta_{i,j}, & \forall i < j \in [n].
\end{cases}
\end{equation}
Then the non-identifiable parameter set is 
\begin{align}
\begin{split}
    N_{BTL}^{2n-2} = &\left\{(\bm{a}_{2:n},\bm{b}_{2:n}) \in Q^{2n-2}_{BTL}: \exists (\bm{a}^{\#},\bm{b}^{\#}) \in Q^{2n-2}_{BTL}, \text{ s.t.} \right. \\
    & \left. (\bm{a}^{\#},\bm{b}^{\#}) \neq (\bm{a},\bm{b}) \wedge  \bm{\eta}(\bm{a}^{\#},\bm{b}^{\#}) = \bm{\eta}(\bm{a},\bm{b}) \right\}.
\end{split}
\end{align}
Here, for the parameter set, the Lebesgue measure is considered in a $(2n-2)$-dimensional space. 

Thus, we say that this mixture of BTL models with unknown (respectively, known) mixing probabilities is generically identifiable if $N_{BTL,p}^{2n-1}$ (respectively, $N_{BTL}^{2n-2}$) has measure zero (with respect to the corresponding Lebesgue measure).

\section{A framework for generic identifiability of polynomial systems}
\label{sect:main-result}

In this section, we state our main results on generic identifiability of polynomial systems. We will consider two different cases: the complex case, where the parameters lie in $\mathbb{C}$, and the real case, where the parameters lie in subsets of $\mathbb{R}$ with positive (or infinite) Lebesgue measure. Since $\mathbb{C}$ is algebraically closed, we can directly apply certain results from algebraic geometry in the first case, whereas the latter case requires more careful proofs. Relevant terminology from algebraic geometry is reviewed in Appendix~\ref{apn:alg-geo-prelim}.

\subsection{Complex case}

Let $\mathscr{P}(\bm{t},\bm{x}) \subseteq \mathbb{C}[\bm{t}][\bm{x}]=\mathbb{C}[\bm{t},\bm{x}]$ be a set of polynomial equations in the variables $\bm{x}=(x_1,\ldots,x_n)$, with coefficients given by polynomials in the parameters $\bm{t}=(t_1,\ldots,t_m)$, i.e., each element of $\mathscr{P}(\bm{t},\bm{x})$ is of the form $f(\bm{t},\bm{x})=\sum_{\bm{e} \in \mathbb{N}^n} f_{\bm{e}}(\bm{t}) \bm{x}^{\bm{e}} = 0$, where $ \bm{x}^{\bm{e}} = \prod_{i=1}^n x_i^{e_i}$ and $f_{\bm{e}}(\bm{t})$ is a polynomial in $\bm{t}$ with coefficients in $\mathbb{C}$. 

Let $\succ$ be a block order on $\mathbb{C}[\bm{t},\bm{x}]$ such that $\bm{x} \succ \bm{t}$ (e.g., the inverse lexicographic order\footnote{More details on lexicographic order can be found in Cox et al.~\cite{MR3330490}.}), i.e.,
\[
\bm{t}^{\bm{d}_1}\bm{x}^{\bm{e}_1} \succ \bm{t}^{\bm{d}_2}\bm{x}^{\bm{e}_2} \Leftrightarrow \bm{x}^{\bm{e}_1} \succ_{\bm{x}} \bm{x}^{\bm{e}_2} \text{ or } (\bm{x}^{\bm{e}_1} = \bm{x}^{\bm{e}_2} \text{ and } \bm{t}^{\bm{d}_1} \succ_{\bm{t}} \bm{t}^{\bm{d}_2}),
\]
where $\succ_{\bm{x}}$ and $\succ_{\bm{t}}$ are arbitrary orders on $\mathbb{C}[\bm{x}]$ and $\mathbb{C}[\bm{t}]$, respectively. Let $I(\bm{t},\bm{x}) \subseteq \mathbb{C}[\bm{t},\bm{x}]$ be the ideal generated by the polynomials $f(\bm{t},\bm{x})$ in $\mathscr{P}(\bm{t},\bm{x})$ and let
$G(\bm{t},\bm{x})=\{g_1(\bm{t},\bm{x}),\ldots,g_s(\bm{t},\bm{x})\}$
be a Gr\"{o}bner basis\footnote{More details on Gr\"{o}bner bases can be found in Chapter 2 of Cox et al.~\cite{MR3330490}.} of $I(\bm{t},\bm{x})$ with respect to the block order $\succ$. We regard each $g_i(\bm{t},\bm{x})$ as a polynomial in $\bm{x}$ with coefficients which are functions of $\bm{t}$. Let $\mathrm{Bad}(\bm{t})=\{h_1(\bm{t}),\ldots,h_r(\bm{t})\}$ be the set of non-zero polynomials in $\bm{t}$ appearing as coefficients of some $g_i(\bm{t},\bm{x})$. Then
\begin{equation}\label{defn:Z(t)}
Z(\bm{t}):=\bigcup_{i=1}^r \text{zero-set}(h_i(\bm{t})) \subseteq \mathbb{C}^m
\end{equation}
is a \emph{Zariski closed subset}, which has $\lambda_m^{\mathbb{C}}$-measure zero by Lemma \ref{closed-0}. In most cases, this subset is efficiently computable using software such as Magma.
\begin{example}
If we have a set of polynomial equations,
$$
\mathscr{P}(\bm{t},\bm{x})=\{x_1x_2-2=0,tx_1x_2+x_1-1=0\},
$$
then one of its Gr\"{o}bner bases with respect to the block order $\succ$ is $$G(\bm{t},\bm{x})=\{x_1+(2t-1),(2t-1)x_2+2\}.$$ Hence,
$\mathrm{Bad}(\bm{t})=\{2t-1\}$ and $\  Z(\bm{t})=\{1/2\}$.
\end{example}

We now introduce two assumptions:
\begin{assumption}\label{assum:key-assumption1}
The system $\mathscr{P}(\bm{a},\bm{x})$ has at least $\ell$ solutions in $\mathbb{C}$ for all $\bm{a} \in \mathbb{C}^m \setminus E$, where $E \subseteq \mathbb{C}^m$ is a Zariski closed subset (hence of $\lambda_m^{\mathbb{C}}$-measure zero).
\end{assumption}
\begin{assumption}\label{assum:key-assumption2}
There exists $\bm{a'} \in \mathbb{C}^m \setminus Z(\bm{t})$ such that $\mathscr{P}(\bm{a'},\bm{x})$ has exactly $\ell$ solutions in $\mathbb{C}$ (counted with multiplicity), where $Z(\bm{t}) \subseteq \mathbb{C}^m$ is the $\lambda_m^{\mathbb{C}}$-measure zero subset defined in equation~\eqref{defn:Z(t)}.
\end{assumption}
\begin{theorem}\label{main-result}
Under Assumptions \ref{assum:key-assumption1} and \ref{assum:key-assumption2}, $\mathscr{P}(\bm{a},\bm{x})$ has exactly $\ell$ solutions in $\mathbb{C}$ (counted with multiplicity) for all $\bm{a} \in \mathbb{C}^m$ but a set of $\lambda_m^{\mathbb{C}}$-measure zero.
\end{theorem}

Usually, Assumption \ref{assum:key-assumption1} is inherent (and hence guaranteed) from the polynomial system itself, while Assumption \ref{assum:key-assumption2} can be checked using software (e.g., Magma), at least when the scale of the polynomial system is not too large. Indeed, we see from the proof of Theorem \ref{main-result} (cf.\ Section~\ref{apn:main-results-complex}) that
$\mathscr{P}(\bm{a'},\bm{x})$ has exactly $\ell$ solutions in $\mathbb{C}$ (counted with multiplicity) if and only if some dimension and degree condition is true. If $\mathscr{P}(\bm{a'},\bm{x})$ has $\mathbb{Q}$-coefficients, this condition can be checked using commands in Magma~\cite{MR1484478}.

\subsection{Real case}

We now present a variant of Theorem~\ref{main-result} for the real case. Note that the arguments in the complex case cannot be applied directly, because $\mathbb{R}$ is not algebraically closed.

\begin{theorem}\label{thm:general-case}
Let $K \subseteq \mathbb{R}^m$ be a subset of positive (or infinite) $\lambda_m$-measure. Under Assumptions~\ref{assum:key-assumption1} and~\ref{assum:key-assumption2}, $\mathscr{P}(\bm{a},\bm{x})$ has exactly $\ell$ solutions in $\mathbb{C}$ (counted with multiplicity) for all $\bm{a} \in K$ but a set of $\lambda_m$-measure zero.
\end{theorem}

\begin{remark}
Assumptions~\ref{assum:key-assumption1} and~\ref{assum:key-assumption2} are independent of the specific parameter space $K$ in Theorem~\ref{thm:general-case}. In particular, to satisfy Assumption~\ref{assum:key-assumption2}, the parameter $\bm{a}'$ need not be chosen in our parameter space $K$, resulting in much freedom for practical applications.
\end{remark}

\section{Generic identifiability of mixtures of ranking models}
\label{sec:examples}

In this section, we apply the general framework to different models to obtain the results summarized in Table~\ref{tab:summary-table}\footnote{The table does not include results in some references that are outside the scope of our work.}. We present the results for mixtures of BTL models, and present the results for MNL and PL models, as well as all proofs, in Appendix~\ref{apn:examples_proofs}. 

\begin{table*}[!thp]
\caption{Summary of results}
\label{tab:summary-table}
{\footnotesize
\begin{center}
\begin{tabular}{c|c|c|c|c}
\toprule
\shortstack{\textbf{Mixtures with} \\ \textbf{ $k$ components}} & \textbf{Data format} & \shortstack{\textbf{Mixing} \\ \textbf{probabilities}} & \textbf{Previous work} & \textbf{Our work} \\
\midrule
\multirow{2}{*}{BTL model} & \multirow{2}{*}{pairwise} & unknown &  & $\surd$ ($k=2$)\\\cline{3-5}
& & known &  & $\surd$ ($k=2$)\\ 
\hline
\multirow{3}{*}{\shortstack{MNL model \\ with 2\&3-slate}}  &  \multirow{3}{*}{\shortstack{pairwise + \\ triplet}}  & unknown &  & $\surd$ ($k=2$)\\\cline{3-5}
& & known & \shortstack{\cite{chierichetti2018learning} ($k=2$, uniform mixture) \\ \cite{tang2020learning} ($k=2$) } & $\surd$ ($k=2$)\\
\hline
\multirow{2}{*}{\shortstack{MNL model \\ with 3-slate}} & \multirow{2}{*}{triplet} & unknown &  & $\surd$ ($k=2$)\\\cline{3-5}
& & known &  & $\surd$ ($k=2$)\\
\hline
\multirow{3}{*}{PL model} & \multirow{3}{*}{list} & unknown & \cite{zhao2016learning} ($k\geq2$), \cite{MR2921622} ($k\geq2$) & $\surd$ ($k=2$)\\\cline{3-5}
& & known & \shortstack{\cite{zhao2016learning} $(k \geq 2)$ from proof details \\ \cite{MR2921622} $(k\geq2)$} & $\surd$ ($k=2$) \\
\bottomrule
\end{tabular}
\end{center}
} 
\end{table*}


\subsection{Mixtures of BTL models: Unknown mixing probabilities}

Here, we consider the case when the mixing probabilities $(p_1, p_2)$ are unknown. The parameters are $(\bm{a}_{2:n},\bm{b}_{2:n},p_1)$ in the set $Q_{BTL,p}^{2n-1}$. The estimation/ranking problem is to solve the equation system~\eqref{ex:BTL_2mixtures_pairwise_p} in variables $(\bm{x},\bm{y},p)$, with coefficients $\bm{\eta}(\bm{a},\bm{b},p_1)$.

Note that the system~\eqref{ex:BTL_2mixtures_pairwise_p} has at least two solutions in $\mathbb{C}$, i.e., $(\bm{x},\bm{y},p)=(\bm{a},\bm{b},p_1)$ or $(\bm{b},\bm{a},p_2)$, coming from the initial data. The following result shows that it is also the unique solution (up to reordering) in $\mathbb{C}$ for generic $(\bm{a}_{2:n},\bm{b}_{2:n},p_1) \in Q^{2n-1}_{BTL,p}$:

\begin{theorem}\label{prop-thm:btl_p}
If $n \geq 5$, the system~\eqref{ex:BTL_2mixtures_pairwise_p} has exactly two solutions in $\mathbb{C}$ (counted with multiplicity) for all $(\bm{a}_{2:n},\bm{b}_{2:n},p_1) \in Q^{2n-1}_{BTL,p}$ but a set of $\lambda_{2n-1}$-measure zero, given by $(\bm{x},\bm{y},p)=(\bm{a},\bm{b},p_1)$ or $(\bm{x},\bm{y},p)=(\bm{b},\bm{a},1-p_1)$.
\end{theorem}

\begin{proof}[Proof sketch]
The complete proof is in Appendix~\ref{apn-sub:examples_btl_proofs_p}. The proof uses induction on $n$. 
For the base case $n=5$, we will apply Theorem~\ref{thm:general-case} and verify Assumptions~\ref{assum:key-assumption1} and~\ref{assum:key-assumption2}. To do so, we first transform the system into a polynomial system. We eliminate the denominators in the system~\eqref{ex:BTL_2mixtures_pairwise_p} such that coefficients are given by polynomials in $(\bm{a},\bm{b},p_1)$:
\begin{equation}\label{sketch:ex:BTL_2mixtures_pairwise_p_poly_ab}
\begin{cases}
x_1 = y_1 = 1, \\
(c_{ij}-d_{ij}) x_iy_i + (c_{ij}-pd_{ij})x_iy_j + (c_{ij}-(1-p)d_{ij})x_jy_i + c_{ij}x_j y_j = 0, & \forall i < j \in [5], \\
t_{ij}(x_i+x_j) = 1, \ h_{ij}(y_i+y_j) = 1, & \forall i < j \in [5],
\end{cases}
\end{equation}
where $c_{ij}:=p_1a_i(b_i+b_j)+(1-p_1)b_i(a_i+a_j)$ and $d_{ij}:=(a_i+a_j)(b_i+b_j)$. Note that this is a system in the variables $(\bm{x},\bm{y},p,\bm{t},\bm{h})$ with coefficients given by polynomials in $(\bm{a},\bm{b},p_1)$. As a result of introducing new variables and equations to the system~\eqref{ex:BTL_2mixtures_pairwise_p} to prevent the denominators from being zero, it follows that for any $(\bm{a}_{2:n},\bm{b}_{2:n},p_1) \in Q^{9}_{BTL,p}$, the system~\eqref{ex:BTL_2mixtures_pairwise_p} is equivalent to the system~\eqref{sketch:ex:BTL_2mixtures_pairwise_p_poly_ab}. In particular, the system~\eqref{ex:BTL_2mixtures_pairwise_p} has exactly two solutions in $\mathbb{C}$ (counted with multiplicity) if and only if the system~\eqref{sketch:ex:BTL_2mixtures_pairwise_p_poly_ab} has exactly two solutions in $\mathbb{C}$ (counted with multiplicity).

We now verify the assumptions.
For Assumption~\ref{assum:key-assumption1}, note that the system~\eqref{sketch:ex:BTL_2mixtures_pairwise_p_poly_ab} has two distinct solutions:
\small
\[
(\bm{x},\bm{y},p,\bm{t},\bm{h})=\left(\bm{a},\bm{b},p_1,\left(\dfrac{1}{a_i{+}a_j}\right)_{i,j},\left(\dfrac{1}{b_i{+}b_j}\right)_{i,j}\right) \text{ or }\left(\bm{b},\bm{a},1-p_1,\left(\dfrac{1}{b_i{+}b_j}\right)_{i,j},\left(\dfrac{1}{a_i{+}a_j}\right)_{i,j}\right),
\]
\normalsize
for all $(\bm{a}_{2:5},\bm{b}_{2:5},p_1) \in \mathbb{C}^9$ except a Zariski closed proper subset of the parameters, which makes the two solutions the same or makes the denominators above vanish (i.e., $a_i+a_j=0$ or $b_i+b_j=0$). For Assumption~\ref{assum:key-assumption2}, we find $(\bm{a'}_{1:5},\bm{b'}_{1:5},p'_1)=(1,2,3,4,5;1,8,9,3,2;0.3)$ such that
$(\bm{a'}_{2:5},\bm{b'}_{2:5},p'_1) \in \mathbb{C}^{9} \setminus Z(\bm{a}_{2:5},\bm{b}_{2:5},p_1)$ and the associated equation system~\eqref{sketch:ex:BTL_2mixtures_pairwise_p_poly_ab} has exactly two solutions in $\mathbb{C}$ (counted with multiplicity). Both conditions can be checked using Magma (see Appendix~\ref{apn-sub:examples_btl_proofs_p}). 
Altogether, this means the system~\eqref{ex:BTL_2mixtures_pairwise_p} has a unique solution (up to reordering) in $\mathbb{C}$ (counted with multiplicity) for all $(\bm{a}_{2:5},\bm{b}_{2:5},p) \in Q_{BTL,p}^{9}$ but a set $N_{BTL,p}^{9}$ of $\lambda_{9}$-measure zero.

Turning to the induction step, suppose generic identifiability holds for $n-1$. We first split the equations into two parts:
\begin{equation}\label{sketch:equ:ind1}
\begin{cases}
x_1=y_1=1, \\
p\dfrac{x_i}{x_i+x_j}+(1-p)\dfrac{y_i}{y_i+y_j}=\eta_{ij}, & \forall i<j \in [n-1],
\end{cases}
\end{equation}
and
\begin{equation}\label{sketch:equ:ind2}
p\frac{x_i}{x_i+x_n}+(1-p)\frac{y_i}{y_i+y_n}=\eta_{in},\quad \forall i \in [n-1].
\end{equation}
By the induction hypothesis, generically, the system~\eqref{sketch:equ:ind1} has a unique solution for $(\bm{x}_{1:n-1}, \bm{y}_{1:n-1},p)$ (up to reordering) in $\mathbb{C}$ (counted with multiplicity). We can also manipulate equation~\eqref{sketch:equ:ind2} to obtain a linear system in $(x_n, y_n)$. This yields a unique solution for $(x_n, y_n)$ except for a measure-zero set, yielding a rank-deficient coefficient matrix. This establishes generic identifiability for $n$. 
\end{proof}

\subsection{Mixtures of BTL models: Known mixing probabilities}

If we instead assume that $(p_1, p_2)$ are known, the parameter space becomes $(\bm{a}_{2:n},\bm{b}_{2:n}) \in Q^{2n-2}_{BTL}$ and we consider the system~\eqref{ex:BTL_2mixtures_pairwise}. 
In this case, our goal is to give an approach that, given $(p_1, p_2)$, can establish generic identifiability of the system~\eqref{ex:BTL_2mixtures_pairwise}.
For concreteness, we will consider the case when $p_1=0.7$; cases involving different $p_1$ can be checked using identical procedures.

\begin{proposition}
\label{prop:btl}
If $n \geq 5$ and $(p_1,p_2) = (0.7,0.3)$,  the system~\eqref{ex:BTL_2mixtures_pairwise} has a unique solution in $\mathbb{C}$ (counted with multiplicity) for all $(\bm{a}_{2:n},\bm{b}_{2:n}) \in Q^{2n-2}_{BTL}$ but a set of $\lambda_{2n-2}$-measure zero, given by $(\bm{x},\bm{y})=(\bm{a},\bm{b})$.
\end{proposition}

\begin{proof}[Proof sketch]
The proof structure is similar to the proof of Theorem~\ref{prop-thm:btl_p}. We first prove the base case $n=5$. We translate the system~\eqref{ex:BTL_2mixtures_pairwise} into the following polynomial system with coefficients given by polynomials in $(\bm{a},\bm{b})$:
\begin{equation}\label{sketch:ex:BTL_2mixtures_pairwise_poly_ab_n=5}
\begin{cases}
x_1=y_1=1, \\
p_1(a_jx_i - a_i x_j)(b_i+b_j)(y_i+y_j) + p_2(b_jy_i +b_iy_j)(a_i+a_j)(x_j-x_i)= 0, & \forall i < j \in [5].
\end{cases}
\end{equation}
Showing the identifiability of the system~\eqref{ex:BTL_2mixtures_pairwise} is equivalent to showing that the system~\eqref{sketch:ex:BTL_2mixtures_pairwise_poly_ab_n=5} has exactly three solutions in $(\bm{x},\bm{y})$, which can be analytically written as
\small
\begin{equation}
(\bm{a}_{1:5},\bm{b}_{1:5}) \text{ or }\left(a_1,0,0,0,0;b_1,\left(\frac{1{-}\eta_{1j}}{\eta_{1j}{-}p_1}\right)_{j=2,\ldots,5}\right) \text{ or } \left(a_1,\left(\frac{1{-}\eta_{1j}}{\eta_{1j}{-}p_2}\right)_{j=2,\ldots,5};b_1,0,0,0,0\right).
\end{equation}
\normalsize
Then we can show that the system~\eqref{sketch:ex:BTL_2mixtures_pairwise_poly_ab_n=5} has at least three solutions, except possibly on a Zariski closed proper subset of the parameters which makes the three solutions the same or makes the denominators above vanish (i.e., $\eta_{1j}-p_1=0$ or $\eta_{1j}-p_2=0$). Thus, Assumption~\ref{assum:key-assumption1} holds.

For Assumption~\ref{assum:key-assumption2}, we first analytically compute $Z(\bm{a}_{2:5},\bm{b}_{2:5})$, defined in equation~\eqref{defn:Z(t)}, with Magma. Next, we pick a particular $(\bm{a}'_{2:5}, \bm{b}'_{2:5}) \notin Z(\bm{a}_{2:5},\bm{b}_{2:5})$ and run Magma to show that the system~\eqref{sketch:ex:BTL_2mixtures_pairwise_poly_ab_n=5}, plugged into $(\bm{a}'_{1:5}, \bm{b}'_{1:5})$, has exactly three solutions (counted with multiplicity), which means Assumption~\ref{assum:key-assumption2} holds. 
Altogether, this means the system~\eqref{ex:BTL_2mixtures_pairwise} has a unique solution in $\mathbb{C}$ (counted with multiplicity) for all $(\bm{a}_{2:5},\bm{b}_{2:5}) \in Q_{BTL}^{8}$ but a set $N_{BTL}^8$ of $\lambda_{8}$-measure zero.

The inductive step, as before, splits the whole equation system into two parts, where the first part only involves the variables $(\bm{x}_{1:n-1}, \bm{y}_{1:n-1})$ and the second part also involves the variables $(\bm{x}_n, \bm{y}_n)$. Based on the induction hypothesis, generically, we obtain a unique solution for $(\bm{x}_{1:n-1}, \bm{y}_{1:n-1})$. Plugging into the second part of the equation system, we can obtain a linear system and deduce that the solution in $(x_n,y_n)$ is also generically unique. 
\end{proof}



\section{Conclusion and discussion}
\label{sec:conclusion}

We have developed a general framework to check generic identifiability of polynomial systems via algebraic geometry, which can be applied to identifiability problems in machine learning. In particular, we applied it to prove the generic identifiability of three examples of mixtures of ranking models, including mixtures of BTL models, which was unsolved in the literature. 

One limitation of our work is that our results are for mixtures of two components. In principle, our general framework can also be applied for mixtures with more than two components. A natural future direction is then to apply our framework to the generic identifiability problem for mixtures of arbitrarily many components.
Another natural direction is the problem of polynomial-learning, i.e., whether parameters can be computed in polynomial time using a polynomial number of samples, assuming identifiability. 

\mypara{Societal impacts.} Our work is theoretical in nature; thus, we foresee no immediate negative societal impact. We are of the opinion that our theoretical framework may enhance the understanding of mixtures of ranking models and inspire the development of improved learning methods for them, which may have a positive impact in practice. We also perceive that our general framework for proving generic identifiability can be useful to researchers working on other machine learning models.

\section*{Acknowledgement}
We would like to thank Bernd Sturmfels for kindly informing us the rich literature on the identifiability problem from algebraic statistics. The first XZ would like to thank Professor Nigel Boston for helpful discussions and suggestions from a Mathematician perspective, and thank Julia Lindberg and Jose Israel Rodriguez for helpful conversations. The second XZ would like to thank his supervisor Jochen Heinloth for pointing out an error in a preliminary version of this paper and many helpful discussions on Theorem~\ref{main-result}.




\newpage

\bibliographystyle{unsrt} 

\bibliography{NIPS2022}


\newpage
\appendix

\begin{center}
	\textbf{\LARGE Supplementary Material }
\end{center}

\begin{center}
	\textbf{\large On the Identifiability of Mixtures of Ranking Models}
\end{center}

We provide the supplementary material in section order.

\section{Supplementary material for Section~\ref{sec:background}} \label{apn:background}
In this section, we will cover an example for non-identifiability and describe the MNL model with 3-slate and the PL model.

\subsection{An example of non-identifiability of mixture of ranking models}\label{apn:non-identifiable-example}

An example showing that mixtures of BTL ranking models may not be identifiable, suggested by~\cite{oh2014learning} and~\cite{chierichetti2018learning}, is provided here.

For some $n \geq 3$ and some $t > 0$, consider the following two latent uniform mixture parameters: 
\begin{enumerate}
    \item $a_1=a_2=t,b_1=b_2=\dfrac{1}{t},a_i=b_i=1$ for each $i \in \{3,4,\dots\}$.
    \item $a_1=b_2=t,b_1=a_2=\dfrac{1}{t},a_i=b_i=1$ for each $i \in \{3,4,\dots\}$.
\end{enumerate}
Let $ \Prob{i \succ j} = \dfrac{1}{2} \dfrac{a_i}{a_i+a_j} + \dfrac{1}{2}\dfrac{b_i}{b_i+b_j}$.
It is not hard to verify that the probabilities $\Prob{i\succ j}$ are equal for all $i,j \in [n]$ (see Theorem 2 in~\cite{chierichetti2018learning}). Therefore, for the parameters $\bm{a} = (t,t,1,\dots,1)$ and $\bm{b}=\left(\dfrac{1}{t},\dfrac{1}{t},1\dots,1\right)$, the following equation system is non-identifiable for any fixed $t$:
\begin{align*}
F(\bm{x},\bm{y};\bm{a},\bm{b})=
\begin{cases}
x_1-t=0, \\
y_1-t=0, \\
\dfrac{1}{2}\dfrac{x_1}{x_1+x_i} + \dfrac{1}{2}\dfrac{y_1}{y_1+y_i} - \dfrac{1}{2} = 0, & \forall i =2,3,\dots,n, \\
\dfrac{1}{2}\dfrac{x_2}{x_2+x_i} + \dfrac{1}{2}\dfrac{y_2}{y_2+y_i} - \dfrac{1}{2} = 0, & \forall i =3,4,\dots,n.
\end{cases}
\end{align*}

This example provides infinitely many choices of the parameters $(\bm{a}, \bm{b})$ which lead to non-identifiability.
Nonetheless, as the results of our paper suggest, we believe that such choices of $(\bm{a},\bm{b})$ constitute special cases which might justifiably be assumed to be avoided in practice.

\subsection{MNL model with 3-slate}

We now describe the mixtures of MNL model, and then provide details of what it means for the model to be identifiable.

\subsubsection{Probabilistic model}

An MNL model over $n$ items assigns probabilities for a slate, where a slate is a subset of all items. It gives the probability of an item being selected from a slate. Mathematically, for a slate $S = \{s_1,s_2,\dots,s_k\} \subseteq [n]$, we have 
\[
\Prob{s_i \mbox{ is selected from~} S} = \frac{a_{s_i}}{\sum_{j=1}^k a_{s_j}},
\]
where $a_{s_i}$ is the weight/score of item $s_i$. (A BTL model is a special case giving the probability of an item being selected from a slate of size two.)

For a mixture of two MNL models with $k$-slate, assuming the mixtures follow a Bernoulli distribution with parameters $(p_1, p_2)$, we obtain
\[
\Prob{s_i \mbox{ is selected from~} S} = p_1 \dfrac{a_{s_i}}{\sum_{j=1}^k a_{s_j}} + p_2 \dfrac{b_{s_i}}{\sum_{j=1}^k b_{s_j}}.
\]
Now suppose $k = 3$ and $n \geq 3$. Let $\bm{a}_{1:n}$ and $\bm{b}_{1:n}$ be the score parameters of the two mixtures. Then we obtain
\begin{equation}\label{eq:eta_def_mnl}
 \eta_{i,j,k} = p_1 \dfrac{ a_i}{a_i + a_j + a_k} +  p_2\dfrac{b_i}{b_i + b_j + b_k}. \quad \forall i < j < k \in [n].
\end{equation}
We choose to scale up $\bm{a}$ by multiplying by a constant such that $a_1 = 1$, and similarly manipulate $\bm{b}$ to have $b_1 = 1$. This does not affect the values of the $\eta_{i,j,k}$'s. 

\subsubsection{Identifiability}

When we consider $(p_1, p_2)$ to be unknown, our domain $Q_{MNL,p}^{2n-1}$ of $(\bm{a},\bm{b},p)$ is defined as $Q_{MNL,p}^{2n-2} \times (0,1) \subseteq \real^{2n-1}$, where $Q^{2n-2}_{MNL} := \prod_{i=1}^{2n-2}\left[r_i,R_i\right] \subseteq \mathbb{R}^{2n-2}$ is the domain of $(\bm{a}_{2:n},\bm{b}_{2:n})$, with $R_i > r_i > 0$. We then solve the equation system
\begin{align}\label{ex:BTL_2mixtures_triplet_p}
\begin{cases}
x_1 = y_1 = 1, \\
p\dfrac{x_i}{x_i + x_j + x_k} +  (1-p)\dfrac{y_i}{y_i + y_j + y_k} = \eta_{i,j,k}. & \forall i < j < k \in [n].
\end{cases}
\end{align}
Accordingly, we define the set of bad parameters $N^{2n-1}_{MNL,p}$ in the same way as $N^{2n-1}_{BTL,p}$.

Given $p_1$ and $p_2$, we wish to solve the following equation system in the variables $(\bm{x}, \bm{y}) := (x_{1:n}, y_{1:n})$, where $n \geq 3$:
\begin{align}\label{ex:BTL_2mixtures_triplet}
\begin{cases}
x_1 = y_1 = 1, \\
p_1 \dfrac{x_i}{x_i + x_j + x_k} + p_2 \dfrac{y_i}{y_i + y_j + y_k} = \eta_{i,j,k}. & \forall i < j < k \in [n].
\end{cases}
\end{align}

With $Q^{2n-2}_{MNL}$ being the domain of $(\bm{a}_{2:n},\bm{b}_{2:n})$, we can then define the set of bad parameters $N^{2n-2}_{MNL}$ in a similar fashion to $N^{2n-2}_{BTL}$.
 \begin{align}
\begin{split}
    N^{2n-2}_{MNL} = &\left\{(\bm{a}_{2:n},\bm{b}_{2:n}) \in Q^{2n-2}_{MNL}: \exists \left(\bm{a}_{2:n}^{\#},\bm{b}^{\#}_{2:n}\right) \in Q^{2n-2}_{MNL}, \text{ s.t.}  \left(\bm{a}^{\#}_{2:n} \neq \bm{a}_{2:n} \vee \bm{b}^{\#}_{2:n} \neq \bm{b}_{2:n}\right) \right. \wedge \\
    & \left. \left(\forall i,j,k \in [n],  \eta_{i,j,k}(\bm{a}^{\#},\bm{b}^{\#}) = \eta_{i,j,k}(\bm{a},\bm{b}) \text{ for } a_1^{\#} =b_1^{\#} = a_1 = b_1 = 1\right) \right\}.
    \end{split}
    \end{align}

\subsection{Plackett-Luce model}\label{sec:pl_model}

We now consider the mixtures of Plackett-Luce model.

\subsubsection{Probabilistic model}

The Plackett-Luce model \cite{plackett1975analysis, luce1959} assigns probabilities to all the rankings $i_1 \succ i_2 \succ \dots \succ i_n$. They are from the set $\mathfrak{S}_n$ of all permutations of $\{1,2,\dots,n\}$. 

In particular, for any ranking from $\mathfrak{S}_n$, we have 
\[
\Prob{i_1 \succ i_2 \succ \dots \succ i_n} = \frac{a_{i_1}}{a_{i_1} + a_{i_2} + \dots + a_{i_n}} \times \frac{a_{i_2}}{a_{i_2} + a_{i_3} + \dots + a_{i_n}} \times \dots \times \frac{a_{i_{n-1}}}{a_{i_{n-1}} +a_{i_n}}.
\]
We choose to scale up $\bm{a}_{1:n}$ by multiplying by a constant such that $a_1 + a_2 + \dots + a_n = 1$, and similarly manipulate $\bm{b}_{1:n}$ to have $b_1 + b_2 + \dots + b_n= 1$. 

\subsubsection{Identifiability}

Suppose $(p_1,p_2)$ are unknown. Let 
\[
Q^{2n-2}_{PL} := \left\{\begin{pmatrix}\bm{a}_{2:n} \\ \bm{b}_{2:n}\end{pmatrix} \in \mathbb{R}_+^{2n-2} : \sum_{i=2}^n a_i < 1 \wedge \sum_{i=2}^n b_i < 1 \right\}
\]
be the domain of $(\bm{a}_{2:n},\bm{b}_{2:n})$ and $(0,1) \subseteq \real$ be the domain of $p_1$. Note that $Q^{2n-2}_{PL}$ is a polytope of positive volume. We set $\bm{a}_{2:n}$ and $\bm{b}_{2:n}$ as free parameters, with $a_1 = 1 - \sum_{i=2}^n a_i$ and $b_1 = 1 - \sum_{i=2}^n b_i$. We define $Q_{PL,p}^{2n-1}:= Q_{PL}^{2n-2} \times (0,1) \subseteq \real^{2n-1}$. The equation system we wish to solve is 
\begin{align}\label{ex:PL-2mixtures_p}
\begin{cases}
\eta_{\sigma(1), \dots, \sigma(n)} = p \displaystyle\prod_{i=1}^{n-1}\dfrac{x_{\sigma(i)}}{\sum_{j=i}^n x_{\sigma(j)}} + (1-p) \displaystyle\prod_{i=1}^{n-1}\dfrac{y_{\sigma(i)}}{\sum_{j=i}^n y_{\sigma(j)}}, & \forall \sigma \in \mathfrak{S}_n. \\
x_1 = 1 - \sum_{i=2}^n x_i, \quad y_1 = 1 - \sum_{i=2}^n y_i.
\end{cases}
\end{align}
Let $\bm{\eta}_{\mathfrak{S}_n} = (\bm{\eta}_\sigma)_{\sigma \in \mathfrak{S}_n}$ be the  $n!$ probabilities of rankings among $n$ items.
The corresponding bad set is
\begin{align}
\begin{split}
    N^{2n-1}_{PL,p} = & \left\{(\bm{a}_{2:n},\bm{b}_{2:n},p_1) \in Q^{2n-1}_{PL,p}: \exists \left(\bm{a}_{2:n}^{\#},\bm{b}^{\#}_{2:n},p^{\#}\right) \in Q^{2n-1}_{PL,p}, \text{ s.t.}  \left(\bm{a}^{\#}_{2:n} \neq \bm{a}_{2:n} \vee \bm{b}^{\#}_{2:n} \neq \bm{b}_{2:n} \vee p^{\#} \neq p_1\right) \wedge \right. \\
    & \left. \left(\bm{\eta}_{\mathfrak{S}_n}(\bm{a}^{\#},\bm{b}^{\#},p^{\#}) = \bm{\eta}_{\mathfrak{S}_n}(\bm{a},\bm{b},p_1) \text{ for } \begin{bmatrix}a_1^{\#} \\ b_1^{\#}\end{bmatrix} = \begin{bmatrix} 1- \sum_{i=2}^n a_i^{\#} \\ 1- \sum_{i=2}^n b_i^{\#}\end{bmatrix}, \begin{bmatrix}a_1 \\ b_1\end{bmatrix} = \begin{bmatrix} 1- \sum_{i=2}^n a_i \\ 1- \sum_{i=2}^n b_i\end{bmatrix} \right) \right\}.
\end{split}
\end{align}

Given $(p_1,p_2)$, to determine the scores of two mixtures, we try to solve the following system of equations in $(\bm{x}, \bm{y})$:
\begin{align}\label{ex:PL-2mixtures}
\begin{cases}
\eta_{\sigma(1), \sigma(2), \dots, \sigma(n)} =p_1 \displaystyle\prod_{i=1}^{n-1}\dfrac{x_{\sigma(i)}}{\sum_{j=i}^n x_{\sigma(j)}}+p_2 \displaystyle\prod_{i=1}^{n-1}\dfrac{y_{\sigma(i)}}{\sum_{j=i}^n y_{\sigma(j)}}, & \forall \sigma \in \mathfrak{S}_n, \\
x_1 = 1 - \sum_{i=2}^n x_i, \quad y_1 = 1 - \sum_{i=2}^n y_i.
\end{cases}
\end{align}

The domain of $(\bm{a}_{2:n},\bm{b}_{2:n})$ is $Q^{2n-2}_{PL}$. Then the set of bad parameters which do not achieve identifiability is
\begin{align}
\begin{split}
    N^{2n-2}_{PL} = &\left\{(\bm{a}_{2:n},\bm{b}_{2:n}) \in Q^{2n-2}_{PL}: \exists (\bm{a}_{2:n}^{\#},\bm{b}^{\#}_{2:n}) \in Q^{2n-2}_{PL}, \text{ s.t. } (\bm{a}^{\#}_{2:n} \neq \bm{a}_{2:n} \vee \bm{b}^{\#}_{2:n} \neq \bm{b}_{2:n}\right) \wedge \\
    & \left. \left(\bm{\eta}_{\mathfrak{S}_n}(\bm{a}^{\#},\bm{b}^{\#}) = \bm{\eta}_{\mathfrak{S}_n}(\bm{a},\bm{b})  \text{ for } \begin{bmatrix}a_1^{\#} \\ b_1^{\#}\end{bmatrix} = \begin{bmatrix} 1- \sum_{i=2}^n a_i^{\#} \\ 1- \sum_{i=2}^n b_i^{\#}\end{bmatrix}, \begin{bmatrix}a_1 \\ b_1\end{bmatrix} = \begin{bmatrix} 1- \sum_{i=2}^n a_i \\ 1- \sum_{i=2}^n b_i\end{bmatrix} \right) \right\}.
\end{split}
\end{align}
For any of these mixture models, we say it achieves generic identifiability if the corresponding set of bad parameters is of measure zero.
\subsubsection{Derivations for Section~\ref{sec:pl_model}}\label{apn:pl-three}
Let $\{i: i=1,\ldots,n\}$ denote a set of $n$ alternatives. A ranking of $(1,\ldots,n)$ determined by $\sigma \in \mathfrak{S}_n$ is an ordering
\[
\sigma(1) \succ \sigma(2) \succ \cdots \succ \sigma(n).
\]
Given a parameter $\bm{\theta}=(\theta_1,\ldots,\theta_n)$, the probability of the ranking $\sigma$ is given by
\[
\Prob{\sigma}:=\prod_{i=1}^n \dfrac{\theta_{\sigma(i)}}{\sum_{j \geq i} \theta_{\sigma(j)}}.
\]
We will prove the following statement for $n \geq 3$ by induction:
\[
\Prob{\sigma(1)\succ\sigma(2)\succ others} = \frac{\theta_{\sigma(1)}}{\sum_{i=1}^n\theta_{\sigma(i)}}\frac{\theta_{\sigma(2)}}{\sum_{i=2}^n\theta_{\sigma(i)}}.
\]
For the base case $n=3$, we see that by definition, 
\[
\Prob{\sigma(1)\succ\sigma(2)\succ others} = \Prob{\sigma(1)\succ\sigma(2)\succ \sigma(3)} =  \frac{\theta_{\sigma(1)}}{\sum_{i=1}^3\theta_{\sigma(i)}}\frac{\theta_{\sigma(2)}}{\sum_{i=2}^3\theta_{\sigma(i)}}\frac{\theta_{\sigma(3)}}{\theta_{\sigma(3)}} = \frac{\theta_{\sigma(1)}}{\sum_{i=1}^3\theta_{\sigma(i)}}\frac{\theta_{\sigma(2)}}{\sum_{i=2}^3\theta_{\sigma(i)}}.
\]
For the inductive step, assume we have 
\begin{align}\label{eq:n-1case}
\Prob{\sigma(1)\succ\sigma(2)\succ other\ n-3\ items} = \frac{\theta_{\sigma(1)}}{\sum_{i=1}^{n-1}\theta_{\sigma(i)}}\frac{\theta_{\sigma(2)}}{\sum_{i=2}^{n-1}\theta_{\sigma(i)}}.
\end{align}
We then split our case of $n$ into $n-2$ subcases:
\begin{align*}
\begin{split}
\Prob{\sigma(1)\succ\sigma(2)\succ others} &= \sum_{i\neq \sigma(1),\sigma(2)} \Prob{\sigma(1)\succ\sigma(2)\succ i \succ other\ n-3\ items}.
\end{split}
\end{align*}
Note that the probability formula for $\sigma(2) \succ i \succ other \ n-3 \ items$ does not depend on $\theta_{\sigma(1)}$ and has $n-1$ items in it. Then by the hypothesis~\eqref{eq:n-1case}, we further obtain
\begin{align*}
\begin{split}
\Prob{\sigma(1)\succ\sigma(2)\succ i \succ other\ n-3\ items} = \frac{\theta_{\sigma(1)}}{\sum_{j = 1}^n \theta_{\sigma(j)}} \frac{\theta_{\sigma(2)}}{\sum_{j = 2}^n \theta_{\sigma(j)}} \frac{\theta_i}{\sum_{k \neq \sigma(1),\sigma(2)}^n \theta_{k}}.
\end{split}
\end{align*}
Hence,
\[
\Prob{\sigma(1)\succ\sigma(2) \succ others}= \frac{\theta_{\sigma(1)}}{\sum_{j = 1}^n \theta_{\sigma(j)}} \frac{\theta_{\sigma(2)}}{\sum_{j = 2}^n \theta_{\sigma(j)}} \sum_{i\neq \sigma(1),\sigma(2)}\frac{\theta_i}{\sum_{k \neq \sigma(1),\sigma(2)}^n \theta_{k}}=\frac{\theta_{\sigma(1)}}{\sum_{j = 1}^n \theta_{\sigma(j)}} \frac{\theta_{\sigma(2)}}{\sum_{j = 2}^n \theta_{\sigma(j)}}.
\]
Without loss of generality, we assumed that $\sum_{i=1}^n \theta_i = 1$. Thus, we conclude
\[
\Prob{\sigma(1)\succ\sigma(2) \succ others} =\frac{\theta_{\sigma(1)}}{1} \frac{\theta_{\sigma(2)}}{1 - \theta_{\sigma(1)}} = \frac{\theta_{\sigma(1)}\theta_{\sigma(2)}}{1 - \theta_{\sigma(1)}} .
\]
For two mixtures, we then have 
\[
\Prob{\sigma(1)\succ\sigma(2) \succ others; \mbox{ two mixtures}} = p_1 \frac{a_{\sigma(1)}a_{\sigma(2)}}{1 - a_{\sigma(1)}} + p_2 \frac{b_{\sigma(1)}b_{\sigma(2)}}{1 - b_{\sigma(1)}},
\]
where the $a_i$'s are parameters for the first mixture and the $b_i$'s are parameters for the second mixture.

Finally, the last formula to be proved is
\begin{align*}
\begin{split}
\Prob{\sigma(1) \succ others; \mbox{ two mixtures}} &= \sum_{i\neq \sigma(1)} \Prob{\sigma(1)\succ i \succ others; \mbox{ two mixtures }} \\
& = p_1 \frac{a_{\sigma(1)}}{1 - a_{\sigma(1)}} \sum_{i \neq \sigma(1)} a_{i} + p_2 \frac{b_{\sigma(1)}}{1 - b_{\sigma(1)}} \sum_{i \neq \sigma(1)} b_{i} \\
& = p_1 \frac{a_{\sigma(1)}}{1 - a_{\sigma(1)}} (1 - a_{\sigma(1)}) + p_2 \frac{b_{\sigma(1)}}{1 - b_{\sigma(1)}} (1 - b_{\sigma(1)}) \\
& = p_1 a_{\sigma(1)} + p_2 b_{\sigma(1)}.
\end{split}
\end{align*}

\section{Supplementary material for Section~\ref{sect:main-result}}
\label{apn:alg-geo-prelim}

We first review some terminology and results from algebraic geometry, and then give the proofs for our main theory.

\subsection{Reminder from algebraic geometry}

We now introduce some definitions and results from algebraic geometry along with some guiding examples. More details may be found in any of the standard textbooks in algebraic geometry, e.g., \cite{MR0463157, MR4225278, MR1917232}.

\subsubsection{Definitions}

Let $\mathbb{C}[x_1,\ldots,x_n]$ be the polynomial ring in $n$ variables over $\mathbb{C}$.

\begin{definition}[Zero set]\label{zero-set}
For any subset $S \subseteq \mathbb{C}[x_1,\ldots,x_n]$, we define the \textit{zero set} of $S$ to be the common zeros of all elements in $S$, namely
\[
\mathbb{V}(S):=\{x \in \mathbb{C}^n: f(x)=0 \text{ for all } f \in S\}.
\]
\end{definition}
Clearly, if $\mathfrak{a}$ is the ideal of $\mathbb{C}[x_1,\ldots,x_n]$ generated by $S$, then $\mathbb{V}(S)=\mathbb{V}(\mathfrak{a})$. Furthermore, Hilbert's Basis Theorem \cite[Corollary 7.6]{MR0242802} implies that the polynomial ring $\mathbb{C}[x_1,\ldots,x_n]$ is Noetherian, i.e., all its ideals are finitely generated. Hence, there exist finitely many elements $f_1,\ldots,f_r \in S$ such that $\mathbb{V}(S)=\mathbb{V}(f_1,\ldots,f_r)$.

\begin{proposition-definition}[Zariski topology]\label{Zari-topo}
The sets $\mathbb{V}(\mathfrak{a})$, where $\mathfrak{a}$ runs through the set of ideals of $\mathbb{C}[x_1,\ldots,x_n]$, are the closed sets of a topology on $\mathbb{C}^n$, called the \textit{Zariski topology}.
\end{proposition-definition}
If we consider $\mathbb{C}^n$ as a topological space equipped with the Zariski topology, we will denote this space by $\mathbb{A}^n_{\mathbb{C}}$ and call it the \textit{affine $n$-space over $\mathbb{C}$}. 

By basic results in topology, any non-empty open subset of $\mathbb{A}_{\mathbb{C}}^n$ is dense.
\begin{definition}[Affine variety]
A closed subset of $\mathbb{A}_{\mathbb{C}}^n$ is called an \textit{affine variety}.\footnote{Compared to the usual definition, we do not require irreducibility here.}

For example, $\mathbb{A}_{\mathbb{C}}^n$ itself is an affine variety. Moreover:
\begin{itemize}
\item
$\mathbb{V}(f) \subseteq \mathbb{A}_{\mathbb{C}}^1$ is an affine variety for any non-zero polynomial $f \in \mathbb{C}[x]$, consisting of $\deg f$ points in $\mathbb{C}$ (counted with multiplicity); and
\item
the hyperplane
\begin{align}\label{eq:hyperplane}
H(a_1,\ldots,a_n):=\left\lbrace (x_1,\ldots,x_n) \in \mathbb{A}_{\mathbb{C}}^n: \sum_{i=1}^n a_ix_i=0 \right\rbrace \subseteq \mathbb{A}_{\mathbb{C}}^n
\end{align}
is an affine variety for any $(a_1,\ldots,a_n) \in \mathbb{C}^n \setminus \{0\}$.
\end{itemize}
\end{definition}

Let $X \subseteq \mathbb{A}_{\mathbb{C}}^n$ be an affine variety.
\begin{definition}[Dimension]
The \textit{dimension} of $X$ is the dimension of its underlying topological space.

For example, we have
\begin{itemize}
\item
$\dim \mathbb{A}_{\mathbb{C}}^n=n$,
\item 
$\dim \mathbb{V}(f)=0$ for any non-zero polynomial $f \in \mathbb{C}[x]$,
\item
$\dim H(a_1,\ldots,a_n)=n-1$, with $H$ defined by equation~\eqref{eq:hyperplane}, for any $(a_1,\ldots,a_n) \in \mathbb{C}^n \setminus \{0\}$.
\end{itemize}
\end{definition}
\begin{definition}[Degree]
The \textit{degree} of $X$ is the number of points (counted with multiplicity) in the intersection $X \cap L$ for a generic affine $(n-\dim X)$-space $L \subseteq \mathbb{A}_{\mathbb{C}}^n$. 

For the preceding notion to be well-defined, we need to show that the number of points (counted with multiplicity) in the intersection $X \cap L$ is constant across generic affine $(n-\dim X)$-spaces $L \subseteq \mathbb{A}_{\mathbb{C}}^n$. A proof of that fact can be found in~\cite[\S (14.31)]{MR4225278}.

To provide some examples:
\begin{itemize}
\item $\deg \mathbb{A}_{\mathbb{C}}^n=1$, since the only affine $0$-space is $\{0\}$.
\item 
$\deg \mathbb{V}(f)=\#\mathbb{V}(f)=\deg f$ for any non-zero polynomial $f \in \mathbb{C}[x]$.
\item
$\deg H(a_1,\ldots,a_n)=1$, with $H$ defined by equation~\eqref{eq:hyperplane}, for any $(a_1,\ldots,a_n) \in \mathbb{C}^n \setminus \{0\}$.
\end{itemize}
\end{definition}
\begin{definition}\label{order>}
For two polynomials $\mathbf{P}_1(z),\mathbf{P}_2(z) \in \mathbb{Q}[z]$, we write $\mathbf{P}_1(z) \succ \mathbf{P}_2(z)$ if $\mathbf{P}_1(n) \geq \mathbf{P}_2(n)$ for all $n \gg 0, n \in \mathbb{Z}$.
\end{definition}
\subsubsection{Results}
\begin{theorem}[\cite{MR2568219}, Chapter I, Section B, Corollary 10]
For any non-zero polynomial $f \in \mathbb{C}[x_1,\ldots,x_m]$, the set $\mathbb{V}(f) \subseteq \mathbb{C}^m$ is of $\lambda_m^{\mathbb{C}}$-measure zero.
\end{theorem}
\begin{lemma}\label{closed-0}
Any Zariski closed proper subset of $\mathbb{A}_{\mathbb{C}}^m$ is of $\lambda_m^{\mathbb{C}}$-measure zero.
\end{lemma}
\begin{proof}
Let $Z \subsetneq \mathbb{A}_{\mathbb{C}}^m$ be a Zariski closed subset. We write 
\[
Z=\mathbb{V}(f_1,\ldots,f_r)=\bigcap_{i=1}^r \mathbb{V}(f_i)
\]
for some polynomials $f_i \in \mathbb{C}[x_1,\ldots,x_m]$. Since $Z \subsetneq \mathbb{A}_{\mathbb{C}}^m$, at least one of the $f_i$'s is non-zero, say $f_1$. By \cite[Chapter I, Section B, Corollary 10]{MR2568219}, it holds that $\mathbb{V}(f_1) \subseteq \mathbb{C}^m$ is of $\lambda_m^{\mathbb{C}}$-measure zero. As a subset of a set of $\lambda_m^{\mathbb{C}}$-measure zero, $Z$ is again of $\lambda_m^{\mathbb{C}}$-measure zero, since the Lebesgue measure on $\mathbb{C}^m$ is complete.
\end{proof}
\begin{theorem}[\cite{MR0463157}, Proof of Corollary III.9.10]
Let $X \subseteq \mathbb{P}_{\mathbb{C}}^n$ be a projective variety of dimension $r$. Then the degree of its Hilbert polynomial is $r$ and the coefficient of its leading term is $\frac{\deg X}{r!}$. In other words, the Hilbert polynomial of $X \subseteq \mathbb{P}_{\mathbb{C}}^n$ has the form
\[
\mathbf{P}_X(z)=\dfrac{\deg X}{r!}z^r+\text{lower-degree terms} \in \mathbb{Q}[z].
\]
In particular, if $r=0$, then $\mathbf{P}_X(z)=\deg X=\#X(\mathbb{C})$ (counted with multiplicity) is constant.
\end{theorem}
\begin{theorem}[Special case of \cite{MR2223407}, Theorem 5.13]\label{open}
Let $\mathscr{X} \subseteq \mathbb{A}^m_{\mathbb{C}} \times \mathbb{P}^n_{\mathbb{C}}$ be a closed subvariety. For any $t \in \mathbb{A}^m_{\mathbb{C}}$, let $\mathscr{X}_t \subseteq \mathbb{P}^n_{\mathbb{C}}$ be the fiber over $t$ under the projection $\mathscr{X} \subseteq \mathbb{A}^m_{\mathbb{C}} \times \mathbb{P}^n_{\mathbb{C}} \to \mathbb{A}^m_{\mathbb{C}}$ and $\mathbf{P}_t(z) \in \mathbb{Q}[z]$ its Hilbert polynomial. If $f$ is the minimal element in $\mathscr{P}:=\{\mathbf{P}_t(z): t \in \mathbb{A}^m_{\mathbb{C}}\}$ with respect to $\succ$ (in the sense of Definition \ref{order>}), then
\[
S_f:=\{t \in \mathbb{A}^m_{\mathbb{C}}: \mathbf{P}_t(z)=f\} \subseteq \mathbb{A}^m_{\mathbb{C}}
\]
is a Zariski open dense subset.
\end{theorem}
In Theorem \ref{open}, it is often useful to view $\mathscr{X}$ as a family of projective varieties $\mathscr{X}_t \subseteq \mathbb{P}^n_{\mathbb{C}}$ parametrized by $t \in \mathbb{A}^m_{\mathbb{C}}$.

\subsection{Proof for Section~\ref{sect:main-result}}\label{apn:main-results}

Here, we provide the proofs for the two main theorems. 
\subsubsection{Proof of Theorem~\ref{main-result}}\label{apn:main-results-complex}
The proof is divided into two steps. The first step shows that the conclusion holds ``generically'' with respect to the Zariski topology, and the second step translates the property in Zariski topology into Lebesgue measure.
\begin{proof}[Proof of Theorem~\ref{main-result}]
Geometrically, one can view $\mathscr{P}(\bm{t},\bm{x})$ as a family of affine varieties sitting in $\mathbb{A}^n_{\mathbb{C}}$,
\[
\begin{tikzcd}
\mathscr{Q} \ar[r,hook] \ar[d] & \mathbb{A}^m_{\mathbb{C}} \times \mathbb{A}^n_{\mathbb{C}} \\
\mathbb{A}^m_{\mathbb{C}}
\end{tikzcd}
\]
parametrized by $\mathbb{A}^m_{\mathbb{C}}$, such that the fiber $\mathscr{Q}_{\bm{a}}$ over $\bm{a} \in \mathbb{A}^m_{\mathbb{C}}$ is (isomorphic to) the affine variety $\mathbb{V}(\mathscr{P}(\bm{a},\bm{x})) \subseteq \mathbb{A}^n_{\mathbb{C}}$. Formally $\mathscr{P}(\bm{a},\bm{x})$ is the image of $\mathscr{P}(\bm{t},\bm{x})$ under the evaluation map
\[
\varphi_{\bm{a}}: \mathbb{C}[\bm{t},\bm{x}] \to \mathbb{C}[\bm{x}] \text{ defined by } \bm{t} \mapsto \bm{a}, \bm{x} \mapsto \bm{x}.
\]
Fix an open embedding $\mathbb{A}_{\mathbb{C}}^n \hookrightarrow \mathbb{P}^n_{\mathbb{C}}$ given by $(x_1,\ldots,x_n) \mapsto [x_0=1:x_1:\ldots:x_n]$. Consider the (relative) projective closure of the above diagram with respect to $\bm{x}$ and this embedding 
\[
\begin{tikzcd}
\overline{\mathscr{Q}} \ar[r,hook] \ar[d] & \mathbb{A}^m_{\mathbb{C}} \times \mathbb{P}^n_{\mathbb{C}} \\
\mathbb{A}^m_{\mathbb{C}}
\end{tikzcd}
\]
i.e., $\overline{\mathscr{Q}}_{\bm{a}}|_{x_0=1}=\mathscr{Q}_{\bm{a}}$ for any $\bm{a} \in \mathbb{A}^m_{\mathbb{C}}$. By the minimality of $\overline{\mathscr{Q}_{\bm{a}}}$, it follows that
\begin{equation}\label{source-of-trouble}
\overline{\mathscr{Q}_{\bm{a}}} \subseteq \overline{\mathscr{Q}}_{\bm{a}} \text{ for } \bm{a} \in \mathbb{A}^m_{\mathbb{C}}.
\end{equation}
The containment~\eqref{source-of-trouble} does not necessarily hold with equality, i.e., $\overline{\mathscr{Q}}_{\bm{a}}$ is not the projective closure of $\mathscr{Q}_{\bm{a}}$. For example, for the single polynomial $(1-tx)x=0$, $\overline{\mathscr{Q}_{0}}$ is defined by $x=0$ in $\mathbb{P}^1$ (and hence consists of one point $[0:1]$) while $\overline{\mathscr{Q}}_0$ is defined by $x_0x=0$ in $\mathbb{P}^1$ (and hence consists of two points $[0:1]$ and $[1:0]$). This is the source of much trouble. However, we can prove that generically, the containment~\eqref{source-of-trouble} is an equality (in the previous example, it is an equality except for $t=0$).

\begin{lemma}\label{reduction}
$\overline{\mathscr{Q}_{\bm{a}}}=\overline{\mathscr{Q}}_{\bm{a}}$ for all $\bm{a} \in \mathbb{C}^m \setminus Z(\bm{t})$.
\end{lemma}
\begin{proof}[Proof of Lemma~\ref{reduction}]
Note that
\begin{enumerate}
\item
$\mathscr{Q}_{\bm{a}}=\mathbb{V}(\mathscr{P}(\bm{a},\bm{x}))=\mathbb{V}(I(\bm{a},\bm{x}))=\mathbb{V}(g_1(\bm{a},\bm{x}),\ldots,g_s(\bm{a},\bm{x})) \subseteq \{\bm{a}\} \times \mathbb{A}^n_{\mathbb{C}} \cong \mathbb{A}^n_{\mathbb{C}}$. By our choice of $\bm{a}$, it follows that $I(\bm{a},\bm{x}) \neq 0$, and \cite[Theorem 2.1]{MR1854335} guarantees that $G(\bm{a},\bm{x})$ is a Gr\"{o}bner basis of $I(\bm{a},\bm{x})$. Thus,
\[
\overline{\mathscr{Q}_{\bm{a}}}=\mathbb{V}(^h(g_1(\bm{a},\bm{x})),\ldots,^h(g_s(\bm{a},\bm{x}))) \subseteq \mathbb{P}^n_{\mathbb{C}}
\]
by \cite[Chapter 8, Theorem 4 and 8]{MR3330490}, where $^h(g_i(\bm{a},\bm{x}))$ is the homogenization of $g_i(\bm{a},\bm{x})$ with respect to $\bm{x}$.
\item
$\mathscr{Q}=\mathbb{V}(\mathscr{P}(\bm{t},\bm{x}))=\mathbb{V}(I(\bm{t},\bm{x}))=\mathbb{V}(g_1(\bm{t},\bm{x}),\ldots,g_s(\bm{t},\bm{x})) \subseteq \mathbb{A}^m_{\mathbb{C}} \times \mathbb{A}^n_{\mathbb{C}}$. Hence, 
\[
\overline{\mathscr{Q}}=\mathbb{V}(^hg_1(\bm{t},\bm{x}),\ldots,^hg_s(\bm{t},\bm{x})) \subseteq \mathbb{A}^m_{\mathbb{C}} \times \mathbb{P}^n_{\mathbb{C}}
\]
by \cite[Chapter 8, Theorem 4 and 8]{MR3330490}, where $^hg_i(\bm{t},\bm{x})$ is the homogenization of $g_i(\bm{t},\bm{x})$ with respect to $\bm{x}$. This implies that
\[
\overline{\mathscr{Q}}_{\bm{a}}=\mathbb{V}((^hg_1(\bm{t},\bm{x}))(\bm{a}),\ldots,(^hg_s(\bm{t},\bm{x}))(\bm{a})) \subseteq \{\bm{a}\} \times \mathbb{P}^n_{\mathbb{C}} \cong \mathbb{P}^n_{\mathbb{C}}.
\]
\end{enumerate}
Note that $\bm{a} \in \mathbb{C}^m \setminus Z(\bm{t})$ is a sufficient and necessary condition for
\[
(^hg_i(\bm{t},\bm{x}))(\bm{a})=^h(g_i(\bm{a},\bm{x})), \text{ for } i=1,\ldots,s.
\]
This implies that $\overline{\mathscr{Q}_{\bm{a}}}=\overline{\mathscr{Q}}_{\bm{a}}$ for all $\bm{a} \in \mathbb{C}^m \setminus Z(\bm{t})$.
\end{proof}
\paragraph{Step 1:}
We prove that $\mathscr{P}(\bm{a},\bm{x})$ has exactly $\ell$ solutions in $\mathbb{C}$ (counted with multiplicity) for generic $\bm{a} \in \mathbb{A}_{\mathbb{C}}^m \setminus Z(\bm{t})$.

More precisely, there exists a Zariski open dense
subset $U_m \subseteq \mathbb{A}_{\mathbb{C}}^m$ such that $\mathscr{P}(\bm{a},\bm{x})$ has exactly $\ell$ solutions in $\mathbb{C}$ (counted with multiplicity) for any $\bm{a} \in U_m$.

For any $\bm{a} \in \mathbb{A}_{\mathbb{C}}^m \setminus Z(\bm{t})$, let $\mathbf{P}_{\bm{a}}(z) \in \mathbb{Q}[z]$ the Hilbert polynomial of $\overline{\mathscr{Q}}_{\bm{a}} \subseteq \mathbb{P}^n_{\mathbb{C}}$. Then by definition,
\begin{gather*}
\mathscr{P}(\bm{a},\bm{x}) \text{ has exactly } \ell \text{ solutions in } \mathbb{C} \text{ (counted with multiplicity)} \\
\Updownarrow \\
\dim \mathscr{Q}_{\bm{a}}=0 \text{ and } \deg \mathscr{Q}_{\bm{a}}=\ell \\
\Updownarrow \\
\dim \overline{\mathscr{Q}_{\bm{a}}}=0 \text{ and } \deg \overline{\mathscr{Q}_{\bm{a}}}=\ell \\
\Updownarrow \\
\dim \overline{\mathscr{Q}}_{\bm{a}}=0 \text{ and } \deg \overline{\mathscr{Q}}_{\bm{a}}=\ell \\
\Updownarrow \\
\mathbf{P}_{\bm{a}}(z)=\ell,
\end{gather*}
where the first equivalence holds since $\mathbb{C}$ is algebraically closed; the second equivalence follows from \cite[Proposition 1.11]{MR1137737}; the third equivalence follows from Lemma~\ref{reduction}; and the fourth equivalence follows from \cite[Proof of Corollary III.9.10]{MR0463157}. Hence, it is equivalent to show $\mathbf{P}_{\bm{a}}(z)=\ell$ for generic $\bm{a} \in \mathbb{A}_{\mathbb{C}}^m \setminus Z(\bm{t})$. Note that Assumption~\ref{assum:key-assumption2} implies
\[
\mathbf{P}_{\bm{a'}}(z)=\ell.
\]
\begin{claim}\label{claim:min}
$\mathbf{P}_{\bm{a'}}(z)$ is minimal in $\mathscr{P}:=\{\mathbf{P}_{\bm{a}}(z): \bm{a} \in \mathbb{A}_{\mathbb{C}}^m \setminus \{Z(\bm{t}) \cup E\}\}$ with respect to $\succ$ (in the sense of Definition \ref{order>}).
\end{claim}
If the claim is true, then by the existence of the flattening stratification (see, e.g., \cite[Theorem 5.13]{MR2223407}), we know that
\[
S_\ell:=\{\bm{a} \in \mathbb{A}_{\mathbb{C}}^m \setminus \{Z(\bm{t}) \cup E\}: \mathbf{P}_{\bm{a}}(z)=\ell\} \subseteq \mathbb{A}^m_{\mathbb{C}} \setminus \{Z(\bm{t}) \cup E\} \subseteq \mathbb{A}^m_{\mathbb{C}}
\]
is a Zariski open dense subset. So we can simply take $U_m=S_\ell$. Thus, for \textbf{Step 1}, it remains to prove Claim~\ref{claim:min}.

\begin{proof}[Proof of Claim~\ref{claim:min}]
For any $\bm{a} \in \mathbb{A}_{\mathbb{C}}^m \setminus \{Z(\bm{t}) \cup E\}$, let $r:=\dim \overline{\mathscr{Q}}_{\bm{a}}$. By \cite[Proof of Corollary III.9.10]{MR0463157}, the following statements hold:
\begin{enumerate}
\item
If $r>0$, we have
\[
\mathbf{P}_{\bm{a}}(z)=\dfrac{\deg \overline{\mathscr{Q}}_{\bm{a}}}{r!}z^{r}+\text{lower-degree terms}.
\]
Hence, $\mathbf{P}_{\bm{a}}(n)>\mathbf{P}_{\bm{a'}}(n)=\ell$ for $n \gg 0, n \in \mathbb{Z}$; i.e., $\mathbf{P}_{\bm{a}}(z) \succ \mathbf{P}_{\bm{a'}}(z)$.
\item
If $r=0$, then by Assumption~\ref{assum:key-assumption1}, we have
\[
\mathbf{P}_{\bm{a}}(z)=\deg \overline{\mathscr{Q}}_{\bm{a}} = \deg \overline{\mathscr{Q}_{\bm{a}}}=\deg \mathscr{Q}_{\bm{a}} \geq \ell,
\]
since for the zero-dimensional affine variety $\mathscr{Q}_{\bm{a}}$, its degree is equal to the number of the solutions of its defining equations in $\mathbb{C}$ (counted with multiplicity). This implies that $\mathbf{P}_{\bm{a}}(n) \geq \mathbf{P}_{\bm{a'}}(n)=\ell$ for $n \in \mathbb{Z}$, i.e., $\mathbf{P}_{\bm{a}}(z) \succ \mathbf{P}_{\bm{a'}}(z)$.
\end{enumerate}
This concludes the proof of Claim~\ref{claim:min}, hence also \textbf{Step 1}.
\end{proof}
\paragraph{Step 2:}
We argue that the complement of $U_m$ in $\mathbb{C}^m$ has $\lambda_m^{\mathbb{C}}$-measure zero, i.e., $\lambda_m^{\mathbb{C}}(\mathbb{C}^m \setminus U_m)=0$.

Note that $Z:=\mathbb{A}_{\mathbb{C}}^m \setminus U_m \subsetneq \mathbb{A}_{\mathbb{C}}^m$ is a Zariski closed subset. Then we are done by Lemma~\ref{closed-0}.

Finally, these two steps conclude the proof.
\end{proof}
\begin{remark}
A few words about the assumptions:
\begin{enumerate}
\item 
For Assumption \ref{assum:key-assumption1}: In general, we cannot replace ``$E \subsetneq \mathbb{A}_{\mathbb{C}}^m$ is a Zariski closed subset'' by the more friendly phrase ``$E \subseteq \mathbb{C}^m$ is a subset of $\lambda_m^{\mathbb{C}}$-measure zero''. This is because we need to apply the existence of the flattening stratification \cite[Theorem 5.13]{MR2223407} to $\mathbb{A}_{\mathbb{C}}^m \setminus E$, so a subset $E \subseteq \mathbb{C}^m$ should, at least, be contained in some Zariski closed proper subset. But this is already false in the one-dimensional case: $\mathbb{N} \subseteq \mathbb{C}$ provides such a counterexample.
\item
For Assumption \ref{assum:key-assumption2}: Our condition $\bm{a}' \in \mathbb{C}^m \setminus Z(\bm{t})$ is only a sufficient condition for Lemma~\ref{reduction} (and hence Theorem~\ref{main-result}) to hold, since we use \cite[Theorem 2.1]{MR1854335}. Nevertheless it is a measure-zero condition. However, we emphasis that it is always necessary to exclude some ``bad'' parameters in Assumption \ref{assum:key-assumption2}. For example, the polynomial $(1-tx)x=0$ has a unique solution in $\mathbb{C}$ (counted with multiplicity) at $t=0$, but this does not hold generically. 
\end{enumerate}
\end{remark}

\subsubsection{Proof of Theorem~\ref{thm:general-case}}

The starting point is the lemma below:
\begin{lemma}[\cite{MR331643}, Lemma]\label{measure-0}
For any non-zero polynomial $f \in \mathbb{C}[x_1,\ldots,x_m]$, the set
$\mathbb{V}(f) \cap \mathbb{R}^m \subseteq \mathbb{R}^m $
is of $\lambda_m$-measure zero.
\end{lemma}
This lemma is well-known; however, for the convenience of readers, we repeat its proof here.
\begin{proof}[Proof of Lemma~\ref{measure-0}]
We induct on $m$. 

If $m=1$, then $\mathbb{V}(f) \subseteq \mathbb{C}$ is a finite set, and so is $\mathbb{V}(f) \cap \mathbb{R}$. Trivially, $\mathbb{V}(f) \cap \mathbb{R} \subset \mathbb{R}$ is of $\lambda_1$-measure zero.
Suppose the conclusion holds for $m-1$; we need to show that
\[
N_m:=\mathbb{V}(f) \cap \mathbb{R}^m=\{(x_1,\ldots,x_m) \in \mathbb{R}^m: f(x_1,\ldots,x_m)=0\} \subseteq \mathbb{R}^m
\]
is of $\lambda_m$-measure zero. Consider the cross-section of $N_m$ at the point $(x_1,\ldots,x_{m-1})$:
\[
C(x_1,\ldots,x_{m-1}):=\{x_m \in \mathbb{R}: f(x_1,\ldots,x_{m-1},x_m)=0\} \subseteq \mathbb{R}.
\]
It is a finite set, since $f$ is non-trivial. Let $d>0$ be the degree of $f$ with respect to the variable $x_m$. Write
\[
f(x_1,\ldots,x_m)=\sum_{i=0}^d f_i(x_1,\ldots,x_{m-1})x_m^i,
\]
and
\[
N_{m-1}:=\{(x_1,\ldots,x_{m-1}) \in \mathbb{R}^{m-1}: f_i(x_1,\ldots,x_{m-1})=0 \text{ for } i=0,\ldots,d\}=\bigcap_{i=0}^d \mathbb{V}(f_i) \cap \mathbb{R}^{m-1} \subseteq \mathbb{R}^{m-1}.
\]
By the induction hypothesis, if at least one of $\mathbb{V}(f_i) \cap \mathbb{R}^{m-1} \subseteq \mathbb{R}^{m-1}$ is of $\lambda_{m-1}$-measure zero, so $\lambda_{m-1}(N_{m-1})=0$. Now it follows that 
\begin{align*}
\lambda_m(N_m)&=\int_{\mathbb{R}^m} \mathbf{1}_{N_m}\ \mathrm{d}\lambda_m(x_1,\ldots,x_m) \\
&=\int_{\mathbb{R}^{m-1}} \left( \int_{\mathbb{R}} \mathbf{1}_{C(x_1,\ldots,x_{m-1})} \mathrm{d}\lambda_1(x_m) \right) \mathrm{d}\lambda_{m-1}(x_1,\ldots,x_{m-1}) \\
&=\int_{\mathbb{R}^{m-1}} \lambda_1(C(x_1,\ldots,x_{m-1}))\ \mathrm{d}\lambda_{m-1}(x_1,\ldots,x_{m-1}) \\
&=\int_{N_{m-1}} \lambda_1(C(x_1,\ldots,x_{m-1}))\ \mathrm{d}\lambda_{m-1}(x_1,\ldots,x_{m-1}) \\
& \quad + \int_{\mathbb{R}^{m-1} \setminus N_{m-1}} \lambda_1(C(x_1,\ldots,x_{m-1}))\ \mathrm{d}\lambda_{m-1}(x_1,\ldots,x_{m-1}),
\end{align*}
where the second equality follows directly from Fubini's theorem (see, e.g., \cite[Chapter 11.4.1, Theorem]{MR3445604}). The first term vanishes because of $\lambda_{m-1}(N_{m-1})=0$. The second vanishes since for any $(x_1,\ldots,x_{m-1}) \notin N_{m-1}$, the set $C(x_1,\ldots,x_{m-1})$ is a finite set, so $\lambda_1(C(x_1,\ldots,x_{m-1}))=0$. Thus, $\lambda_m(N_m)=0$.
\end{proof}
\begin{corollary}\label{closed-1}
The intersection of any Zariski closed proper subset of $\mathbb{A}_{\mathbb{C}}^m$ with $\mathbb{R}^m$ is of $\lambda_m$-measure zero.
\end{corollary}
\begin{proof}[Proof of Corollary~\ref{closed-1}]
Let $Z \subsetneq \mathbb{A}_{\mathbb{C}}^m$ be a Zariski closed subset. Then we can write 
\[
Z=\mathbb{V}(f_1,\ldots,f_r)=\bigcap_{i=1}^r \mathbb{V}(f_i),
\]
for some polynomials $f_i \in \mathbb{C}[x_1,\ldots,x_m]$. Since $Z \subsetneq \mathbb{A}_{\mathbb{C}}^m$, at least one of the $f_i$'s is non-zero, say $f_1$. By Lemma \ref{measure-0}, $\mathbb{V}(f_1) \cap \mathbb{R}^m \subseteq \mathbb{R}^m$ is of $\lambda_m$-measure zero. Hence, the subset $Z$ is again of $\lambda_m$-measure zero (since the Lebesgue measure on $\mathbb{R}^m$ is complete).
\end{proof}

\begin{proof}[Proof of Theorem~\ref{thm:general-case}]
We use the same notions in the proof of Theorem \ref{main-result}. It suffices to show that $Z \cap \mathbb{R}^m \subseteq \mathbb{R}^m$
is of $\lambda_m$-measure zero, which is true by Corollary \ref{closed-1}.
\end{proof}

\section{Supplementary material for Section~\ref{sec:examples}}\label{apn:examples_proofs}
In this section, we first provide the detailed proofs for mixtures of BTL models mentioned in Section \ref{sec:examples}, then provide the results and proofs for mixtures of MNL and PL models. We then provide some analysis regarding the tightness of $n$. Finally, we give an additional proof for mixtures of MNL models with 2-slate and 3-slate, whose identifiability has been previously proved by~\cite{chierichetti2018learning} and~\cite{tang2020learning}. Note that all the proofs for these models are using our general framework. 
\subsection{Mixtures of BTL models}
Here, we present the results for mixtures of BTL models.
\subsubsection{Unknown mixing probabilities}\label{apn-sub:examples_btl_proofs_p}
\begin{proof}[Proof of Theorem~\ref{prop-thm:btl_p}]
We first translate the system~\eqref{ex:BTL_2mixtures_pairwise_p} into the following equivalent system, with coefficients given by polynomials in $(\bm{a},\bm{b},p_1)$:
\begin{equation}\label{ex:BTL_2mixtures_pairwise_p_poly_ab}
\begin{cases}
x_1 = y_1 = 1, \\
(c_{ij}-d_{ij}) x_iy_i + (c_{ij}-pd_{ij})x_iy_j + (c_{ij}-(1-p)d_{ij})x_jy_i + c_{ij}x_j y_j = 0, & \forall i < j \in [n], \\
t_{ij}(x_i+x_j) = 1, \ h_{ij}(y_i+y_j) = 1, & \forall i < j \in [n],
\end{cases}
\end{equation}
where
\[
c_{ij}:=p_1a_i(b_i+b_j)+(1-p_1)b_i(a_i+a_j),d_{ij}:=(a_i+a_j)(b_i+b_j).
\]
This is an equation system in the variables $(\bm{x},\bm{y},p,\bm{t},\bm{h})$ and with coefficients $(\bm{a},\bm{b},p_1)$. Indeed, as a result of introducing new variables and equations to the system~\eqref{ex:BTL_2mixtures_pairwise_p} to prevent the denominators from being zero, it follows that for any $(\bm{a}_{2:n},\bm{b}_{2:n},p_1)$, the equation system~\eqref{ex:BTL_2mixtures_pairwise_p} is equivalent to the equation system~\eqref{ex:BTL_2mixtures_pairwise_p_poly_ab}. In particular, the system~\eqref{ex:BTL_2mixtures_pairwise_p} has exactly two solutions in $\mathbb{C}$ (counted with multiplicity) if and only if the system~\eqref{ex:BTL_2mixtures_pairwise_p_poly_ab} has exactly two solutions in $\mathbb{C}$ (counted with multiplicity).

Note that the system~\eqref{ex:BTL_2mixtures_pairwise_p_poly_ab} has the following two (distinct) solutions in $\mathbb{C}$:
\[
(\bm{x},\bm{y},p,\bm{t},\bm{h})=\left(\bm{a},\bm{b},p_1,\left(\dfrac{1}{a_i+a_j}\right)_{i,j},\left(\dfrac{1}{b_i+b_j}\right)_{i,j}\right),\left(\bm{b},\bm{a},1-p_1,\left(\dfrac{1}{b_i+b_j}\right)_{i,j},\left(\dfrac{1}{a_i+a_j}\right)_{i,j}\right),
\]
for all $(\bm{a}_{2:n},\bm{b}_{2:n},p_1) \in \mathbb{C}^{2n-1} \setminus E_n$, where $E_n \subsetneq \mathbb{A}_{\mathbb{C}}^{2n-1}$ is the Zariski closed subset defined by
\[
E_n:=\left\lbrace (\bm{a}_{2:n},\bm{b}_{2:n},p_1) \in \mathbb{A}_{\mathbb{C}}^{2n-1}: p_1-0.5=0 \text{ or } a_i+a_j=0 \text{ or } b_i+b_j=0 \right\rbrace \subseteq \mathbb{A}_{\mathbb{C}}^{2n-1}.
\]
As before, this proposition can be proved by induction on $n$:
\paragraph{Case $n=5$.}
We consider the following subset of equations in the system~\eqref{ex:BTL_2mixtures_pairwise_p_poly_ab}:
\begin{equation}\label{ex:BTL_2mixtures_pairwise_p_poly_ab_n5}
\begin{cases}
x_1 = y_1 = 1, \\
(c_{ij}-d_{ij}) x_iy_i + (c_{ij}-pd_{ij})x_iy_j + (c_{ij}-(1-p)d_{ij})x_jy_i + c_{ij}x_j y_j = 0, & \forall i < j \in [5], \\
t_{23}(x_2+x_3) = 1, \ t_{15}(x_1+x_5) = 1, \ h_{23}(y_2+y_3) = 1.
\end{cases}
\end{equation}
It is clear that the system~\eqref{ex:BTL_2mixtures_pairwise_p_poly_ab} has exactly two solutions in $\mathbb{C}$ (counted with multiplicity) if the system~\eqref{ex:BTL_2mixtures_pairwise_p_poly_ab_n5} has exactly two solutions in $\mathbb{C}$ (counted with multiplicity). Hence, it suffices to consider the system~\eqref{ex:BTL_2mixtures_pairwise_p_poly_ab_n5}. Based on Theorem~\ref{thm:general-case}, it suffices to check Assumptions~\ref{assum:key-assumption1} and~\ref{assum:key-assumption2} for the system~\eqref{ex:BTL_2mixtures_pairwise_p_poly_ab_n5}.

For Assumption~\ref{assum:key-assumption1}: The system~\eqref{ex:BTL_2mixtures_pairwise_p_poly_ab_n5} has at least two solutions in $\mathbb{C}$ for all $(\bm{a}_{2:5},\bm{b}_{2:5},p_1) \in \mathbb{C}^{9} \setminus E_{5}$.

For Assumption~\ref{assum:key-assumption2}: We need to find $(\bm{a'}_{1:5},\bm{b'}_{1:5},p'_1)$ such that
\begin{enumerate}
\item 
$(\bm{a'}_{2:5},\bm{b'}_{2:5},p'_1) \in \mathbb{C}^{9} \setminus Z(\bm{a}_{2:5},\bm{b}_{2:5},p_1)$, and
\item
the associated equation system~\eqref{ex:BTL_2mixtures_pairwise_p_poly_ab_n5} has exactly two solutions in $\mathbb{C}$ (counted with multiplicity).
\end{enumerate}
Both of these statements can be verified using Magma. We choose $(\bm{a'}_{1:5},\bm{b'}_{1:5},p'_1)=(1,2,3,4,5;1,8,9,3,2;0.3)$.

For the first one, we need to determine $\mathrm{Bad}(\bm{a}_{2:5},\bm{b}_{2:5},p_1)$. This can be done via Magma.
\begin{lstlisting}[language=Magma,caption=Gr\"{o}nber basis of BTL models with unknown mixing probabilities,label=lst_btl_ab_p_Groebner,frame=single,basicstyle=\scriptsize]
P<x1,x2,x3,x4,x5,y1,y2,y3,y4,y5,p,t23,h15,h23,
  a1,a2,a3,a4,a5,b1,b2,b3,b4,b5,p1>:=FreeAlgebra(Rationals(),25,"lex");

I:=ideal<P|x1-a1,y1-b1,a1-1,b1-1,
(p1*a1*(b1+b2)+(1-p1)*b1*(a1+a2)-(a1+a2)*(b1+b2))*x1*y1+
(p1*a1*(b1+b2)+(1-p1)*b1*(a1+a2)-p*(a1+a2)*(b1+b2))*x1*y2+
(p1*a1*(b1+b2)+(1-p1)*b1*(a1+a2)-(1-p)*(a1+a2)*(b1+b2))*x2*y1-
(p1*a1*(b1+b2)+(1-p1)*b1*(a1+a2))*x2*y2,
(p1*a1*(b1+b3)+(1-p1)*b1*(a1+a3)-(a1+a3)*(b1+b3))*x1*y1+
(p1*a1*(b1+b3)+(1-p1)*b1*(a1+a3)-p*(a1+a3)*(b1+b3))*x1*y3+
(p1*a1*(b1+b3)+(1-p1)*b1*(a1+a3)-(1-p)*(a1+a3)*(b1+b3))*x3*y1-
(p1*a1*(b1+b3)+(1-p1)*b1*(a1+a3))*x3*y3,
(p1*a1*(b1+b4)+(1-p1)*b1*(a1+a4)-(a1+a4)*(b1+b4))*x1*y1+
(p1*a1*(b1+b4)+(1-p1)*b1*(a1+a4)-p*(a1+a4)*(b1+b4))*x1*y4+
(p1*a1*(b1+b4)+(1-p1)*b1*(a1+a4)-(1-p)*(a1+a4)*(b1+b4))*x4*y1-
(p1*a1*(b1+b4)+(1-p1)*b1*(a1+a4))*x4*y4,
(p1*a1*(b1+b5)+(1-p1)*b1*(a1+a5)-(a1+a5)*(b1+b5))*x1*y1+
(p1*a1*(b1+b5)+(1-p1)*b1*(a1+a5)-p*(a1+a5)*(b1+b5))*x1*y5+
(p1*a1*(b1+b5)+(1-p1)*b1*(a1+a5)-(1-p)*(a1+a5)*(b1+b5))*x5*y1-
(p1*a1*(b1+b5)+(1-p1)*b1*(a1+a5))*x5*y5,
(p1*a2*(b2+b3)+(1-p1)*b2*(a2+a3)-(a2+a3)*(b2+b3))*x2*y2+
(p1*a2*(b2+b3)+(1-p1)*b2*(a2+a3)-p*(a2+a3)*(b2+b3))*x2*y3+
(p1*a2*(b2+b3)+(1-p1)*b2*(a2+a3)-(1-p)*(a2+a3)*(b2+b3))*x3*y2-
(p1*a2*(b2+b3)+(1-p1)*b2*(a2+a3))*x3*y3,
(p1*a2*(b2+b4)+(1-p1)*b2*(a2+a4)-(a2+a4)*(b2+b4))*x2*y2+
(p1*a2*(b2+b4)+(1-p1)*b2*(a2+a4)-p*(a2+a4)*(b2+b4))*x2*y4+
(p1*a2*(b2+b4)+(1-p1)*b2*(a2+a4)-(1-p)*(a2+a4)*(b2+b4))*x4*y2-
(p1*a2*(b2+b4)+(1-p1)*b2*(a2+a4))*x4*y4,
(p1*a2*(b2+b5)+(1-p1)*b2*(a2+a5)-(a2+a5)*(b2+b5))*x2*y2+
(p1*a2*(b2+b5)+(1-p1)*b2*(a2+a5)-p*(a2+a5)*(b2+b5))*x2*y5+
(p1*a2*(b2+b5)+(1-p1)*b2*(a2+a5)-(1-p)*(a2+a5)*(b2+b5))*x5*y2-
(p1*a2*(b2+b5)+(1-p1)*b2*(a2+a5))*x5*y5,
(p1*a3*(b3+b4)+(1-p1)*b3*(a3+a4)-(a3+a4)*(b3+b4))*x3*y3+
(p1*a3*(b3+b4)+(1-p1)*b3*(a3+a4)-p*(a3+a4)*(b3+b4))*x3*y4+
(p1*a3*(b3+b4)+(1-p1)*b3*(a3+a4)-(1-p)*(a3+a4)*(b3+b4))*x4*y3-
(p1*a3*(b3+b4)+(1-p1)*b3*(a3+a4))*x4*y4,
(p1*a3*(b3+b5)+(1-p1)*b3*(a3+a5)-(a3+a5)*(b3+b5))*x3*y3+
(p1*a3*(b3+b5)+(1-p1)*b3*(a3+a5)-p*(a3+a5)*(b3+b5))*x3*y5+
(p1*a3*(b3+b5)+(1-p1)*b3*(a3+a5)-(1-p)*(a3+a5)*(b3+b5))*x5*y3-
(p1*a3*(b3+b5)+(1-p1)*b3*(a3+a5))*x5*y5,
(p1*a4*(b4+b5)+(1-p1)*b4*(a4+a5)-(a4+a5)*(b4+b5))*x4*y4+
(p1*a4*(b4+b5)+(1-p1)*b4*(a4+a5)-p*(a4+a5)*(b4+b5))*x4*y5+
(p1*a4*(b4+b5)+(1-p1)*b4*(a4+a5)-(1-p)*(a4+a5)*(b4+b5))*x5*y4-
(p1*a4*(b4+b5)+(1-p1)*b4*(a4+a5))*x5*y5,
t23*(x2+x3)-1,
h15*(y1+y5)-1,
h23*(y2+y3)-1>;

-> GroebnerBasis(I);
\end{lstlisting}
From the Magma output, we obtain
\begin{equation}\label{BTL_bad_p}
\mathrm{Bad}(\bm{a}_{2:5},\bm{b}_{2:5},p_1)=
\begin{cases}
a_i(1+a_i)(1+b_i),a_i (p_1-1) - b_i p_1-1, & \forall i \in [5], \\ 
b_i(a_i+1-p_1) + a_i p_1-1, b_i(1-p_1)+ a_i (b_i+ p_1-1)-1, & \forall i \in [5], \\
(a_i+a_j)(b_i+b_j), & \forall i < j \in [5], \\
a_j(b_j+b_ip_1)+a_ib_j(1-p_1), & \forall i \neq j \in [5],
\end{cases}
\end{equation}
and verify that $(\bm{a'}_{2:5},\bm{b'}_{2:5},p'_1) \in \mathbb{C}^{9} \setminus Z(\bm{a}_{2:5},\bm{b}_{2:5},p_1)$.

For the second one, we use Magma to check whether~\eqref{ex:BTL_2mixtures_pairwise_p_poly_ab_n5} has exactly two solutions in $\mathbb{C}$ (counted with multiplicity) for this $(\bm{a'}_{1:5},\bm{b'}_{1:5},p'_1)$.
\begin{lstlisting}[language=Magma,caption=Dimension and degree computations of BTL models with unknown mixing probabilities,label=lst_btl_abp,frame=single,basicstyle=\scriptsize]
a:=[1,2,3,4,5];
b:=[1,8,9,3,2];
p1:=3/10;
p2:=7/10;

k:=Rationals();
A<x1,x2,x3,x4,x5,y1,y2,y3,y4,y5,t23,h15,h23,p>:=AffineSpace(k,14);
P:=Scheme(A,[
x1-1,
y1-1,
(p1*a[1]*(b[1]+b[2])+(1-p1)*b[1]*(a[1]+a[2])-(a[1]+a[2])*(b[1]+b[2]))*x1*y1
+(p1*a[1]*(b[1]+b[2])+(1-p1)*b[1]*(a[1]+a[2])-p*(a[1]+a[2])*(b[1]+b[2]))*x1*y2
+(p1*a[1]*(b[1]+b[2])+(1-p1)*b[1]*(a[1]+a[2])-(1-p)*(a[1]+a[2])*(b[1]+b[2]))*x2*y1
+(p1*a[1]*(b[1]+b[2])+(1-p1)*b[1]*(a[1]+a[2]))*x2*y2,
(p1*a[1]*(b[1]+b[3])+(1-p1)*b[1]*(a[1]+a[3])-(a[1]+a[3])*(b[1]+b[3]))*x1*y1
+(p1*a[1]*(b[1]+b[3])+(1-p1)*b[1]*(a[1]+a[3])-p*(a[1]+a[3])*(b[1]+b[3]))*x1*y3
+(p1*a[1]*(b[1]+b[3])+(1-p1)*b[1]*(a[1]+a[3])-(1-p)*(a[1]+a[3])*(b[1]+b[3]))*x3*y1
+(p1*a[1]*(b[1]+b[3])+(1-p1)*b[1]*(a[1]+a[3]))*x3*y3,
(p1*a[1]*(b[1]+b[4])+(1-p1)*b[1]*(a[1]+a[4])-(a[1]+a[4])*(b[1]+b[4]))*x1*y1
+(p1*a[1]*(b[1]+b[4])+(1-p1)*b[1]*(a[1]+a[4])-p*(a[1]+a[4])*(b[1]+b[4]))*x1*y4
+(p1*a[1]*(b[1]+b[4])+(1-p1)*b[1]*(a[1]+a[4])-(1-p)*(a[1]+a[4])*(b[1]+b[4]))*x4*y1
+(p1*a[1]*(b[1]+b[4])+(1-p1)*b[1]*(a[1]+a[4]))*x4*y4,
(p1*a[1]*(b[1]+b[5])+(1-p1)*b[1]*(a[1]+a[5])-(a[1]+a[5])*(b[1]+b[5]))*x1*y1
+(p1*a[1]*(b[1]+b[5])+(1-p1)*b[1]*(a[1]+a[5])-p*(a[1]+a[5])*(b[1]+b[5]))*x1*y5
+(p1*a[1]*(b[1]+b[5])+(1-p1)*b[1]*(a[1]+a[5])-(1-p)*(a[1]+a[5])*(b[1]+b[5]))*x5*y1
+(p1*a[1]*(b[1]+b[5])+(1-p1)*b[1]*(a[1]+a[5]))*x5*y5,
(p1*a[2]*(b[2]+b[3])+(1-p1)*b[2]*(a[2]+a[3])-(a[2]+a[3])*(b[2]+b[3]))*x2*y2
+(p1*a[2]*(b[2]+b[3])+(1-p1)*b[2]*(a[2]+a[3])-p*(a[2]+a[3])*(b[2]+b[3]))*x2*y3
+(p1*a[2]*(b[2]+b[3])+(1-p1)*b[2]*(a[2]+a[3])-(1-p)*(a[2]+a[3])*(b[2]+b[3]))*x3*y2
+(p1*a[2]*(b[2]+b[3])+(1-p1)*b[2]*(a[2]+a[3]))*x3*y3,
(p1*a[2]*(b[2]+b[4])+(1-p1)*b[2]*(a[2]+a[4])-(a[2]+a[4])*(b[2]+b[4]))*x2*y2
+(p1*a[2]*(b[2]+b[4])+(1-p1)*b[2]*(a[2]+a[4])-p*(a[2]+a[4])*(b[2]+b[4]))*x2*y4
+(p1*a[2]*(b[2]+b[4])+(1-p1)*b[2]*(a[2]+a[4])-(1-p)*(a[2]+a[4])*(b[2]+b[4]))*x4*y2
+(p1*a[2]*(b[2]+b[4])+(1-p1)*b[2]*(a[2]+a[4]))*x4*y4,
(p1*a[2]*(b[2]+b[5])+(1-p1)*b[2]*(a[2]+a[5])-(a[2]+a[5])*(b[2]+b[5]))*x2*y2
+(p1*a[2]*(b[2]+b[5])+(1-p1)*b[2]*(a[2]+a[5])-p*(a[2]+a[5])*(b[2]+b[5]))*x2*y5
+(p1*a[2]*(b[2]+b[5])+(1-p1)*b[2]*(a[2]+a[5])-(1-p)*(a[2]+a[5])*(b[2]+b[5]))*x5*y2
+(p1*a[2]*(b[2]+b[5])+(1-p1)*b[2]*(a[2]+a[5]))*x5*y5,
(p1*a[3]*(b[3]+b[4])+(1-p1)*b[3]*(a[3]+a[4])-(a[3]+a[4])*(b[3]+b[4]))*x3*y3
+(p1*a[3]*(b[3]+b[4])+(1-p1)*b[3]*(a[3]+a[4])-p*(a[3]+a[4])*(b[3]+b[4]))*x3*y4
+(p1*a[3]*(b[3]+b[4])+(1-p1)*b[3]*(a[3]+a[4])-(1-p)*(a[3]+a[4])*(b[3]+b[4]))*x4*y3
+(p1*a[3]*(b[3]+b[4])+(1-p1)*b[3]*(a[3]+a[4]))*x4*y4,
(p1*a[3]*(b[3]+b[5])+(1-p1)*b[3]*(a[3]+a[5])-(a[3]+a[5])*(b[3]+b[5]))*x3*y3
+(p1*a[3]*(b[3]+b[5])+(1-p1)*b[3]*(a[3]+a[5])-p*(a[3]+a[5])*(b[3]+b[5]))*x3*y5
+(p1*a[3]*(b[3]+b[5])+(1-p1)*b[3]*(a[3]+a[5])-(1-p)*(a[3]+a[5])*(b[3]+b[5]))*x5*y3
+(p1*a[3]*(b[3]+b[5])+(1-p1)*b[3]*(a[3]+a[5]))*x5*y5,
(p1*a[4]*(b[4]+b[5])+(1-p1)*b[4]*(a[4]+a[5])-(a[4]+a[5])*(b[4]+b[5]))*x4*y4
+(p1*a[4]*(b[4]+b[5])+(1-p1)*b[4]*(a[4]+a[5])-p*(a[4]+a[5])*(b[4]+b[5]))*x4*y5
+(p1*a[4]*(b[4]+b[5])+(1-p1)*b[4]*(a[4]+a[5])-(1-p)*(a[4]+a[5])*(b[4]+b[5]))*x5*y4
+(p1*a[4]*(b[4]+b[5])+(1-p1)*b[4]*(a[4]+a[5]))*x5*y5,
t23*(x2+x3)-1,
h15*(y1+y5)-1,
h23*(y2+y3)-1
]);

-> Dimension(P);
-> @0@

-> Degree(P);
-> @2@
\end{lstlisting}
From Listing~\ref{lst_btl_abp}, $\texttt{Dimension(P)=0}$ and $\texttt{Degree(P)=2}$ means that the system~\eqref{ex:BTL_2mixtures_pairwise_p_poly_ab_n5} has exactly two solutions in $\mathbb{C}$ (counted with multiplicity) for this choice of $(\bm{a'}_{1:5},\bm{b'}_{1:5},p'_1)$. Thus, we have checked Assumption~\ref{assum:key-assumption2} for the system~\eqref{ex:BTL_2mixtures_pairwise_p_poly_ab_n5}, which concludes the proof for $n=5$.

\paragraph{Case $n \geq 6$.}
We may employ a same procedure as we did in the proof of Proposition~\ref{prop:btl}. The idea is to split the equation system~\eqref{ex:BTL_2mixtures_pairwise_p} into two parts. One only involves the variables $(\bm{x}_{1:n-1},\bm{y}_{1:n-1},p)$, for which we have exactly two solutions in $\mathbb{C}$ (counted with multiplicity) generically, using the induction step. The other part is the remaining equations in~\eqref{ex:BTL_2mixtures_pairwise_p}, from which we can build up a system of linear equations in $(x_n, y_n)$ with each of the two solutions of $(\bm{x}_{1:n-1},\bm{y}_{1:n-1},p)$. Using basic linear algebra, we conclude that generically (i.e., outside the set of parameters that make the coefficient matrix not full rank), $(x_n,y_n)$ has a unique solution in $\mathbb{C}$ (counted with multiplicity). Altogether, this shows that the system~\eqref{ex:BTL_2mixtures_pairwise_p} has two solutions in $\mathbb{C}$ (counted with multiplicity) and hence conclude the proof for the case $n \geq 6$.
\end{proof}

\subsubsection{Known mixing probabilities}\label{apn-sub:examples_btl_proofs}
\begin{proof}[Proof of Proposition~\ref{prop:btl}]
This is proved by induction on $n$.
\paragraph{Case $n=5$.}
In this case, we can expand the system~\eqref{ex:BTL_2mixtures_pairwise} so that all coefficients are given by polynomials in $(\bm{a},\bm{b})$:
\begin{equation}\label{ex:BTL_2mixtures_pairwise_poly_ab_n=5}
\begin{cases}
x_1= y_1=1, \\
(p_1a_j(b_i+b_j)+p_2b_j(a_i+a_j))x_iy_i - (p_1a_i(b_i+b_j)+p_2b_i(a_i+a_j))x_j y_j+ \\ (p_1a_j(b_i+b_j)-p_2b_i(a_i+a_j))x_iy_j - (p_1a_i(b_i+b_j)-p_2b_j(a_i+a_j))x_jy_i  = 0, & \forall i < j \in [5].
\end{cases}
\end{equation}
Note this this is not a faithful transformation of the system~\eqref{ex:BTL_2mixtures_pairwise} (but can only increase the number of solutions). To proceed, we need to determine $Z(\bm{a}_{2:5},\bm{b}_{2:5})$ introduced in equation~\eqref{defn:Z(t)}. This can be done by Magma.

\lstset{  
  basicstyle=\footnotesize,
  breaklines=true,
  postbreak=\mbox{\textcolor{red}{$\hookrightarrow$}\space},
}

\begin{lstlisting}[language=Magma,caption=Gr\"{o}nber basis of BTL models with unknown mixing probabilities,label=lst_btl_ab,frame=single,basicstyle=\scriptsize]
P<x1,x2,x3,x4,x5,y1,y2,y3,y4,y5,
  a1,a2,a3,a4,a5,b1,b2,b3,b4,b5>:=FreeAlgebra(Rationals(),20,"lex");

I:=ideal<P|x1-1,y1-1,a1-1,b1-1,
(a1*b1+3/10*a2*b1+7/10*a1*b2)*(x1+x2)*(y1+y2)+
(a1+a2)*(b1+b2)*(-3/10*(x1+x2)*y1-7/10*x1*(y1+y2)),
(a1*b1+3/10*a3*b1+7/10*a1*b3)*(x1+x3)*(y1+y3)+
(a1+a3)*(b1+b3)*(-3/10*(x1+x3)*y1-7/10*x1*(y1+y3)),
(a1*b1+3/10*a4*b1+7/10*a1*b4)*(x1+x4)*(y1+y4)+
(a1+a4)*(b1+b4)*(-3/10*(x1+x4)*y1-7/10*x1*(y1+y4)),
(a1*b1+3/10*a5*b1+7/10*a1*b5)*(x1+x5)*(y1+y5)+
(a1+a5)*(b1+b5)*(-3/10*(x1+x5)*y1-7/10*x1*(y1+y5)),
(a2*b2+3/10*a3*b2+7/10*a2*b3)*(x2+x3)*(y2+y3)+
(a2+a3)*(b2+b3)*(-3/10*(x2+x3)*y2-7/10*x2*(y2+y3)),
(a2*b2+3/10*a4*b2+7/10*a2*b4)*(x2+x4)*(y2+y4)+
(a2+a4)*(b2+b4)*(-3/10*(x2+x4)*y2-7/10*x2*(y2+y4)),
(a2*b2+3/10*a5*b2+7/10*a2*b5)*(x2+x5)*(y2+y5)+
(a2+a5)*(b2+b5)*(-3/10*(x2+x5)*y2-7/10*x2*(y2+y5)),
(a3*b3+3/10*a4*b3+7/10*a3*b4)*(x3+x4)*(y3+y4)+
(a3+a4)*(b3+b4)*(-3/10*(x3+x4)*y3-7/10*x3*(y3+y4)),
(a3*b3+3/10*a5*b3+7/10*a3*b5)*(x3+x5)*(y3+y5)+
(a3+a5)*(b3+b5)*(-3/10*(x3+x5)*y3-7/10*x3*(y3+y5)),
(a4*b4+3/10*a5*b4+7/10*a4*b5)*(x4+x5)*(y4+y5)+
(a4+a5)*(b4+b5)*(-3/10*(x4+x5)*y4-7/10*x4*(y4+y5))>;

-> GroebnerBasis(I);
\end{lstlisting}
From the output, we obtain
\begin{equation}\label{BTL_bad}
\mathrm{Bad}(\bm{a}_{2:5},\bm{b}_{2:5})=\{3a_ib_i-4a_jb_i-7a_jb_j,10a_ib_i+7a_ib_j+3a_jb_i: \forall i \neq j \in [5]\}.
\end{equation}
Based on Theorem~\ref{thm:general-case}, we claim by checking Assumptions~\ref{assum:key-assumption1} and~\ref{assum:key-assumption2} that the system~\eqref{ex:BTL_2mixtures_pairwise_poly_ab_n=5} has exactly three solutions in $\mathbb{C}$ (counted with multiplicity) for all $(\bm{a}_{2:5},\bm{b}_{2:5}) \in Q_{BTL}^{8}$ but a set of $\lambda_{8}$-measure zero.

Assumption~\ref{assum:key-assumption1}: This is clear, since
\begin{gather}\label{3-solutions}
(\bm{x},\bm{y})=(\bm{a}_{1:5},\bm{b}_{1:5}),\left(a_1,0,0,0,0;b_1,\left(\dfrac{1-\eta_{1j}}{\eta_{1j}-p_1}\right)_{j=2,\ldots,5}\right),\left(a_1,\left(\dfrac{1-\eta_{1j}}{\eta_{1j}-p_2}\right)_{j=2,\ldots,5};b_1,0,0,0,0\right)
\end{gather}
are three (distinct) solutions of the system~\eqref{ex:BTL_2mixtures_pairwise_poly_ab_n=5} for all $(\bm{a}_{2:5},\bm{b}_{2:5}) \in \mathbb{C}^{8} \setminus E$, where $E \subsetneq \mathbb{A}_{\mathbb{C}}^{8}$ is the Zariski closed subset defined by
\begin{equation}\label{E-for-BTL}
E:=\bigcup_{j=2}^5 \text{zero-set}(\eta_{1j}(\bm{a},\bm{b})-p_1) \cup \bigcup_{j=2}^5 \text{zero-set}(\eta_{1j}(\bm{a},\bm{b})-p_2) \cup \text{zero-set}(1-\eta_{12}(\bm{a},\bm{b})) \subseteq \mathbb{A}_{\mathbb{C}}^{8}.
\end{equation}
Assumption~\ref{assum:key-assumption2}: Let $(\bm{a'}_{1:5},\bm{b'}_{1:5})=(1,2,3,4,5;1,8,9,3,2)$. It is routine to check that $(\bm{a'}_{2:5},\bm{b'}_{2:5}) \in \mathbb{C}^{8} \setminus Z(\bm{a}_{2:5},\bm{b}_{2:5})$ using equations~\eqref{BTL_bad} and~\eqref{E-for-BTL}. Since the associated equation system~\eqref{ex:BTL_2mixtures_pairwise_poly_ab_n=5} has $\mathbb{Q}$-coefficients, we can use Magma  to check whether it has exactly three solutions in $\mathbb{C}$ (counted with multiplicity) for this $(\bm{a}'_{1:5},\bm{b}'_{1:5})$. 
\begin{lstlisting}[language=Magma,caption=Dimension and degree computations of BTL models with known mixing probabilities,label=lst_btl_ab_5,frame=single,basicstyle=\scriptsize]
a:=[1,2,3,4,5];
b:=[1,8,9,3,2];
p1:=7/10;

k:=Rationals();
A<x1,x2,x3,x4,x5,y1,y2,y3,y4,y5>:=AffineSpace(k,10);
P:=Scheme(A,[
x1-1,
y1-1,
(a[1]*b[1]+p1*a[2]*b[1]+(1-p1)*a[1]*b[2])*(x1+x2)*(y1+y2)+
(a[1]+a[2])*(b[1]+b[2])*(-p1*(x1+x2)*y1-(1-p1)*x1*(y1+y2)),
(a[1]*b[1]+p1*a[3]*b[1]+(1-p1)*a[1]*b[3])*(x1+x3)*(y1+y3)+
(a[1]+a[3])*(b[1]+b[3])*(-p1*(x1+x3)*y1-(1-p1)*x1*(y1+y3)),
(a[1]*b[1]+p1*a[4]*b[1]+(1-p1)*a[1]*b[4])*(x1+x4)*(y1+y4)+
(a[1]+a[4])*(b[1]+b[4])*(-p1*(x1+x4)*y1-(1-p1)*x1*(y1+y4)),
(a[1]*b[1]+p1*a[5]*b[1]+(1-p1)*a[1]*b[5])*(x1+x5)*(y1+y5)+
(a[1]+a[5])*(b[1]+b[5])*(-p1*(x1+x5)*y1-(1-p1)*x1*(y1+y5)),
(a[2]*b[2]+p1*a[3]*b[2]+(1-p1)*a[2]*b[3])*(x2+x3)*(y2+y3)+
(a[2]+a[3])*(b[2]+b[3])*(-p1*(x2+x3)*y2-(1-p1)*x2*(y2+y3)),
(a[2]*b[2]+p1*a[4]*b[2]+(1-p1)*a[2]*b[4])*(x2+x4)*(y2+y4)+
(a[2]+a[4])*(b[2]+b[4])*(-p1*(x2+x4)*y2-(1-p1)*x2*(y2+y4)),
(a[2]*b[2]+p1*a[5]*b[2]+(1-p1)*a[2]*b[5])*(x2+x5)*(y2+y5)+
(a[2]+a[5])*(b[2]+b[5])*(-p1*(x2+x5)*y2-(1-p1)*x2*(y2+y5)),
(a[3]*b[3]+p1*a[4]*b[3]+(1-p1)*a[3]*b[4])*(x3+x4)*(y3+y4)+
(a[3]+a[4])*(b[3]+b[4])*(-p1*(x3+x4)*y3-(1-p1)*x3*(y3+y4)),
(a[3]*b[3]+p1*a[5]*b[3]+(1-p1)*a[3]*b[5])*(x3+x5)*(y3+y5)+
(a[3]+a[5])*(b[3]+b[5])*(-p1*(x3+x5)*y3-(1-p1)*x3*(y3+y5)),
(a[4]*b[4]+p1*a[5]*b[4]+(1-p1)*a[4]*b[5])*(x4+x5)*(y4+y5)+
(a[4]+a[5])*(b[4]+b[5])*(-p1*(x4+x5)*y4-(1-p1)*x4*(y4+y5))
]);

-> Dimension(P);
-> @0@

-> Degree(P);
-> @3@
\end{lstlisting}
From Listing~\ref{lst_btl_ab_5}, $\texttt{Dimension(P)=0}$ and $\texttt{Degree(P)=3}$ means the system~\eqref{ex:BTL_2mixtures_pairwise_poly_ab_n=5} has exactly three solutions in $\mathbb{C}$ (counted with multiplicity) for this choice of $(\bm{a'}_{2:5},\bm{b'}_{2:5})$. 

Thus, by Theorem~\ref{thm:general-case}, the system~\eqref{ex:BTL_2mixtures_pairwise_poly_ab_n=5} has exactly three solutions in $\mathbb{C}$ (counted with multiplicity) for all $(\bm{a}_{2:5},\bm{b}_{2:5}) \in Q_{BTL}^{8}$ but a set $V_5$ of $\lambda_{8}$-measure zero. Replacing $V_5$ by $V_5 \cup E$ if necessary (since $E \subseteq \mathbb{C}^8$ is of $\lambda_{8}$-measure zero), we may assume $E \subseteq V_5$. This implies that for any $(\bm{a}_{2:5},\bm{b}_{2:5}) \in Q_{BTL}^{8} \setminus V_5$, the three solutions of the system~\eqref{ex:BTL_2mixtures_pairwise_poly_ab_n=5} are necessarily given by equation~\eqref{3-solutions}, of which the first is always a solution of~\eqref{ex:BTL_2mixtures_pairwise}, while the last two are not allowed by the system~\eqref{ex:BTL_2mixtures_pairwise}. Altogether, this shows that the system~\eqref{ex:BTL_2mixtures_pairwise} has a unique solution in $\mathbb{C}$ (counted with multiplicity) for all $(\bm{a}_{2:5},\bm{b}_{2:5}) \in Q_{BTL}^{8}$ but a set $V_5$ of $\lambda_{8}$-measure zero. 

\paragraph{Case $n \geq 5$.}
Suppose the conclusion holds for $n-1$. We need to prove that
\begin{gather}\label{equ:ind}
\begin{cases}
x_1=y_1=1 \\
p_1\dfrac{x_i}{x_i+x_j}+p_2\dfrac{y_i}{y_i+y_j}=\eta_{ij},\ \forall i<j \in [n]
\end{cases}
\end{gather}
has a unique solution in $\mathbb{C}$ (counted with multiplicity) for all $(\bm{a}_{2:n},\bm{b}_{2:n}) \in Q_{BTL}^{2n-2}$ but a set $V_n$ of $\lambda_{2n-2}$-measure zero. We begin by splitting the system~\eqref{equ:ind} into two parts:
\begin{gather}\label{equ:ind1}
\begin{cases}
x_1=y_1=1 \\
p_1\dfrac{x_i}{x_i+x_j}+p_2\dfrac{y_i}{y_i+y_j}=\eta_{ij}, & \forall i<j \in [n-1]
\end{cases}
\end{gather}
and
\begin{gather}\label{equ:ind2}
p_1\dfrac{x_i}{x_i+x_n}+p_2\dfrac{y_i}{y_i+y_n}=\eta_{in}, \quad \forall i=1,\ldots,n-1.
\end{gather}
By the induction hypothesis, there exists a $\lambda_{2n-4}$-measure zero subset $V_{n-1} \subseteq \mathbb{C}^{2n-4}$ such that the system~\eqref{equ:ind1} has a unique solution in $\mathbb{C}$ (counted with multiplicity) for any $(\bm{a}_{2:n-1},\bm{b}_{2:n-1}) \in Q_{BTL}^{2n-4} \setminus V_{n-1}$, given by
\[
(\bm{x}_{1:n-1},\bm{y}_{1:n-1})=(\bm{a}_{1:n-1},\bm{b}_{1:n-1}).
\]
Plugging this solution into equation~\eqref{equ:ind2} and simplifying, we obtain
\begin{gather}\label{equ:ind4}
(\eta_{in}-1)a_ib_i+(\eta_{in}-p_1)a_iy_n+(\eta_{in}-p_2)b_ix_n+\eta_{in}x_ny_n=0, \quad \forall i=1,\ldots,n-1.
\end{gather}
Since $\eta_{1n} \neq 0$, by the case of $i=1$ in equation~\eqref{equ:ind4}, we have
\[
x_ny_n=\dfrac{(\eta_{1n}-1)a_1b_1+(\eta_{1n}-p_1)a_1y_n+(\eta_{1n}-p_2)b_1x_n}{\eta_{1n}}.
\]
Plugging this into the cases of $i=2$ and $3$ in equation~\eqref{equ:ind4}, we obtain the following system of linear equations in $(x_n,y_n)$:
\begin{equation}\label{equ:ind3}
\begin{cases}
(\eta_{1n}(\eta_{2n}-p_2)b_2-\eta_{2n}(\eta_{1n}-p_2))x_n+(\eta_{1n}(\eta_{2n}-p_1)a_2-\eta_{2n}(\eta_{1n}-p_1))y_n=c_2, \\
(\eta_{1n}(\eta_{3n}-p_2)b_3-\eta_{3n}(\eta_{1n}-p_2))x_n+(\eta_{1n}(\eta_{3n}-p_1)a_3-\eta_{3n}(\eta_{1n}-p_1))y_n=c_3,
\end{cases}
\end{equation}
where $c_i:=\dfrac{\eta_{in}}{\eta_{1n}}(\eta_{1n}-1)-(\eta_{in}-1)x_iy_i$. We denote the coefficient matrix of the system~\eqref{equ:ind3} by
\[
A:=\begin{bmatrix}
\eta_{1n}(\eta_{2n}-p_2)b_2-\eta_{2n}(\eta_{1n}-p_2) & \eta_{1n}(\eta_{2n}-p_1)a_2-\eta_{2n}(\eta_{1n}-p_1) \\
\eta_{1n}(\eta_{3n}-p_2)b_3-\eta_{3n}(\eta_{1n}-p_2) & \eta_{1n}(\eta_{3n}-p_1)a_3-\eta_{3n}(\eta_{1n}-p_1)
\end{bmatrix},
\]
and define $
W_{n}:=\{(\bm{a}_{2:n},\bm{b}_{2:n}) \in \mathbb{C}^{2n-2}: \det(A) = 0\}$ and $V_n:=(V_{n-1} \times \mathbb{C}^2) \cup W_{n} \subseteq \mathbb{C}^{2n-2}$. Note that $V_n$ is also a $\lambda_{2n-2}$-measure zero subset. We finish the proof by claiming that the system~\eqref{equ:ind} has a unique solution in $\mathbb{C}$ (counted with multiplicity) for all $(\bm{a}_{2:n},\bm{b}_{2:n}) \in Q_{BTL}^{2n-2}$ but the set $V_n$. Indeed, for any $(\bm{a}_{2:n},\bm{b}_{2:n}) \in Q_{BTL}^{2n-2}-V_n$, we have the following conclusions:
\begin{enumerate}
\item 
Since $(\bm{a}_{2:n-1},\bm{b}_{2:n-1}) \notin V_{n-1}$, the system~\eqref{equ:ind1} has a unique solution in $\mathbb{C}$ (counted with multiplicity), given by $(\bm{x}_{1:n-1},\bm{y}_{1:n-1})=(\bm{a}_{1:n-1},\bm{b}_{1:n-1})$.
\item
Since $(\bm{a}_{2:n},\bm{b}_{2:n}) \notin W_{n}$, the system~\eqref{equ:ind3} has a unique solution in $\mathbb{C}$ (counted with multiplicity), given by $(x_n,y_n)=(a_n,b_n)$.
\end{enumerate}
This shows that the system~\eqref{equ:ind} has a unique solution in $\mathbb{C}$ (counted with multiplicity), since $(\bm{x},\bm{y})=(\bm{a},\bm{b})$ is always a solution. Tracing back to our notation in Section~\ref{sec:btl_model}, this means $\lambda_{2n-2}(N_{BTL}^{2n-2}) \leq  \lambda_{2n-2}(V_n) = 0$.
\end{proof}

\subsection{Mixtures of MNL models with 3-slate}

Here, we present the results for mixtures of MNL models with 3-slate.

\subsubsection{Unknown mixing probabilities}

First suppose $p_1$ is unknown. We study the equation system in variables $(\bm{x},\bm{y},p)$ and formally write the conclusion that the equation system achieves generic identifiability up to reordering:

\begin{theorem}
\label{prop-thm:ex-mnl3_p}
If $n \geq 4$, the system~\eqref{ex:BTL_2mixtures_triplet_p} has exactly two solutions in $\mathbb{C}$ (counted with multiplicity) for all $(\bm{a}_{2:n},\bm{b}_{2:n},p_1) \in Q^{2n-1}_{MNL,p}$ but a set of $\lambda_{2n-1}$-measure zero, given by $(\bm{x},\bm{y},p)=(\bm{a},\bm{b},p_1)$ and $(\bm{x},\bm{y},p)=(\bm{b},\bm{a},1-p_1)$.
\end{theorem}

\begin{proof}[Proof of Theorem~\ref{prop-thm:ex-mnl3_p}]

When we consider $p$ as an variable in the polynomial functions, our domain $Q_{MNL,p}^{2n-1}$ becomes $Q_{MNL}^{2n-2} \times (0,1) \subseteq \real^{2n-1}$. We solve the equation system
\begin{align*}
\begin{cases}
x_1 = y_1 = 1, \\
p\dfrac{x_i}{x_i + x_j + x_k} +  (1-p)\dfrac{y_i}{y_i + y_j + y_k} = \eta_{i,j,k}, & \forall i < j < k \in [n].
\end{cases}
\end{align*}
Accordingly, we define the set of bad parameters $N^{2n-1}_{MNL,p}$ in the same way as we did for $N^{2n-1}_{BTL,p}$.

\paragraph{Case $n=4$.} In this case, we can expand the system~\eqref{ex:BTL_2mixtures_triplet_p} such that its coefficients are given by polynomials in $(\bm{a},\bm{b})$:
\begin{equation}\label{ex:BTL_2mixtures_triplet_p_poly_}
\begin{cases}
x_1=y_1=1, \\
(px_i(y_i+y_j+y_k)+(1-p)y_i(x_i+x_j+x_k))(b_i+b_j+b_k)(a_i+a_j+a_k)- \\
\left(p_1a_i(b_i+b_j+b_k)+(1-p_1)b_i(a_i+a_j+a_k)\right)(x_i+x_j+x_k)(y_i+y_j+y_k) = 0, & \forall i < j < k \in [4], \\
t_{i,j,k}(x_i+x_j+x_k) = 1, & \forall i < j < k \in [4], \\
h_{i,j,k}(y_i+y_j+y_k) = 1. & \forall i < j < k \in [4]. \\
\end{cases}
\end{equation}
Note this this is a faithful transformation of the system~\eqref{ex:BTL_2mixtures_triplet_p}. We consider the following subset of equations in the system~\eqref{ex:BTL_2mixtures_triplet_p_poly_}:
\begin{equation}\label{ex:BTL_2mixtures_triplet_p_poly_subset}
\begin{cases}
x_1=y_1=1, \\
(px_i(y_i+y_j+y_k)+(1-p)y_i(x_i+x_j+x_k))(b_i+b_j+b_k)(a_i+a_j+a_k)- \\
\left(p_1a_i(b_i+b_j+b_k)+(1-p_1)b_i(a_i+a_j+a_k)\right)(x_i+x_j+x_k)(y_i+y_j+y_k) = 0, & \forall i < j < k \in [4], \\
t_{123}(x_1+x_2+x_3) = 1, \\
t_{124}(x_1+x_2+x_4) = 1, \\
h_{123}(y_1+y_2+y_3) = 1, \\
h_{124}(y_1+y_2+y_4) = 1.
\end{cases}
\end{equation}
To proceed, we need to determine $Z(\bm{a}_{2:4},\bm{b}_{2:4},p_1)$ defined by equation~\eqref{defn:Z(t)}, which can be done using Magma.
\begin{lstlisting}[language=Magma,caption=Gr\"{o}bner basis of MNL models with unknown mixing probabilities,label={lst:3mix_p_un},frame=single,basicstyle=\scriptsize]
P<t123,t124,h123,h124,p,x2,x3,x4,y2,y3,y4,
  a2,a3,a4,b2,b3,b4,p1>:=FreeAlgebra(Rationals(),18,"lex");

I:=ideal<P|
(p*(1+y2+y3)+(1-p)*(1+x2+x3))*(1+b2+b3)*(1+a2+a3)-
(p1*(1+b2+b3)+(1-p1)*(1+a2+a3))*(1+x2+x3)*(1+y2+y3),
(p*x2*(1+y2+y3)+(1-p)*y2*(1+x2+x3))*(1+b2+b3)*(1+a2+a3)-
(p1*a2*(1+b2+b3)+(1-p1)*b2*(1+a2+a3))*(1+x2+x3)*(1+y2+y3),
(p*(1+y2+y4)+(1-p)*(1+x2+x4))*(1+b2+b4)*(1+a2+a4)-
(p1*(1+b2+b4)+(1-p1)*(1+a2+a4))*(1+x2+x4)*(1+y2+y4),
(p*x2*(1+y2+y4)+(1-p)*y2*(1+x2+x4))*(1+b2+b4)*(1+a2+a4)-
(p1*a2*(1+b2+b4)+(1-p1)*b2*(1+a2+a4))*(1+x2+x4)*(1+y2+y4),
(p*(1+y3+y4)+(1-p)*(1+x3+x4))*(1+b3+b4)*(1+a3+a4)-
(p1*(1+b3+b4)+(1-p1)*(1+a3+a4))*(1+x3+x4)*(1+y3+y4),
(p*x3*(1+y3+y4)+(1-p)*y3*(1+x3+x4))*(1+b3+b4)*(1+a3+a4)-
(p1*a3*(1+b3+b4)+(1-p1)*b3*(1+a3+a4))*(1+x3+x4)*(1+y3+y4),
(p*x2*(y2+y3+y4)+(1-p)*y2*(x2+x3+x4))*(b2+b3+b4)*(a2+a3+a4)-
(p1*a2*(b2+b3+b4)+(1-p1)*b2*(a2+a3+a4))*(x2+x3+x4)*(y2+y3+y4),
(p*x3*(y2+y3+y4)+(1-p)*y3*(x2+x3+x4))*(b2+b3+b4)*(a2+a3+a4)-
(p1*a3*(b2+b3+b4)+(1-p1)*b3*(a2+a3+a4))*(x2+x3+x4)*(y2+y3+y4),
t123*(1+x2+x3)-1,
t124*(1+x2+x4)-1,
h123*(1+y2+y3)-1,
h124*(1+y2+y4)-1>;

-> GroebnerBasis(I);
\end{lstlisting}
From the output, we obtain
\begin{equation}\label{MNLp_bad-1}
\begin{split}
&\mathrm{Bad}(\bm{a}_{2:4},\bm{b}_{2:4},p_1)=\\
&\begin{cases}
-1 - 2 a_4 (1 - p_1) - 2 b_4 p_1 \\
(1 + a_2 + a_3) (1 + b_2 + b_3) \\
(1 + a_2 + a_4) (1 + b_2 + b_4) \\
(1 + a_3 + a_4) (1 + b_3 + b_4) \\
(a_2 + a_3 + a_4) (b_2 + b_3 + b_4) \\
a_2 - a_4 (1 - p_1) - a_2 p_1 + (b_2 - b_4) p_1 \\
a_3 - a_4 (1 - p_1) - a_3 p_1 + (b_3 - b_4) p_1 \\
-1 - b_2 + a_3 b_2 - a_2 (1 - b_3) - a_4 (1 - b_3) + a_3 b_3 - b_4 + a_3 b_4 \\
a_2 (1 + b_2 + b_3) + (1 + a_3) (b_2 - b_4) - a_4 (1 + b_3 + b_4) \\
a_3 (1 + b_2 + b_3) + (1 + a_2) (b_3 - b_4) - a_4 (1 + b_2 + b_4) \\
(a_3 + a_4) b_2 (1 - p_1) + a_2 (b_2 + (b_3 + b_4) p_1)\\
(a_3 + a_4) (1 - p_1) + (1 + (b_3 + b_4) p_1) \\
(a_2 + a_3) (1 - p_1) + (1 + (b_2 + b_3) p_1) \\
(a_2 + a_4) (1 - p_1) + (1 + (b_2 + b_4) p_1) \\
(a_3 + a_4) b_2 (1 - p_1) + a_2 (b_2 + (b_3 + b_4) p_1) \\
(a_4 + a_2) b_3 (1 - p_1) + a_3 (b_3 + b_2 p_1 + b_3 p_1 + b_4 p_1)\\
a_4 - b_2 - a_3 b_2 - a_4 p_1 + a_3 b_2 p_1 + b_4 p_1 - a_2 (1 + b_2 + b_3 p_1)\\
(a_2 b_3 - b_4) p_1 - b_2 (1 - a_3 + a_3 p_1) - (1 + a_2 + a_4 - a_4 p_1)\\
1 + b_2 + a_4 (1 + b_2) (1 - p_1) + b_4 p_1 + a_2 (1 + b_2 + b_4 p_1)\\
a_2 (b_3 + b_4) (1 - p_1) + (a_3 + a_4) (b_3 + b_4 + b_2 p_1)\\
a_4 (b_2 + b_4) + a_2 (b_2 + b_4 + b_3 p_1) + a_3 (b_2 + b_4 - b_2 p_1 - b_3 p_1 - b_4 p_1)\\
a_2 (b_3 + b_4) (1 - p_1) + (a_3 + a_4) (b_3 + b_4 + b_2 p_1)\\
(1 + a_2) b_4 p_1 + (1 + a_2 + a_4 - a_4 p_1) + b_2 (1 + a_2 + a_4 - a_4 p_1)\\
(1 + a_2) b_3 p_1 + (1 + a_2 + a_3 - a_3 p_1) + b_2 (1 + a_2 + a_3 - a_3 p_1)\\
a_2 (1 + b_2 + b_4) + a_3 (1 + b_3 + b_4) + (1 + a_4) (2 + b_2 + b_3 + 2 b_4)\\
1 - a_3 b_2 (1 - p_1) + a_4 (2 + b_2) (1 - p_1) - a_2 b_3 p_1 + 2 b_4 p_1 + a_2 b_4 p_1\\
-2 - 2 a_4 - a_2 (1 - p_1) - a_3 (1 - p_1) + 2 a_4 p_1 - (b_2 + b_3 + 2 b_4) p_1\\
a_4 (1 + b_2) p_1 + b_4 (1 + a_2 + a_4 - p_1 - a_2 p_1)\\
a_3 (1 + b_2) p_1 + b_3 (1 + a_2 + a_3 - p_1 - a_2 p_1)\\
a_3 (1 + b_2) (1 - p_1) - a_4 (1 + b_2) (1 - p_1) + (1 + a_2) (b_3 - b_4) p_1\\
2 + a_2 + a_3 + 2 a_4 - a_3 p_1 - 2 a_4 p_1 + (b_3 + (2 + a_2) b_4) p_1 + b_2 (1 + a_2 + a_4 - a_4 p_1)\\
-a_4 b_3 + b_4 - a_2 b_3 (1 - p_1) + a_4 p_1 - b_4 p_1 - a_3 (b_3 + b_4 + b_2 p_1)\\
1 - a_2 b_3 + b_3 p_1 + a_2 b_3 p_1 + b_4 p_1 + a_4 (1 - b_3 - p_1) - a_3 (b_3 + p_1 + b_2 p_1 + b_3 p_1 + b_4 p_1-1)\\
(1 + a_2) (b_3 - b_4) (1 - p_1) + a_3 (b_3 + p_1 + b_2 p_1) - a_4 (b_4 + p_1 + b_2 p_1)\\
1 + a_2 (b_2 + b_4) + a_4 (1 + b_2 + b_4 - p_1) + b_3 p_1 (1 + a_2 ) + b_4 p_1 + a_3 (1 + b_2 (1 - p_1) + b_4 (1 - p_1) - p_1 - b_3 p_1)
\end{cases}
\end{split}
\end{equation}
Based on Theorem~\ref{thm:general-case}, we see that by checking Assumptions~\ref{assum:key-assumption1} and~\ref{assum:key-assumption2}, the system~\eqref{ex:BTL_2mixtures_triplet_p_poly_subset} has exactly two solutions in $\mathbb{C}$ (counted with multiplicity) for all $(\bm{a}_{2:4},\bm{b}_{2:4},p_1) \in Q_{MNL,p}^{7}$ but a set of $\lambda_{7}$-measure zero.

Assumption~\ref{assum:key-assumption1}: This is clear, since
\begin{equation}\label{2-solutions}
\begin{aligned}
(\bm{x}_{1:4},\bm{y}_{1:4},\bm{t},\bm{h},p)=&\left(\bm{a}_{1:4},\bm{b}_{1:4},\left(\dfrac{1}{a_i+a_j+a_k}\right)_{i<j<k},\left(\dfrac{1}{b_i+b_j+b_k}\right)_{i<j<k},p_1\right) \mbox{ or }\\
&\left(\bm{b}_{1:4},\bm{a}_{1:4},\left(\dfrac{1}{b_i+b_j+b_k}\right)_{i<j<k},\left(\dfrac{1}{a_i+a_j+a_k}\right)_{i<j<k},1-p_1\right)
\end{aligned}
\end{equation}
are two (distinct) solutions of the system~\eqref{ex:BTL_2mixtures_triplet_p_poly_subset} for all $(\bm{a}_{2:4},\bm{b}_{2:4},p_1) \in \mathbb{C}^{7} \setminus E$, where $E \subseteq \mathbb{A}_{\mathbb{C}}^7$ is the Zariski closed proper subset defined by
\[
E:=\{(\bm{a}_{2:4},\bm{b}_{2:4},p_1) \in \mathbb{A}_{\mathbb{C}}^6 \times \mathbb{A}_{\mathbb{C}}^1: p_1-0.5=0 \text{ or } a_i+a_j+a_k=0 \text{ or } b_i+b_j+b_k=0\} \subseteq \mathbb{A}_{\mathbb{C}}^7.
\]

Assumption~\ref{assum:key-assumption2}: Choose $(\bm{a'}_{1:4},\bm{b'}_{1:4},p'_1)=(1,2,3,4;1,5,4,2;0.7)$. It is routine to check that $(\bm{a'}_{2:4},\bm{b'}_{2:4},p'_1) \in \mathbb{C}^{7} \setminus Z(\bm{a}_{2:4},\bm{b}_{2:4},p_1)$, using equation~\eqref{MNLp_bad-1}. Since the associated equation system~\eqref{ex:BTL_2mixtures_triplet_p_poly_subset} has $\mathbb{Q}$-coefficients, we can use Magma to check whether it has exactly two solutions in $\mathbb{C}$ (counted with multiplicity) for this $(\bm{a'}_{2:4},\bm{b'}_{2:4},p'_1)$.
\begin{lstlisting}[language=Magma,caption=Dimension and degree computations of MNL models with unknown mixing probabilities, label=mnl3_model_abp,frame=single,basicstyle=\scriptsize]
a:=[1,2,3,4];
b:=[1,5,4,2];
p1:=7/10;

k:=Rationals();
A<x1,x2,x3,x4,y1,y2,y3,y4,p,t123,t124,h123,h124>:=AffineSpace(k,13);
P:=Scheme(A,[x1-1,y1-1,
(p*x1*(y1+y2+y3)+(1-p)*y1*(x1+x2+x3))*(b[1]+b[2]+b[3])*(a[1]+a[2]+a[3])-
(p1*a[1]*(b[1]+b[2]+b[3])+(1-p1)*b[1]*(a[1]+a[2]+a[3]))*(x1+x2+x3)*(y1+y2+y3),
(p*x2*(y1+y2+y3)+(1-p)*y2*(x1+x2+x3))*(b[1]+b[2]+b[3])*(a[1]+a[2]+a[3])-
(p1*a[2]*(b[1]+b[2]+b[3])+(1-p1)*b[2]*(a[1]+a[2]+a[3]))*(x1+x2+x3)*(y1+y2+y3),
(p*x1*(y1+y2+y4)+(1-p)*y1*(x1+x2+x4))*(b[1]+b[2]+b[4])*(a[1]+a[2]+a[4])-
(p1*a[1]*(b[1]+b[2]+b[4])+(1-p1)*b[1]*(a[1]+a[2]+a[4]))*(x1+x2+x4)*(y1+y2+y4),
(p*x2*(y1+y2+y4)+(1-p)*y2*(x1+x2+x4))*(b[1]+b[2]+b[4])*(a[1]+a[2]+a[4])-
(p1*a[2]*(b[1]+b[2]+b[4])+(1-p1)*b[2]*(a[1]+a[2]+a[4]))*(x1+x2+x4)*(y1+y2+y4),
(p*x1*(y1+y3+y4)+(1-p)*y1*(x1+x3+x4))*(b[1]+b[3]+b[4])*(a[1]+a[3]+a[4])-
(p1*a[1]*(b[1]+b[3]+b[4])+(1-p1)*b[1]*(a[1]+a[3]+a[4]))*(x1+x3+x4)*(y1+y3+y4),
(p*x3*(y1+y3+y4)+(1-p)*y3*(x1+x3+x4))*(b[1]+b[3]+b[4])*(a[1]+a[3]+a[4])-
(p1*a[3]*(b[1]+b[3]+b[4])+(1-p1)*b[3]*(a[1]+a[3]+a[4]))*(x1+x3+x4)*(y1+y3+y4),
(p*x2*(y2+y3+y4)+(1-p)*y2*(x2+x3+x4))*(b[2]+b[3]+b[4])*(a[2]+a[3]+a[4])-
(p1*a[2]*(b[2]+b[3]+b[4])+(1-p1)*b[2]*(a[2]+a[3]+a[4]))*(x2+x3+x4)*(y2+y3+y4),
(p*x3*(y2+y3+y4)+(1-p)*y3*(x2+x3+x4))*(b[2]+b[3]+b[4])*(a[2]+a[3]+a[4])-
(p1*a[3]*(b[2]+b[3]+b[4])+(1-p1)*b[3]*(a[2]+a[3]+a[4]))*(x2+x3+x4)*(y2+y3+y4),
(x1+x2+x3)*t123-1,
(x1+x2+x4)*t124-1,
(y1+y2+y3)*h123-1,
(y1+y2+y4)*h124-1
]);

-> Dimension(P);
-> @0@

-> Degree(P);
-> @2@
\end{lstlisting}
From Listing~\ref{mnl3_model_abp}, $\texttt{Dimension(P)=0}$ and $\texttt{Degree(P)=2}$ means the system~\eqref{ex:BTL_2mixtures_triplet_p_poly_subset} has exactly two solutions in $\mathbb{C}$ (counted with multiplicity) for this choice of $(\bm{a'}_{2:4},\bm{b'}_{2:4},p'_1)$. 

Thus, by Theorem~\ref{thm:general-case}, we have proved that the system~\eqref{ex:BTL_2mixtures_triplet_p_poly_subset} (and hence the system~\eqref{ex:BTL_2mixtures_triplet_p}) has exactly two solutions in $\mathbb{C}$ (counted with multiplicity) for all $(\bm{a}_{2:4},\bm{b}_{2:4},p_1) \in Q_{MNL,p}^{7}$ but a set $V_4$ of $\lambda_{7}$-measure zero.
\paragraph{Case $n \geq 4$.} This can be treated by induction in exactly the same way as in Proposition~\ref{prop:ex-mnl3}.
\end{proof}

\subsubsection{Known mixing probabilities}

In the case of mixtures of MNL models with 3-slate with known $(p_1,p_2)$, we study the equation system~\eqref{ex:BTL_2mixtures_triplet} in variables $(\bm{x},\bm{y})$ for both $p_1 \neq 0.5$ and $p_1 = 0.5$. For $p_1 \neq 0.5$, we show that the equation system achieves generic identifiability. For $p_1 = 0.5$, we show that the equation system achieves generic identifiability up to reordering. We have the proposition below:

\begin{proposition}
\label{prop:ex-mnl3}
Suppose $n \geq 4$.
\begin{enumerate}
\item 
If $p_1=0.3$, the system~\eqref{ex:BTL_2mixtures_triplet} 
has a unique solution in $\mathbb{C}$ (counted with multiplicity) for all $(\bm{a}_{2:n},\bm{b}_{2:n}) \in Q^{2n-2}_{MNL}$ but a set of $\lambda_{2n-2}$-measure zero, given by $(\bm{x},\bm{y})=(\bm{a},\bm{b})$. Thus, generic identifiability holds.
\item
If $p_1 = 0.5$, the system~\eqref{ex:BTL_2mixtures_triplet} has a unique solution (up to reordering) in $\mathbb{C}$ (counted with multiplicity) for all $(\bm{a}_{2:n},\bm{b}_{2:n}) \in Q^{2n-2}_{MNL}$ but a set of $\lambda_{2n-2}$-measure zero, given by $(\bm{x},\bm{y})=(\bm{a},\bm{b})$ or $(\bm{x},\bm{y})=(\bm{b},\bm{a})$. Thus, generic identifiability holds.
\end{enumerate}
\end{proposition}
\begin{proof}[Proof of Proposition~\ref{prop:ex-mnl3}]
It suffices to prove the case $p_1=0.7$, since the remaining case $p_1=0.5$ can be treated in exactly the same way. As before, this proposition is proved by induction on $n$.
\paragraph{Case $n=4$.} In this case, we recast the equation system~\eqref{ex:BTL_2mixtures_triplet} in the following form, with coefficients given by polynomials in $(\bm{a},\bm{b})$:
\begin{equation}\label{eq:BTL_2mixtures_triplet_poly_}
\begin{cases}
x_1=y_1=1, \\
(p_1x_i(y_i+y_j+y_k)+p_2y_i(x_i+x_j+x_k))(b_i+b_j+b_k)(a_i+a_j+a_k)- \\
\left(p_1a_i(b_i+b_j+b_k)+p_2b_i(a_i+a_j+a_k)\right)(x_i+x_j+x_k)(y_i+y_j+y_k) = 0. \quad \forall i < j < k \in [4].
\end{cases}
\end{equation}
Note this this is \emph{not} a faithful transformation of the system~\eqref{ex:BTL_2mixtures_triplet} (but can only increase the number of solutions). To proceed, we need to determine $Z(\bm{a}_{2:5},\bm{b}_{2:5})$ introduced in \eqref{defn:Z(t)}. This can be done by Magma.
\begin{lstlisting}[language=Magma,caption=Gr\"{o}bner basis of MNL models with unknown mixing probabilities,label={lst:3mix_p07},frame=single,basicstyle=\scriptsize]
P<x1,x2,x3,x4,y1,y2,y3,y4,
  a1,a2,a3,a4,b1,b2,b3,b4>:=FreeAlgebra(Rationals(),16,"lex");

I:=ideal<P|x1-1,y1-1,a1-1,b1-1,
(3/10*x1*(y1+y2+y3)+(1-3/10)*y1*(x1+x2+x3))*(b1+b2+b3)*(a1+a2+a3)-
(3/10*a1*(b1+b2+b3)+(1-3/10)*b1*(a1+a2+a3))*(x1+x2+x3)*(y1+y2+y3),
(3/10*x2*(y1+y2+y3)+(1-3/10)*y2*(x1+x2+x3))*(b1+b2+b3)*(a1+a2+a3)-
(3/10*a2*(b1+b2+b3)+(1-3/10)*b2*(a1+a2+a3))*(x1+x2+x3)*(y1+y2+y3),
(3/10*x1*(y1+y2+y4)+(1-3/10)*y1*(x1+x2+x4))*(b1+b2+b4)*(a1+a2+a4)-
(3/10*a1*(b1+b2+b4)+(1-3/10)*b1*(a1+a2+a4))*(x1+x2+x4)*(y1+y2+y4),
(3/10*x2*(y1+y2+y4)+(1-3/10)*y2*(x1+x2+x4))*(b1+b2+b4)*(a1+a2+a4)-
(3/10*a2*(b1+b2+b4)+(1-3/10)*b2*(a1+a2+a4))*(x1+x2+x4)*(y1+y2+y4),
(3/10*x1*(y1+y3+y4)+(1-3/10)*y1*(x1+x3+x4))*(b1+b3+b4)*(a1+a3+a4)-
(3/10*a1*(b1+b3+b4)+(1-3/10)*b1*(a1+a3+a4))*(x1+x3+x4)*(y1+y3+y4),
(3/10*x3*(y1+y3+y4)+(1-3/10)*y3*(x1+x3+x4))*(b1+b3+b4)*(a1+a3+a4)-
(3/10*a3*(b1+b3+b4)+(1-3/10)*b3*(a1+a3+a4))*(x1+x3+x4)*(y1+y3+y4),
(3/10*x2*(y2+y3+y4)+(1-3/10)*y2*(x2+x3+x4))*(b2+b3+b4)*(a2+a3+a4)-
(3/10*a2*(b2+b3+b4)+(1-3/10)*b2*(a2+a3+a4))*(x2+x3+x4)*(y2+y3+y4),
(3/10*x3*(y2+y3+y4)+(1-3/10)*y3*(x2+x3+x4))*(b2+b3+b4)*(a2+a3+a4)-
(3/10*a3*(b2+b3+b4)+(1-3/10)*b3*(a2+a3+a4))*(x2+x3+x4)*(y2+y3+y4)>;

-> GroebnerBasis(I);
\end{lstlisting}
From the output, we obtain
\begin{equation}\label{MNL_bad}
\mathrm{Bad}(\bm{a}_{2:4},\bm{b}_{2:4})=
\begin{cases}
(1+a_2+a_4)(1+b_2+b_4) \\
(1+a_3 + a_4) (1 + b_3 + b_4) \\
10 a_2 b_2 + 7 (a_3 + a_4) b_2 + 3 a_2 (b_3 + b_4)\\
79 + 49 a_2 + 70 a_4 + 9 b_2 - 28 a_2 b_3 - 49 a_3 b_3 - 
 49 a_4 b_3 + 58 b_4 - 49 a_3 b_4 \\
3 a_2 b_2 - 4 a_2 (b_3 + b_4) - 7 (a_3 + a_4) (b_3 + b_4) \\
51 + 21 a_2 + 42 a_4 - 19 b_2 + 28 a_3 b_2 - 21 a_3 b_3 - 
 21 a_4 b_3 + 30 b_4 - 21 a_3 b_4 \\
7 a_2 b_2 + (a_3 + a_4) (4 b_2 - 3 (b_3 + b_4)) \\
1 - a_2 b_3 + 2 b_4 + a_2 b_4 + a_3 (b_4 - b_2) + a_4 (2 + b_2 + b_3 + 2 b_4) \\
-60 - 21 a_3 - 42 a_4  + 70 a_2 b_2 + 49 a_4 b_2 + 40 b_2 - 9 b_3 - 18 b_4 + 21 a_2 b_4 \\
a_2 (1 + b_2 + b_3) + (1 + a_3) (b_2 - b_4) - a_4 (1 + b_3 + b_4) \\
a_3 (1 + b_2 + b_3) + (1 + a_2) (b_3 - b_4) - a_4 (1 + b_2 + b_4) \\
a_2 (1 + b_2 + b_4) + a_3 (1 + b_3 + b_4) + (1 + a_4) (2 + b_2 + b_3 + 2 b_4) \\
-21 a_4 - (40 + 70 a_2 + 49 a_3) b_2 - 3 (7 a_2 b_3 + 3 b_4)\\
79 + 70 a_4 + 9 b_2 - 7 a_2 (-7 + 3 b_2 - 4 b_4) + 58 b_4 + 49 a_4 b_4 \\
79 + 70 a_3 + 9 b_2 - 7 a_2 (-7 + 3 b_2 - 4 b_3) + 58 b_3 + 49 a_3 b_3 \\
79 + 70 a_4 + 9 b_3 - 7 a_3 (-7 + 3 b_3 - 4 b_4) + 58 b_4 + 49 a_4 b_4 \\
51 + 42 a_4 - 19 b_3 - 28 a_4 b_3 - 7 a_3 (-3 + 7 b_3) + 30 b_4 + 21 a_4 b_4 \\
51 + 42 a_4 - 19 b_2 - 28 a_4 b_2 - 7 a_2 (-3 + 7 b_2) + 30 b_4 + 21 a_4 b_4 \\
51 + 42 a_3 - 19 b_2 - 28 a_3 b_2 - 7 a_2 (-3 + 7 b_2) + 30 b_3 + 21 a_3 b_3 \\
7 a_3 (13 + 3 b_2 + 10 b_3) + (79 + 49 a_2) (b_3 - b_4) - 7 a_4 (13 + 3 b_2 + 10 b_4) \\
- 7 a_3 (-3 + 7 b_2) + 7 a_4 (-3 + 7 b_2) - 3 (-3 + 7 a_2) (b_3 - b_4) \\
7 a_3 (10 + 7 b_3) + 2 (29 + 14 a_2) (b_3 - b_4) - 7 a_4 (10 + 7 b_4) \\
20 + 7 a_2 + 7 a_3 + 14 a_4 + 3 b_2 + 3 b_3 + 6 b_4 \\
a_2 (1 + b_2 + b_3) + (1 + a_3) (b_2 - b_4) - a_4 (1 + b_3 + b_4) \\
a_3 (1 + b_2 + b_3) + (1 + a_2) (b_3 - b_4) - a_4 (1 + b_2 + b_4) \\
3 + 7 a_1 + 7 a_2 + 7 a_4 + 3 b_2 + 3 b_4 \\
3 + 7 a_1 + 7 a_3 + 7 a_4 + 3 b_3 + 3 b_4 \\
7 a_3 - 7 a_4 + 3 b_3 - 3 b_4 \\
7 a_2 - 7 a_4 + 3 b_2 - 3 b_4 \\
7 (a_2 + a_4) b_3 + a_3 (3 b_2 + 10 b_3 + 3 b_4) \\
7 a_2 (b_2 + b_4) + 7 a_4 (b_2 + b_4) + a_3 (4 b_2 - 3 b_3 + 4 b_4)\\
7 a_3 b_3 - a_2 (3 b_2 - 4 b_3 + 3 b_4) - a_4 (3 b_2 - 4 b_3 + 3 b_4) \\
7 a_3 (b_2 + b_4) + a_2 (10 b_2 + 3 b_3 + 10 b_4) + a_4 (10 b_2 + 3 b_3 + 10 b_4) \\
30 - 40 b_3 - 70 a_3 b_3 + 21a_4 - 49 a_4b_3 + 9 b_4 - 21 a_3 b_4 \\
30 - 40 b_2 - 70 a_2 b_2 + 21a_4 - 49 a_4 b_2 + 9 b_4 - 21 a_2 b_4 \\
30 + 49 a_3 b_2 - 7 a_4 (-6 + 7 b_2) + 21 a_2 b_3 + 18 b_4 - 21 a_2 b_4 \\
30 + 49 a_3 b_2 + 21 a_2 b_3 - 40 b_2 + 21 a_4 + 9 b_4\\
100 + 91 a_4 + 30 b_3 + 21 a_4 b_3 + 79 b_4 + 70 a_4 b_4 + 7 a_3 (10 + 7 b_4) \\
100 + 91 a_4 + 79 b_2 - 21 a_3 b_2 - 49b_2 + 70 a_2 - 49 a_2 b_3 - 70 a_3 b_3 - 70 a_4 b_3 + 79 b_4 - 70 a_3 b_4 \\
10 + 7 a_2 + 7 a_4 + 3 b_2 + 3 b_4 \\
10 + 7 a_2 + 7 a_3 + 3 b_2 + 3 b_3 \\
10 + 7 a_3 + 7 a_4 + 3 b_3 + 3 b_4 \\
10 + 14 a_4 + 6 b_4 
\end{cases}
\end{equation}
Based on Theorem~\ref{thm:general-case}, it follows by checking Assumptions~\ref{assum:key-assumption1} and~\ref{assum:key-assumption2} that the system~\eqref{eq:BTL_2mixtures_triplet_poly_} has a unique solution in $\mathbb{C}$ (counted with multiplicity) for all $(\bm{a}_{2:4},\bm{b}_{2:4}) \in Q_{MNL}^{6}$ but a set of $\lambda_{6}$-measure zero.

Assumption~\ref{assum:key-assumption1}: This is clear, since $(\bm{x}_{1:4},\bm{y}_{1:4})=(\bm{a}_{1:4},\bm{b}_{1:4})$ is a solution of the system~\eqref{eq:BTL_2mixtures_triplet_poly_} for all $(\bm{a}_{2:4},\bm{b}_{2:4}) \in \mathbb{C}^6$.

Assumption~\ref{assum:key-assumption2}: Choose $(\bm{a'}_{1:4},\bm{b'}_{1:4})=(1,2,3,4;1,5,4,2)$. It is routine to check that $(\bm{a'}_{2:4},\bm{b'}_{2:4}) \in \mathbb{C}^{6} \setminus Z(\bm{a}_{2:4},\bm{b}_{2:4})$ using equation~\eqref{MNL_bad}. Since the associated equation system~\eqref{eq:BTL_2mixtures_triplet_poly_} has $\mathbb{Q}$-coefficients, we can use Magma to check whether it has a unique solution in $\mathbb{C}$ (counted with multiplicity) for this $(\bm{a}'_{2:4},\bm{b}'_{2:4})$. 
\begin{lstlisting}[language=Magma,caption=Dimension and degree computations of MNL model with known mixing probabilities,label={lst:3mix_p07'},frame=single,basicstyle=\scriptsize]
p1:=3/10;
a:=[1,2,3,4];
b:=[1,5,4,2];

k:=Rationals();
A<x1,x2,x3,x4,y1,y2,y3,y4>:=AffineSpace(k,8);
P:=Scheme(A,
[
x1-1,
y1-1,
(p1*x1*(y1+y2+y3)+(1-p1)*y1*(x1+x2+x3))*(b[1]+b[2]+b[3])*(a[1]+a[2]+a[3])-
(p1*a[1]*(b[1]+b[2]+b[3])+(1-p1)*b[1]*(a[1]+a[2]+a[3]))*(x1+x2+x3)*(y1+y2+y3),
(p1*x2*(y1+y2+y3)+(1-p1)*y2*(x1+x2+x3))*(b[1]+b[2]+b[3])*(a[1]+a[2]+a[3])-
(p1*a[2]*(b[1]+b[2]+b[3])+(1-p1)*b[2]*(a[1]+a[2]+a[3]))*(x1+x2+x3)*(y1+y2+y3),
(p1*x1*(y1+y2+y4)+(1-p1)*y1*(x1+x2+x4))*(b[1]+b[2]+b[4])*(a[1]+a[2]+a[4])-
(p1*a[1]*(b[1]+b[2]+b[4])+(1-p1)*b[1]*(a[1]+a[2]+a[4]))*(x1+x2+x4)*(y1+y2+y4),
(p1*x2*(y1+y2+y4)+(1-p1)*y2*(x1+x2+x4))*(b[1]+b[2]+b[4])*(a[1]+a[2]+a[4])-
(p1*a[2]*(b[1]+b[2]+b[4])+(1-p1)*b[2]*(a[1]+a[2]+a[4]))*(x1+x2+x4)*(y1+y2+y4),
(p1*x1*(y1+y3+y4)+(1-p1)*y1*(x1+x3+x4))*(b[1]+b[3]+b[4])*(a[1]+a[3]+a[4])-
(p1*a[1]*(b[1]+b[3]+b[4])+(1-p1)*b[1]*(a[1]+a[3]+a[4]))*(x1+x3+x4)*(y1+y3+y4),
(p1*x3*(y1+y3+y4)+(1-p1)*y3*(x1+x3+x4))*(b[1]+b[3]+b[4])*(a[1]+a[3]+a[4])-
(p1*a[3]*(b[1]+b[3]+b[4])+(1-p1)*b[3]*(a[1]+a[3]+a[4]))*(x1+x3+x4)*(y1+y3+y4),
(p1*x2*(y2+y3+y4)+(1-p1)*y2*(x2+x3+x4))*(b[2]+b[3]+b[4])*(a[2]+a[3]+a[4])-
(p1*a[2]*(b[2]+b[3]+b[4])+(1-p1)*b[2]*(a[2]+a[3]+a[4]))*(x2+x3+x4)*(y2+y3+y4),
(p1*x3*(y2+y3+y4)+(1-p1)*y3*(x2+x3+x4))*(b[2]+b[3]+b[4])*(a[2]+a[3]+a[4])-
(p1*a[3]*(b[2]+b[3]+b[4])+(1-p1)*b[3]*(a[2]+a[3]+a[4]))*(x2+x3+x4)*(y2+y3+y4)
]);

-> Dimension(P);
-> @0@

-> Degree(P);
-> @1@
\end{lstlisting}
From Listing~\ref{lst:3mix_p07'}, $\texttt{Dimension(P)=0}$ and $\texttt{Degree(P)=1}$ means that the system~\eqref{eq:BTL_2mixtures_triplet_poly_} has a unique solution in $\mathbb{C}$ (counted with multiplicity) for this choice of $(\bm{a'}_{2:4},\bm{b'}_{2:4})$. 

Thus, by Theorem~\ref{thm:general-case}, we conclude that the system~\eqref{eq:BTL_2mixtures_triplet_poly_} (and hence also the system~\eqref{ex:BTL_2mixtures_triplet}) has a unique solution in $\mathbb{C}$ (counted with multiplicity) for all $(\bm{a}_{2:4},\bm{b}_{2:4}) \in Q_{MNL}^{6}$ but a set $V_4$ of $\lambda_{6}$-measure zero.

\paragraph{Case $n \geq 4$.} Suppose the conclusion holds for $n-1$. We need to prove that
\begin{equation}\label{eq:BTL_2mixtures_triplet_poly}
\begin{cases}
x_1=y_1=1, \\
p_1x_i(y_i+y_j+y_k)+p_2y_i(x_i+x_j+x_k)-\eta_{i,j,k}(x_i+x_j+x_k)(y_i+y_j+y_k) = 0, \quad \forall i < j < k \in [n],
\end{cases}
\end{equation}
has a unique solution in $\mathbb{C}$ (counted with multiplicity) for all $(\bm{a}_{2:n},\bm{b}_{2:n}) \in Q_{MNL}^{2n-2}$ but a set $V_n$ of $\lambda_{2n-2}$-measure zero. We begin by splitting the system~\eqref{eq:BTL_2mixtures_triplet_poly} into two parts:
\begin{gather}\label{equ:BTLind1}
\begin{cases}
x_1=y_1=1, \\
p_1x_i(y_i+y_j+y_k)+p_2y_i(x_i+x_j+x_k)-\eta_{i,j,k}(x_i+x_j+x_k)(y_i+y_j+y_k) = 0,\ \forall i < j < k \in [n-1],
\end{cases}
\end{gather}
and
\begin{equation}\label{equ:BTLind2}
\begin{aligned}
p_1x_i(y_i+y_j+y_n)+p_2y_i(x_i+x_j+x_n)-\eta_{i,j,n}(x_i+x_j+x_n)(y_i+y_j+y_n) = 0, \quad \forall i<j \in [n-1].
\end{aligned}
\end{equation}
By the induction hypothesis, there exists a $\lambda_{2n-4}$-measure zero subset $V_{n-1} \subseteq \mathbb{C}^{2n-4}$ such that the system~\eqref{equ:BTLind1} has a unique solution in $\mathbb{C}$ (counted with multiplicity) for any $(\bm{a}_{2:n-1},\bm{b}_{2:n-1}) \in Q_{MNL}^{2n-4} \setminus V_{n-1}$, given by
\[
(\bm{x}_{1:n-1},\bm{y}_{1:n-1})=(\bm{a}_{1:n-1},\bm{b}_{1:n-1}).
\]
Plugging this solution into equation~\eqref{equ:BTLind2} and simplifying, we obtain
\begin{equation}\label{equ:BTLind3}
c_{i,j,n}+(p_2b_i-\eta_{i,j,n}(b_i+b_j))x_n+(p_1a_i-\eta_{i,j,n}(a_i+a_j))y_n-\eta_{i,j,n}x_ny_n=0,
\end{equation}
where 
\[
c_{i,j,n}:=p_1a_i(b_i+b_j)+p_2b_i(a_i+a_j)-\eta_{i,j,n}(a_i+a_j)(b_i+b_j).
\]
Since $\eta_{1,2,n} \neq 0$, by the case of $(i,j)=(1,2)$ in equation~\eqref{equ:BTLind3}, we have
\[
x_ny_n=\dfrac{c_{1,2,n}+(p_2b_1-\eta_{1,2,n}(b_1+b_2))x_n+(p_1a_1-\eta_{1,2,n}(a_1+a_2))y_n}{\eta_{1,2,n}}.
\]
Plugging into the cases of $(i,j)=(1,3),(2,3)$ in equation~\eqref{equ:BTLind3}, we further the following system of linear equations in the variables $(x_n,y_n)$:
\begin{equation}\label{equ:BTLind4}
\begin{cases}
d_{13}x_n+e_{13}y_n=f_2, \\
d_{23}x_n+e_{23}y_n=f_3,
\end{cases}
\end{equation}
where 
\begin{align*}
d_{ij}=\eta_{1,2,n}(p_2b_i-\eta_{i,j,n}(b_i+b_j))&-\eta_{i,j,n}(p_2b_1-\eta_{1,2,n}(b_1+b_2)), \\ e_{ij}=\eta_{1,2,n}(p_1a_i-\eta_{i,j,n}(a_i+a_j))&-\eta_{i,j,n}(p_1a_1-\eta_{1,2,n}(a_1+a_2)),
\end{align*}
and $f_2$ and $f_3$ are some constants. Denote the coefficient matrix for the system~\eqref{equ:BTLind4} by
\[
A:=\begin{bmatrix}
d_{13} & e_{13} \\
d_{23} & e_{13}
\end{bmatrix},
\]
and define
\[
W_{n}:=\{(\bm{a}_{2:n},\bm{b}_{2:n}) \in \mathbb{C}^{2n-2}: \det(A) = 0\} \subseteq \mathbb{C}^{2n-2}, \qquad V_{n}:=(V_{n-1} \times \mathbb{C}^2) \cup W_{n} \subseteq \mathbb{C}^{2n-2}.
\]
Note that $V_n$ is a $\lambda_{2n-2}$-measure zero subset. We finish the proof by claiming that the system~\eqref{eq:BTL_2mixtures_triplet_poly} has a unique solution in $\mathbb{C}$ (counted with multiplicity) for all $(\bm{a}_{2:n},\bm{b}_{2:n}) \in Q_{MNL}^{2n-2}$ but the set $V_{n}$. Indeed, for any $(\bm{a}_{2:n},\bm{b}_{2:n}) \in Q_{MNL}^{2n-2} \setminus V_{n}$, we have the following statements:
\begin{enumerate}
\item 
Since $(\bm{a}_{2:n-1},\bm{b}_{2:n-1}) \notin V_{n-1}$, the system~\eqref{equ:BTLind1} has a unique solution in $\mathbb{C}$ (counted with multiplicity), given by $(\bm{x}_{1:n-1},\bm{y}_{1:n-1})=(\bm{a}_{1:n-1},\bm{b}_{1:n-1})$.
\item
Since $(\bm{a}_{2:n},\bm{b}_{2:n}) \notin W_{n}$, the system~\eqref{equ:BTLind4} has a unique solution in $\mathbb{C}$ (counted with multiplicity), given by $(x_n,y_n)=(a_n,b_n)$.
\end{enumerate}
This shows that the system~\eqref{eq:BTL_2mixtures_triplet_poly} has a unique solution in $\mathbb{C}$ (counted with multiplicity), since $(\bm{x},\bm{y})=(\bm{a},\bm{b})$ is always a solution.
\end{proof}

\subsection{Mixtures of Plackett-Luce models}\label{sec:ex_pl_model}

In this section, we consider mixtures of Plackett-Luce models. We first consider the case of unknown mixing probabilities and then the case of known mixing probabilities. Note that the first case has been studied by~\cite{zhao2016learning}, where the authors take a tensor-decomposition approach.

\subsubsection{Unknown mixing probabilities}

In this subsection, we consider the Plackett-Luce model with the parameter space $(\bm{a}_{2:n},\bm{b}_{2:n},p_1) \in Q_{PL,p}^{2n-1}$. We consider the equation system~\eqref{ex:PL-2mixtures_p}. Note that the system has at least two solutions in $\mathbb{C}$ coming from the initial data 
\[
(\bm{x},\bm{y},p)=(\bm{a},\bm{b},p_1) \text{ and } (\bm{x},\bm{y},p)=(\bm{b},\bm{a},1-p_1).
\]
Our goal is to show that no additional solutions exist in $\mathbb{C}$. The following result is proved:
\begin{theorem}
\label{prop-thm:ex-pl_p}
If $n \geq 4$, the system~\eqref{ex:PL-2mixtures_p} 
has exactly two solutions in $\mathbb{C}$ (counted with multiplicity) for all $(\bm{a}_{2:n},\bm{b}_{2:n},p_1) \in Q^{2n-1}_{PL,p}$ but a set of $\lambda_{2n-1}$-measure zero, given by $(\bm{x},\bm{y},p) \!=\! (\bm{a},\bm{b},p_1)$ and $(\bm{x},\bm{y},p) \!=\! (\bm{b},\bm{a},1-p_1)$.
\end{theorem}
\begin{proof}[Proof of Theorem~\ref{prop-thm:ex-pl_p}]
As before, we consider the following part of the equation system~\eqref{ex:PL-2mixtures_p}:
\begin{equation}\label{ex:PL-2mixtures_p_part}
    \begin{cases}
    \eta_{1,i,\cdot} = p \dfrac{x_1x_i}{1-x_1} + (1-p) \dfrac{y_1y_i}{1-y_1}, & \forall i \in \{2,3,4\}\\
    \eta_{2,i,\cdot} = p \dfrac{x_2x_i}{1-x_2} + (1-p) \dfrac{y_2y_i}{1-y_2}, & \forall i \in \{1,3,4\}\\
    \eta_{3,i,\cdot} = p \dfrac{x_3x_i}{1-x_3} + (1-p) \dfrac{y_3y_i}{1-y_3}, & \forall i \in \{1,2,4\}\\
    \eta_{1,i,\cdot} = p \dfrac{x_1x_i}{1-x_1} + (1-p)  \dfrac{y_1y_i}{1-y_1}, & \forall i \geq 5, \\
    \eta_{2,i,\cdot} = p \dfrac{x_2x_i}{1-x_2} + (1-p) \dfrac{y_2y_i}{1-y_2}, & \forall i \geq 5, \\
    \eta_{3,i,\cdot} = p \dfrac{x_3x_i}{1-x_3} + (1-p) \dfrac{y_3y_i}{1-y_3}, & \forall i \geq 5.
\end{cases}
\end{equation}
These equations are linear combinations of the equations in the system~\eqref{ex:PL-2mixtures_p}. Therefore, it suffices to establish the generic identifiability of the system~\eqref{ex:PL-2mixtures_p_part} (up to reordering) for $n \geq 4$; i.e., the system has exactly two solutions in $\mathbb{C}$ (counted with multiplicity), given by $(\bm{a},\bm{b},p_1)$ and $(\bm{b},\bm{a},1-p_1)$.

As before, we eliminate the denominators in the system~\eqref{ex:PL-2mixtures_p_part} by introducing new variables and equations on $(t_i, h_i)$ and multiplying by $(1-a_i)(1-b_i)$, to make the coefficients of the system~\eqref{ex:PL-2mixtures_p_part} equal to polynomials in $(\bm{a}, \bm{b},p_1)$. Finally, we obtain
\begin{equation}\label{ex:PL-2mixtures_poly_p}
\begin{cases}
& \forall i \in \{2,3,4\} \\
&\ (p_1a_1a_i(1-b_1)+p_2b_1b_i(1-a_1))(1-x_1)(1-y_1)=(1-a_1)(1-b_1)(px_1x_i(1-y_1)+(1-p)y_1y_i(1-x_1)), \\
& \forall i \in \{1,3,4\} \\
& \ (p_1a_2a_i(1-b_2)+p_2b_2b_i(1-a_2))(1-x_2)(1-y_2)=(1-a_2)(1-b_2)(px_2x_i(1-y_2)+(1-p)y_2y_i(1-x_2)), \\
& \forall i \in \{1,2,4\} \\
&\ (p_1a_3a_i(1-b_3)+p_2b_3b_i(1-a_3))(1-x_3)(1-y_3)=(1-a_3)(1-b_3)(p x_3x_i(1-y_3)+(1-p)y_3y_i(1-x_3)), \\
& \forall i \geq 5, \\
&\ (p_1a_1a_i(1-b_1)+p_2b_1b_i(1-a_1))(1-x_1)(1-y_1)=(1-a_1)(1-b_1)(px_1x_i(1-y_1)+(1-p)y_1y_i(1-x_1)), \\
& \forall i \geq 5, \\
&\ (p_1a_2a_i(1-b_2)+p_2b_2b_i(1-a_2))(1-x_2)(1-y_2) = (1-a_2)(1-b_2)(px_2x_i(1-y_2)+(1-p)y_2y_i(1-x_2)), \\
& \forall i \geq 5, \\
&\ (p_1a_3a_i(1-b_3)+p_2b_3b_i(1-a_3))(1-x_3)(1-y_3) = (1-a_3)(1-b_3)(p x_3x_i(1-y_3) + (1-p) y_3y_i(1-x_3)), \\
& \forall i \in [n], \ t_i(1-x_i) = 1, \\
& \forall i \in [n], \ h_i(1-y_i) = 1.
\end{cases}
\end{equation}

As before, we first consider the initial case $n=4$ and then deduce the general case.

\paragraph{Case $n=4$.}

In this case, the system~\eqref{ex:PL-2mixtures_poly_p} becomes
\begin{equation}\label{ex:PL-2mixtures_poly_abp_n4}
\begin{cases}
\forall i \in \{2,3,4\} \\
(p_1a_1a_i(1-b_1)+p_2b_1b_i(1-a_1))(1-x_1)(1-y_1)=(1-a_1)(1-b_1)(px_1x_i(1-y_1)+(1-p)y_1y_i(1-x_1)), \\
\forall i \in \{1,3,4\} \\
(p_1a_2a_i(1-b_2)+p_2b_2b_i(1-a_2))(1-x_2)(1-y_2)=(1-a_2)(1-b_2)(px_2x_i(1-y_2)+(1-p)y_2y_i(1-x_2)), \\
\forall i \in \{1,2,4\} \\
(p_1a_3a_i(1-b_3)+p_2b_3b_i(1-a_3))(1-x_3)(1-y_3)=(1-a_3)(1-b_3)(px_3x_i(1-y_3)+(1-p)y_3y_i(1-x_3)), \\
\forall i \in [4], \
t_i(1-x_i) = 1, \\
\forall i \in [4], \ h_i(1-y_i) = 1.
\end{cases}
\end{equation}
We will apply Theorem~\ref{thm:general-case} to this case, for which we need to check Assumptions~\ref{assum:key-assumption1} and~\ref{assum:key-assumption2}.

For Assumption~\ref{assum:key-assumption1}: We already know that (from how we define $\bm{\eta}$ and transform the equations)
\[
(\bm{x}_{1:4},\bm{y}_{1:4},p)=(\bm{a}_{1:4},\bm{b}_{1:4},p_1) \text{ or } (\bm{b}_{1:4},\bm{a}_{1:4},1-p_1)
\]
are two (distinct) solutions of the system~\eqref{ex:PL-2mixtures_poly_abp_n4} for all $(\bm{a}_{2:4},\bm{b}_{2:4},p_1) \in \mathbb{C}^{7} \setminus E$, where $E \subseteq \mathbb{A}_{\mathbb{C}}^7$ is the Zariski closed proper subset defined by
\[
E:=\left\lbrace (\bm{a}_{2:4},\bm{b}_{2:4},p_1) \in \mathbb{A}_{\mathbb{C}}^7: p_1-0.5=0 \right\rbrace.
\]
For Assumption~\ref{assum:key-assumption2}: We choose
\[
(\bm{a'}_{1:4},\bm{b'}_{1:4},p_1')=\left(\frac{1}{10},\frac{2}{10},\frac{3}{10},\frac{4}{10};\frac{1}{20},\frac{14}{20},\frac{2}{20},\frac{3}{20};\frac{7}{10}\right).
\]
To proceed, we first compute the Gr\"obner basis via Magma.
\begin{lstlisting}[language=Magma,caption=Gr\"{o}bner basis of Plackett-Luce models with unknown mixing probabilities, label=grobner_pl_model_p,frame=single,basicstyle=\scriptsize]
P<x1,x2,x3,x4,y1,y2,y3,y4,p,t1,t2,t3,t4,h1,h2,h3,h4,
  a1,a2,a3,a4,b1,b2,b3,b4,p1>:=FreeAlgebra(Rationals(),26,"lex");

I:=ideal<P|
(p1*a1*a2*(1-b1)+(1-p1)*b1*b2*(1-a1))*(1-x1)*(1-y1)-
(1-a1)*(1-b1)*(p*x1*x2*(1-y1)+(1-p)*y1*y2*(1-x1)),
(p1*a1*a3*(1-b1)+(1-p1)*b1*b3*(1-a1))*(1-x1)*(1-y1)-
(1-a1)*(1-b1)*(p*x1*x3*(1-y1)+(1-p)*y1*y3*(1-x1)),
(p1*a1*a4*(1-b1)+(1-p1)*b1*b4*(1-a1))*(1-x1)*(1-y1)-
(1-a1)*(1-b1)*(p*x1*x4*(1-y1)+(1-p)*y1*y4*(1-x1)),
(p1*a2*a1*(1-b2)+(1-p1)*b2*b1*(1-a2))*(1-x2)*(1-y2)-
(1-a2)*(1-b2)*(p*x2*x1*(1-y2)+(1-p)*y2*y1*(1-x2)),
(p1*a2*a3*(1-b2)+(1-p1)*b2*b3*(1-a2))*(1-x2)*(1-y2)-
(1-a2)*(1-b2)*(p*x2*x3*(1-y2)+(1-p)*y2*y3*(1-x2)),
(p1*a2*a4*(1-b2)+(1-p1)*b2*b4*(1-a2))*(1-x2)*(1-y2)-
(1-a2)*(1-b2)*(p*x2*x4*(1-y2)+(1-p)*y2*y4*(1-x2)),
(p1*a3*a1*(1-b3)+(1-p1)*b3*b1*(1-a3))*(1-x3)*(1-y3)-
(1-a3)*(1-b3)*(p*x3*x1*(1-y3)+(1-p)*y3*y1*(1-x3)),
(p1*a3*a2*(1-b3)+(1-p1)*b3*b2*(1-a3))*(1-x3)*(1-y3)-
(1-a3)*(1-b3)*(p*x3*x2*(1-y3)+(1-p)*y3*y2*(1-x3)),
(p1*a3*a4*(1-b3)+(1-p1)*b3*b4*(1-a3))*(1-x3)*(1-y3)-
(1-a3)*(1-b3)*(p*x3*x4*(1-y3)+(1-p)*y3*y4*(1-x3)),
(p1*a4*a1*(1-b4)+(1-p1)*b4*b1*(1-a4))*(1-x4)*(1-y4)-
(1-a4)*(1-b4)*(p*x4*x1*(1-y4)+(1-p)*y4*y1*(1-x4)),
(p1*a4*a2*(1-b4)+(1-p1)*b4*b2*(1-a4))*(1-x4)*(1-y4)-
(1-a4)*(1-b4)*(p*x4*x2*(1-y4)+(1-p)*y4*y2*(1-x4)),
t1*(1-x1)-1,
t2*(1-x2)-1,
t3*(1-x3)-1,
t4*(1-x4)-1,
h1*(1-y1)-1,
h2*(1-y2)-1,
h3*(1-y3)-1,
h4*(1-y4)-1>;

-> GroebnerBasis(I);
\end{lstlisting}
From the output, we obtain
\begin{equation}\label{pl_bad_p}
\mathrm{Bad}(\bm{a}_{2:4},\bm{b}_{2:4},p_1)=
\begin{cases}
(a_i-1 )b_ib_j(p_1-1)+a_ia_j (1-b_i)p_1, 
& \forall i < j \in [4], \\
a_i-1,\ b_i-1, & \forall i \in [3], \\
\end{cases}
\end{equation}
from which we can verify that
\[
(\bm{a'}_{2:4},\bm{b'}_{2:4},p_1') \notin Z(\bm{a}_{2:4},\bm{b}_{2:4},p_1).
\]
Next, we can check that the system~\eqref{ex:PL-2mixtures_poly_abp_n4} has exactly two solutions in $\mathbb{C}$ (counted with multiplicity) for this $(\bm{a'}_{1:4},\bm{b'}_{1:4},p_1')$ using Magma.
\begin{lstlisting}[language=Magma,caption=Dimension and degree computations of Plackett-Luce models with with unknown mixing probabilities, label=pl_model_abp,frame=single,basicstyle=\scriptsize]
a:=[1/10,2/10,3/10,4/10];
b:=[1/20,14/20,2/20,3/20];
p1:=7/10;

k:=Rationals();
A<x1,x2,x3,x4,y1,y2,y3,y4,p,t1,t2,t3,t4,h1,h2,h3,h4>:=AffineSpace(k,17);
P:=Scheme(A,
[
(p1*a[1]*a[2]*(1-b[1])+(1-p1)*b[1]*b[2]*(1-a[1]))*(1-x1)*(1-y1)-
(1-a[1])*(1-b[1])*(p*x1*x2*(1-y1)+(1-p)*y1*y2*(1-x1)),
(p1*a[1]*a[3]*(1-b[1])+(1-p1)*b[1]*b[3]*(1-a[1]))*(1-x1)*(1-y1)-
(1-a[1])*(1-b[1])*(p*x1*x3*(1-y1)+(1-p)*y1*y3*(1-x1)),
(p1*a[1]*a[4]*(1-b[1])+(1-p1)*b[1]*b[4]*(1-a[1]))*(1-x1)*(1-y1)-
(1-a[1])*(1-b[1])*(p*x1*x4*(1-y1)+(1-p)*y1*y4*(1-x1)),
(p1*a[2]*a[1]*(1-b[2])+(1-p1)*b[2]*b[1]*(1-a[2]))*(1-x2)*(1-y2)-
(1-a[2])*(1-b[2])*(p*x2*x1*(1-y2)+(1-p)*y2*y1*(1-x2)),
(p1*a[2]*a[3]*(1-b[2])+(1-p1)*b[2]*b[3]*(1-a[2]))*(1-x2)*(1-y2)-
(1-a[2])*(1-b[2])*(p*x2*x3*(1-y2)+(1-p)*y2*y3*(1-x2)),
(p1*a[2]*a[4]*(1-b[2])+(1-p1)*b[2]*b[4]*(1-a[2]))*(1-x2)*(1-y2)-
(1-a[2])*(1-b[2])*(p*x2*x4*(1-y2)+(1-p)*y2*y4*(1-x2)),
(p1*a[3]*a[1]*(1-b[3])+(1-p1)*b[3]*b[1]*(1-a[3]))*(1-x3)*(1-y3)-
(1-a[3])*(1-b[3])*(p*x3*x1*(1-y3)+(1-p)*y3*y1*(1-x3)),
(p1*a[3]*a[2]*(1-b[3])+(1-p1)*b[3]*b[2]*(1-a[3]))*(1-x3)*(1-y3)-
(1-a[3])*(1-b[3])*(p*x3*x2*(1-y3)+(1-p)*y3*y2*(1-x3)),
(p1*a[3]*a[4]*(1-b[3])+(1-p1)*b[3]*b[4]*(1-a[3]))*(1-x3)*(1-y3)-
(1-a[3])*(1-b[3])*(p*x3*x4*(1-y3)+(1-p)*y3*y4*(1-x3)),
(p1*a[4]*a[1]*(1-b[4])+(1-p1)*b[4]*b[1]*(1-a[4]))*(1-x4)*(1-y4)-
(1-a[4])*(1-b[4])*(p*x4*x1*(1-y4)+(1-p)*y4*y1*(1-x4)),
(p1*a[4]*a[2]*(1-b[4])+(1-p1)*b[4]*b[2]*(1-a[4]))*(1-x4)*(1-y4)-
(1-a[4])*(1-b[4])*(p*x4*x2*(1-y4)+(1-p)*y4*y2*(1-x4)),
t1*(1-x1)-1,
t2*(1-x2)-1,
t3*(1-x3)-1,
t4*(1-x4)-1,
h1*(1-y1)-1,
h2*(1-y2)-1,
h3*(1-y3)-1,
h4*(1-y4)-1
]);

-> Dimension(P);
-> @0@

-> Degree(P);
-> @2@
\end{lstlisting}
From Listing~\ref{pl_model_abp}, $\texttt{Dimension(P)=0}$ and $\texttt{Degree(P)=2}$ means that the system~\eqref{ex:PL-2mixtures_poly_abp_n4} has exactly two solutions in $\mathbb{C}$ (counted with multiplicity) for this $(\bm{a'}_{1:4},\bm{b'}_{1:4},p_1')$, given by $(\bm{x}_{1:4},\bm{y}_{1:4},p)=(\bm{a'}_{1:4},\bm{b'}_{1:4},p_1')$ or $(\bm{b'}_{1:4},\bm{a'}_{1:4},1-p_1')$.

Altogether, we have proved that the system~\eqref{ex:PL-2mixtures_poly_abp_n4} (and hence the system~\eqref{ex:PL-2mixtures_p}) has exactly two solutions in $\mathbb{C}$ (counted with multiplicity) for all $(\bm{a}_{2:4},\bm{b}_{2:4},p_1) \in Q_{PL,p}^{7}$ but a set $V_4$ of $\lambda_7$-measure zero.

\paragraph{Case $n \geq 4$.}
In the cases where $n \geq 4$, we split the system~\eqref{ex:PL-2mixtures_poly_p} into two parts. One is the system~\eqref{ex:PL-2mixtures_poly_abp_n4}, and the other is
{\small\begin{equation}\label{ex:pln > 5}
\begin{cases}
(p_1a_2a_i(1-b_2)+p_2b_2b_i(1-a_2))(1-x_2)(1-y_2) =(1-a_2)(1-b_2)(p x_2x_i(1-y_2)+(1-p)y_2y_i(1-x_2)), & \forall i \geq 5, \\
(p_1a_3a_i(1-b_3)+p_2b_3b_i(1-a_3))(1-x_3)(1-y_3) =(1-a_3)(1-b_3)(p x_3x_i(1-y_3) + (1-p) y_3y_i(1-x_3)), & \forall i \geq 5.
\end{cases}
\end{equation}}

From the case where $n=4$, we know there exists a $\lambda_7$-measure zero set $V_4$ such that the system~\eqref{ex:PL-2mixtures_poly_abp_n4} has exactly two solutions in $\mathbb{C}$ (counted with multiplicity) for all $(\bm{a}_{2:4},\bm{b}_{2:4},p_1) \in Q_{PL,p}^7 \setminus V_4$, given by
\[
(\bm{x}_{1:4},\bm{y}_{1:4},p)=(\bm{a}_{1:4},\bm{b}_{1:4},p_1) \text{ and } (\bm{b}_{1:4},\bm{a}_{1:4},1-p_1).
\]
To proceed, we determine $(x_i,y_i)$ for each $i \geq 5$. Let $c_i:=p_1a_i(1-b_i)$ and $d_i:=p_2b_i(1-a_i)$, for $i=2$ and $3$.
\begin{enumerate}
\item 
Plugging $(\bm{x}_{2:3},\bm{y}_{2:3},p)=(\bm{a}_{2:3},\bm{b}_{2:3},p_1)$ into the system~\eqref{ex:pln > 5} and simplifying, we obtain the following system of linear equations in $(x_i,y_i)$:
\begin{equation}\label{PL_pxiyi}
\begin{cases}
c_2x_i+d_2y_i=c_2a_i+d_2b_i, \\
c_3x_i+d_3y_i=c_3a_i+d_3b_i.
\end{cases}
\end{equation}
If the coefficient matrix of this system of linear equations is non-zero, i.e., $c_2d_3-c_3d_2 \neq 0$ (this is a condition on $a_2,a_3,b_2,b_3$), the system~\eqref{PL_pxiyi} has a unique solution in $\mathbb{C}$ (counted with multiplicity), given by $(x_i,y_i)=(a_i,b_i)$.
\item
Plugging $(\bm{x}_{2:3},\bm{y}_{2:3},p)=(\bm{b}_{2:3},\bm{a}_{2:3},1-p_1)$ into the system~\eqref{ex:pln > 5} and simplifying, we obtain the following system of linear equations in $(x_i,y_i)$:
\begin{equation}\label{PL_pxiyi2}
\begin{cases}
d_2x_i+c_2y_i=c_2a_i+d_2b_i, \\
d_3x_i+c_3y_i=c_3a_i+d_3b_i.
\end{cases}
\end{equation}
If the coefficient matrix of this system of linear equations is non-zero, i.e., $c_2d_3-c_3d_2 \neq 0$ (this is a condition on $a_2,a_3,b_2,b_3$), the system~\eqref{PL_pxiyi2} has a unique solution in $\mathbb{C}$ (counted with multiplicity), given by $(x_i,y_i)=(b_i,a_i)$. 
\end{enumerate}
Altogether, we can define
\[
V_{n}:=\{(\bm{a}_{2:n},\bm{b}_{2:n},p_1) \in \mathbb{C}^{2n-1}: (\bm{a}_{2:4},\bm{b}_{2:4},p_1) \in V_4 \text{ or } c_2d_3-c_3d_2 = 0\} \subseteq \mathbb{C}^{2n-1},
\]
which is of $\lambda_{2n-1}$-measure zero by Lemma~\ref{closed-0}, since it is defined by a non-zero polynomial. From the arguments above, for all $(\bm{a}_{2:n},\bm{b}_{2:n},p_1) \in Q^{2n-1}_{PL,p} \setminus V_{n}$, the system~\eqref{ex:PL-2mixtures_poly_p} (and hence the system~\eqref{ex:PL-2mixtures_p}) has exactly two solutions in $\mathbb{C}$ (counted with multiplicity), given by $(\bm{x},\bm{y},p)=(\bm{a},\bm{b},p_1)$ and $(\bm{b},\bm{a},1-p_1)$. This finishes the proof.
\end{proof}

\subsubsection{Known mixing probabilities}

We now prove the generic identifiability of the two mixtures of Plackett-Luce model for any $n \geq 3$ with given $p_1$ and $p_2$. Note that the system~\eqref{ex:PL-2mixtures} has at least one solution in $\mathbb{C}$ coming from the initial data $(\bm{x},\bm{y})=(\bm{a},\bm{b})$. Our goal is to show it is the unique solution in $\mathbb{C}$. For the concrete case $p_1=0.7$, the following is proved:

\begin{proposition}
\label{prop:ex-pl}
If $n \geq 3$ and $(p_1,p_2) = (0.7,0.3)$, the system~\eqref{ex:PL-2mixtures} has a unique solution in $\mathbb{C}$ (counted with multiplicity) for all $(\bm{a}_{2:n},\bm{b}_{2:n}) \in Q^{2n-2}_{PL}$ but a set of $\lambda_{2n-2}$-measure zero, given by $(\bm{x},\bm{y})=(\bm{a},\bm{b})$.
\end{proposition}
\begin{proof}[Proof of Proposition~\ref{prop:ex-pl}]
According to the discussion above, we wish to consider the following equation system, with equations given by linear combinations of the equations in the system~\eqref{ex:PL-2mixtures}:
\begin{equation}
\label{ex:PL-2mixtures-}
\begin{cases}
\eta_{1,i,\cdot} = p_1 \dfrac{x_1x_i}{1-x_1} + p_2 \dfrac{y_1y_i}{1-y_1}, & \forall i \in \{2,3,4\} \\
\eta_{2,i,\cdot} = p_1 \dfrac{x_2x_i}{1-x_2} + p_2 \dfrac{y_2y_i}{1-y_2}, & \forall i \in \{1,3,4\} \\
\eta_{3,i,\cdot} = p_1 \dfrac{x_3x_i}{1-x_3} + p_2 \dfrac{y_3y_i}{1-y_3}, & \forall i \in \{1,2,4\} \\
\eta_{1,i,\cdot} = p_1 \dfrac{x_1x_i}{1-x_1} + p_2  \dfrac{y_1y_i}{1-y_1}, & \forall i \geq 5, \\
\eta_{2,i,\cdot} = p_1 \dfrac{x_2x_i}{1-x_2} + p_2 \dfrac{y_2y_i}{1-y_2}, & \forall i \geq 5, \\
\eta_{3,i,\cdot} = p_1 \dfrac{x_3x_i}{1-x_3} + p_2 \dfrac{y_3y_i}{1-y_3}, & \forall i \geq 5,
\end{cases}
\end{equation}
where $\eta_{k,l,\cdot}$ corresponds to the probability that $k \succ l \succ others$. A more detailed derivation of the system~\eqref{ex:PL-2mixtures-} can be found in Appendix~\ref{apn:pl-three}. 

Clearly, solutions to the system~\eqref{ex:PL-2mixtures} are necessarily also solutions to the system~\eqref{ex:PL-2mixtures-}; thus, it suffices to prove the generic identifiability of the system~\eqref{ex:PL-2mixtures-}.

To apply the results from Section~\ref{sect:main-result}, we first translate the equation system~\eqref{ex:PL-2mixtures-} into the following equivalent system, with coefficients equal to polynomials in $(\bm{a},\bm{b})$:
\begin{equation}\label{ex:PL-2mixtures_poly_ab}
\begin{cases}
\forall i \in \{2,3,4\}, \\ \ (p_1a_1a_i(1-b_1)+p_2b_1b_i(1-a_1))(1-x_1)(1-y_1)=(1-a_1)(1-b_1)(p_1 x_1x_i(1-y_1) + p_2 y_1y_i(1-x_1)),  \\
\forall i \in \{1,3,4\}, \\ \ (p_1a_2a_i(1-b_2)+p_2b_2b_i(1-a_2))(1-x_2)(1-y_2)=(1-a_2)(1-b_2)(p_1 x_2x_i(1-y_2) + p_2 y_2y_i(1-x_2)), \\
\forall i \in \{1,2,4\}, \\ \ (p_1a_3a_i(1-b_3)+p_2b_3b_i(1-a_3))(1-x_3)(1-y_3)=(1-a_3)(1-b_3)(p_1 x_3x_i(1-y_3) + p_2 y_3y_i(1-x_3)), \\
\forall i \geq 5, \\
\ (p_1a_1a_i(1-b_1)+p_2b_1b_i(1-a_1))(1-x_1)(1-y_1)=(1-a_1)(1-b_1)(p_1 x_1x_i(1-y_1)+p_2 y_1y_i(1-x_1)), \\
\forall i \geq 5, \\
\ (p_1a_2a_i(1-b_2)+p_2b_2b_i(1-a_2))(1-x_2)(1-y_2) =(1-a_2)(1-b_2)(p_1 x_2x_i(1-y_2) + p_2 y_2y_i(1-x_2)), \\
\forall i \geq 5, \\
\ (p_1a_3a_i(1-b_3)+p_2b_3b_i(1-a_3))(1-x_3)(1-y_3) =(1-a_3)(1-b_3)(p_1 x_3x_i(1-y_3) + p_2 y_3y_i(1-x_3)),  \\
\forall i \in [n], \ t_i(1-x_i) = 1, \\
\forall i \in [n], \ h_i(1-y_i) = 1.
\end{cases}
\end{equation}

We will prove the generic identifiability of the system~\eqref{ex:PL-2mixtures_poly_ab}. We will first consider the case $n=4$, and then leverage the result to prove the cases $n \geq 5$.

\paragraph{Case $n=4$.}
In this case, the system~\eqref{ex:PL-2mixtures_poly_ab} becomes
\begin{equation}\label{ex:PL-2mixtures_poly_ab_n4}
\begin{cases}
\forall i \in \{2,3,4\}, \\
\ (p_1a_1a_i(1-b_1)+p_2b_1b_i(1-a_1))(1-x_1)(1-y_1)=(1-a_1)(1-b_1)(p_1 x_1x_i(1-y_1) + p_2 y_1y_i(1-x_1)), \\
\forall i \in \{1,3,4\}, \\
\ (p_1a_2a_i(1-b_2)+p_2b_2b_i(1-a_2))(1-x_2)(1-y_2)=(1-a_2)(1-b_2)(p_1 x_2x_i(1-y_2) + p_2 y_2y_i(1-x_2)),\\
\forall i \in \{1,2,4\}, \\ \ (p_1a_3a_i(1-b_3)+p_2b_3b_i(1-a_3))(1-x_3)(1-y_3)=(1-a_3)(1-b_3)(p_1 x_3x_i(1-y_3) + p_2 y_3y_i(1-x_3)),\\
\forall i \in [4], \ t_i(1-x_i) = 1, \\
\forall i \in [4], \ h_i(1-y_i) = 1.
\end{cases}
\end{equation}

To apply Theorem~\ref{thm:general-case} to the sytem~\eqref{ex:PL-2mixtures_poly_ab_n4}, we need to check Assumptions~\ref{assum:key-assumption1} and~\ref{assum:key-assumption2} with $\ell = 1$.

For Assumption~\ref{assum:key-assumption1}: It is clear that the system~\eqref{ex:PL-2mixtures_poly_ab_n4} has at least one solution in $\mathbb{C}$, given by $(\bm{x}_{1:4},\bm{y}_{1:4})=(\bm{a}_{1:4},\bm{b}_{1:4})$.

For Assumption~\ref{assum:key-assumption2}: 
We first compute the Gr\"{o}bner basis of the system~\eqref{ex:PL-2mixtures_poly_ab_n4}.
\begin{lstlisting}[language=Magma,caption=Gr\"{o}bner basis of Plackett-Luce models with known mixing probabilities, label=grobner_pl_model_p07,frame=single,basicstyle=\scriptsize]
P<x1,x2,x3,x4,y1,y2,y3,y4,t1,t2,t3,t4,h1,h2,h3,h4,
  a1,a2,a3,a4,b1,b2,b3,b4>:=FreeAlgebra(Rationals(),24,"lex");

I:=ideal<P|
(7/10*a1*a2*(1-b1)+(1-7/10)*b1*b2*(1-a1))*(1-x1)*(1-y1)-
(1-a1)*(1-b1)*(7/10*x1*x2*(1-y1)+(1-7/10)*y1*y2*(1-x1)),
(7/10*a1*a3*(1-b1)+(1-7/10)*b1*b3*(1-a1))*(1-x1)*(1-y1)-
(1-a1)*(1-b1)*(7/10*x1*x3*(1-y1)+(1-7/10)*y1*y3*(1-x1)),
(7/10*a1*a4*(1-b1)+(1-7/10)*b1*b4*(1-a1))*(1-x1)*(1-y1)-
(1-a1)*(1-b1)*(7/10*x1*x4*(1-y1)+(1-7/10)*y1*y4*(1-x1)),
(7/10*a2*a1*(1-b2)+(1-7/10)*b2*b1*(1-a2))*(1-x2)*(1-y2)-
(1-a2)*(1-b2)*(7/10*x2*x1*(1-y2)+(1-7/10)*y2*y1*(1-x2)),
(7/10*a2*a3*(1-b2)+(1-7/10)*b2*b3*(1-a2))*(1-x2)*(1-y2)-
(1-a2)*(1-b2)*(7/10*x2*x3*(1-y2)+(1-7/10)*y2*y3*(1-x2)),
(7/10*a2*a4*(1-b2)+(1-7/10)*b2*b4*(1-a2))*(1-x2)*(1-y2)-
(1-a2)*(1-b2)*(7/10*x2*x4*(1-y2)+(1-7/10)*y2*y4*(1-x2)),
(7/10*a3*a1*(1-b3)+(1-7/10)*b3*b1*(1-a3))*(1-x3)*(1-y3)-
(1-a3)*(1-b3)*(7/10*x3*x1*(1-y3)+(1-7/10)*y3*y1*(1-x3)),
(7/10*a3*a2*(1-b3)+(1-7/10)*b3*b2*(1-a3))*(1-x3)*(1-y3)-
(1-a3)*(1-b3)*(7/10*x3*x2*(1-y3)+(1-7/10)*y3*y2*(1-x3)),
(7/10*a3*a4*(1-b3)+(1-7/10)*b3*b4*(1-a3))*(1-x3)*(1-y3)-
(1-a3)*(1-b3)*(7/10*x3*x4*(1-y3)+(1-7/10)*y3*y4*(1-x3)),
t1*(1-x1)-1,
t2*(1-x2)-1,
t3*(1-x3)-1,
t4*(1-x4)-1,
h1*(1-y1)-1,
h2*(1-y2)-1,
h3*(1-y3)-1,
h4*(1-y4)-1>;

-> GroebnerBasis(I);
\end{lstlisting}
From the output, we obtain
\begin{equation}\label{pl_bad}
\mathrm{Bad}(\bm{a}_{2:4},\bm{b}_{2:4})=
\begin{cases}
7a_ia_j(b_i-1)+3(a_i-1)b_ib_j, 
& \forall i < j \in [4], \\
a_i-1,\ b_i-1, & \forall i \in [3].
\end{cases}
\end{equation}
We choose $(\bm{a'}_{1:4},\bm{b'}_{1:4})=(1/10,2/10,3/10,4/10;1/20,7/20,9/20,3/20)$. It is routine to check that $(\bm{a'}_{2:4},\bm{b'}_{2:4}) \in \mathbb{C}^6 \setminus Z(\bm{a}_{2:4},\bm{b}_{2:4})$ using equation~\eqref{pl_bad}. Next, we can check that the system~\eqref{ex:PL-2mixtures_poly_ab_n4} has a unique solution in $\mathbb{C}$ (counted with multiplicity) using Magma.
\begin{lstlisting}[language=Magma,caption=Dimension and degree computations of Plackett-Luce models with known mixing probabilities, label=pl_model_p07,frame=single,basicstyle=\scriptsize]
a:=[1/10,2/10,3/10,4/10];
b:=[1/20,7/20,9/20,3/20];
p1:=7/10;
p2:=3/10;

k:=Rationals();
A<x1,x2,x3,x4,y1,y2,y3,y4,t1,t2,t3,t4,h1,h2,h3,h4>:=AffineSpace(k,16);
P:=Scheme(A,
[
(p1*a[1]*a[2]*(1-b[1])+(1-p1)*b[1]*b[2]*(1-a[1]))*(1-x1)*(1-y1)-
(1-a[1])*(1-b[1])*(p1*x1*x2*(1-y1)+(1-p1)*y1*y2*(1-x1)),
(p1*a[1]*a[3]*(1-b[1])+(1-p1)*b[1]*b[3]*(1-a[1]))*(1-x1)*(1-y1)-
(1-a[1])*(1-b[1])*(p1*x1*x3*(1-y1)+(1-p1)*y1*y3*(1-x1)),
(p1*a[1]*a[4]*(1-b[1])+(1-p1)*b[1]*b[4]*(1-a[1]))*(1-x1)*(1-y1)-
(1-a[1])*(1-b[1])*(p1*x1*x4*(1-y1)+(1-p1)*y1*y4*(1-x1)),
(p1*a[2]*a[1]*(1-b[2])+(1-p1)*b[2]*b[1]*(1-a[2]))*(1-x2)*(1-y2)-
(1-a[2])*(1-b[2])*(p1*x2*x1*(1-y2)+(1-p1)*y2*y1*(1-x2)),
(p1*a[2]*a[3]*(1-b[2])+(1-p1)*b[2]*b[3]*(1-a[2]))*(1-x2)*(1-y2)-
(1-a[2])*(1-b[2])*(p1*x2*x3*(1-y2)+(1-p1)*y2*y3*(1-x2)),
(p1*a[2]*a[4]*(1-b[2])+(1-p1)*b[2]*b[4]*(1-a[2]))*(1-x2)*(1-y2)-
(1-a[2])*(1-b[2])*(p1*x2*x4*(1-y2)+(1-p1)*y2*y4*(1-x2)),
(p1*a[3]*a[1]*(1-b[3])+(1-p1)*b[3]*b[1]*(1-a[3]))*(1-x3)*(1-y3)-
(1-a[3])*(1-b[3])*(p1*x3*x1*(1-y3)+(1-p1)*y3*y1*(1-x3)),
(p1*a[3]*a[2]*(1-b[3])+(1-p1)*b[3]*b[2]*(1-a[3]))*(1-x3)*(1-y3)-
(1-a[3])*(1-b[3])*(p1*x3*x2*(1-y3)+(1-p1)*y3*y2*(1-x3)),
(p1*a[3]*a[4]*(1-b[3])+(1-p1)*b[3]*b[4]*(1-a[3]))*(1-x3)*(1-y3)-
(1-a[3])*(1-b[3])*(p1*x3*x4*(1-y3)+(1-p1)*y3*y4*(1-x3)),
t1*(1-x1)-1,
t2*(1-x2)-1,
t3*(1-x3)-1,
t4*(1-x4)-1,
h1*(1-y1)-1,
h2*(1-y2)-1,
h3*(1-y3)-1,
h4*(1-y4)-1
]);

-> Dimension(P);
-> @0@

-> Degree(P);
-> @1@
\end{lstlisting}
From Listing~\ref{pl_model_p07}, $\texttt{Dimension(P)=0}$ and $\texttt{Degree(P)=1}$ means that the system~\eqref{ex:PL-2mixtures_poly_ab_n4} has a unique solution in $\mathbb{C}$ (counted with multiplicity) for this choice of $(\bm{a'}_{2:4},\bm{b'}_{2:4})$. 

Thus, by Theorem~\ref{thm:general-case}, we have proved that the system~\eqref{ex:PL-2mixtures_poly_ab_n4} (and hence the system~\eqref{ex:PL-2mixtures}) has a unique solution in $\mathbb{C}$ (counted with multiplicity) for all $(\bm{a}_{2:4},\bm{b}_{2:4}) \in Q_{PL}^{6}$ but a set $V_4$ of $\lambda_{6}$-measure zero.
\paragraph{Case $n \geq 5$.}
We consider two parts of the system~\eqref{ex:PL-2mixtures_poly_ab}. One is the system~\eqref{ex:PL-2mixtures_poly_ab_n4}, and the other is 
\begin{equation}\label{ex:PL-2mixtures_poly_ab5n}
\begin{cases}
\forall i \geq 5, \\
\ (p_1a_2a_i(1-b_2)+p_2b_2b_i(1-a_2))(1-x_2)(1-y_2) =(1-a_2)(1-b_2)(p_1x_2x_i(1-y_2)+p_2 y_2y_i(1-x_2)), \\
 \forall i \geq 5, \\
\ (p_1a_3a_i(1-b_3)+p_2b_3b_i(1-a_3))(1-x_3)(1-y_3) =(1-a_3)(1-b_3)(p_1x_3x_i(1-y_3)+p_2 y_3y_i(1-x_3)). \\
\end{cases}
\end{equation}

From the case where $n=4$, we know there exists a $\lambda_6$-measure zero set $V_4$ such that the system~\eqref{ex:PL-2mixtures_poly_ab_n4} has a unique solution in $\mathbb{C}$ (counted with multiplicity) for all $(\bm{a}_{2:4},\bm{b}_{2:4}) \in Q_{PL}^6 \setminus V_4$, given by
\[
(\bm{x}_{1:4},\bm{y}_{1:4})=(\bm{a}_{1:4},\bm{b}_{1:4}).
\]
To proceed, we determine $(x_i,y_i)$ for each $i \geq 5$. Plugging $(\bm{x}_{2:3},\bm{y}_{2:3})=(\bm{a}_{2:3},\bm{b}_{2:3})$ into the system~\eqref{ex:PL-2mixtures_poly_ab5n} and simplifying, we obtain the following system of linear equations in $(x_i,y_i)$:
\begin{equation}\label{PL_xiyi}
\begin{cases}
c_2x_i+d_2y_i=c_2a_i+d_2b_i, \\
c_3x_i+d_3y_i=c_3a_i+d_3b_i,
\end{cases}
\end{equation}
where $c_i:=p_1a_i(1-b_i)$ and $d_i:=p_2b_i(1-a_i)$ for $i=2$ and $3$. If the coefficient matrix of this system of linear equations is non-zero, i.e., $c_2d_3-c_3d_2 \neq 0$ (this is a condition on $a_2,a_3,b_2$, and $b_3$), the system~\eqref{PL_xiyi} has a unique solution in $\mathbb{C}$ (counted with multiplicity), given by $(x_i,y_i)=(a_i,b_i)$. 

Altogether, we can define
\[
V_{n}:=\{(\bm{a}_{2:n},\bm{b}_{2:n}) \in \mathbb{C}^{2n-2}: (\bm{a}_{2:4},\bm{b}_{2:4}) \in V_4 \text{ or } c_2d_3-c_3d_2 = 0\} \subseteq \mathbb{C}^{2n-2},
\]
which is of $\lambda_{2n-2}$-measure zero by Lemma~\ref{closed-0}, since it is defined by a non-zero polynomial. From the arguments above, for all $(\bm{a}_{2:n},\bm{b}_{2:n}) \in Q^{2n-2}_{PL} \setminus V_{n}$, the system~\eqref{ex:PL-2mixtures_poly_ab} (and hence the system~\eqref{ex:PL-2mixtures}) has a unique solution in $\mathbb{C}$ (counted with multiplicity), given by $(\bm{x},\bm{y})=(\bm{a},\bm{b})$. This finishes the proof.
\end{proof}

\begin{remark}
In practice, the number $n$ of items being ranked is often quite large, so
the condition on $n$ in our theorems/propositions are typically satisfied. Nonetheless, in Table~\ref{tab:tightness} in Appendix~\ref{apn:tight}, we provide details about the smallest possible $n$ for generic identifiability vs.\ the smallest $n$ in our results. We see that our requirements on $n$ are nearly tight for all three mixtures of ranking models.
\end{remark}

\subsection{Tightness of our results on the number of items $n$}\label{apn:tight}

\begin{table}[htp!]
\caption{Tightness of our results on the number of items $n$.}
\label{tab:tightness}
\begin{center}
\begin{tabular}{c|c|c}
\toprule
Mixture model & The smallest possible & The smallest guaranteed $n$\\
 & generic tight $n$ &  \\
\midrule
\shortstack{ BTL model (known $p_1$)} & 4 & 5\\
\hline
\shortstack{BTL model (unknown $p_1$)} & 5 & 5\\
\hline
\shortstack{MNL model with 3-slate (known $p_1$)} & 4 & 4 \\
\hline
\shortstack{MNL model with 3-slate (unknown $p_1$)} & 4 & 4 \\
\hline
\shortstack{Plackett-Luce model (known $p_1$)} & 3 & 3\\
\hline
\shortstack{Plackett-Luce model (unknown $p_1$)} & 3 & 4\\
\bottomrule
\end{tabular}  
\end{center}
\end{table}

Table~\ref{tab:tightness} summarizes results concerning the tightness of $n$. The second column gives the smallest $n$ that is possible for any mixture model to achieve generic identifiability, because the number of equations is larger than or equal to the number of variables. The third column gives the smallest $n$ that our results can guarantee. We can see that except for the first and last rows, the smallest possible tight $n$ equals to the smallest guaranteed $n$; for the first and last rows, they are quite close.

\subsection{Mixtures of MNL models with 2-slate and 3-slate}\label{apn:mnl23}
In this example, we consider mixtures of MNL models with 2-slate and 3-slate for $n \geq 3$. Let $\bm{a}_{1:n}$, $\bm{b}_{1:n}$ be the score parameters of the two mixtures. We then obtain
\begin{align}\label{eq:eta_def_mnl23}
\begin{split}
    \forall i \neq j \in [n], \quad \eta_{i,j} &= p_1 \dfrac{a_i}{a_i + a_j} + p_2 \dfrac{b_i}{b_i + b_j}, \\
    \forall i \neq j \neq k \in [n], \quad \eta_{i,j,k} & = \ p_1 \dfrac{a_i}{a_i + a_j + a_k} + p_2 \frac{b_i}{b_i + b_j + b_k}. \\
\end{split}
\end{align}
Here, we choose to scale up $\bm{a}_{1:n}$ by multiplying by a constant so that $a_1 = 1$, and similarly manipulate $\bm{b}_{1:n}$ to have $b_1 = 1$. This does not influence the values of the $\eta_{i,j}$'s or $\eta_{i,j,k}$'s. In~\cite{chierichetti2018learning}, the authors scale up $\bm{a}_{1:n}$ to get $a_1 + a_2 + a_3 = 1$ and $\bm{b}_{1:n}$ to have $b_1 + b_2 + b_3 = 1$. This choice is for the convenience of defining the set of bad parameters.

Given $p_1$ and $p_2$, to determine the scores of two mixtures, we try to solve the following equation system in $(\bm{x}, \bm{y}) :=x_{1:n}, y_{1:n})$, for $n \geq 3$:
\begin{align}\label{ex:mnl23_2mixture}
\begin{cases}
x_1=y_1=1, \\
p_1 \dfrac{x_i}{x_i + x_j} + p_2 \dfrac{y_i}{y_i + y_j} = \eta_{i,j}, & \forall i\neq j \in [n], \\
p_1 \dfrac{x_i}{x_i + x_j + x_k} + p_2 \dfrac{y_i}{y_i + y_j + y_k} = \eta_{i,j,k}, & \forall i\neq j 
\neq k\in [n].
\end{cases}
\end{align}

Let $Q^{2n-2}_{MNL23} := \prod_{i=1}^{2n-2}\left[r_i,R_i\right] \subseteq \mathbb{R}^{2n-2}$ be the domain of $(\bm{a}_{2:n},\bm{b}_{2:n})$, where $R_i > r_i > 0$. Then the set of bad parameters that do not have the identifiability property is
\begin{align}
\begin{split}
    N^{2n-2}_{MNL23} = & \left\{(\bm{a}_{2:n},\bm{b}_{2:n}) \in Q^{2n-2}_{MNL23}: \exists \left(\bm{a}_{2:n}^{\#},\bm{b}^{\#}_{2:n}\right) \in Q^{2n-2}_{MNL23}, \text{ s.t. }  (\bm{a}^{\#}_{2:n} \neq \bm{a}_{2:n} \vee \bm{b}^{\#}_{2:n} \neq \bm{b}_{2:n}) \quad \wedge \right. \\
    & \left. \left(\forall i < j \in [n],\  \eta_{i,j}(\bm{a}^{\#},\bm{b}^{\#}) = \eta_{i,j}(\bm{a},\bm{b})\mbox{ for } a_1^{\#} =b_1^{\#} = a_1 = b_1 = 1\right) \quad \wedge \right. \\
    & \left. \left(\forall i,j,k\in [n], \ \eta_{i,j,k}(\bm{a}^{\#},\bm{b}^{\#}) = \eta_{i,j,k}(\bm{a},\bm{b})\mbox{ for } a_1^{\#} =b_1^{\#} = a_1 = b_1 = 1\right) \right\}.
\end{split}
\end{align}
We will later show that $N^{2n-2}_{MNL23}$ has Lebesgue measure zero for the mixtures of MNL models with 2-slate and 3-slate being identifiable.

When we consider $p$ as an additional variable, our domain $Q_{MNL23,p}^{2n-1}$ becomes $Q_{MNL23}^{2n-2} \times (0,1) \subseteq \real^{2n-1}$, and we define the set of bad parameters to be
\begin{align}
\begin{split}
    N^{2n-1}_{MNL23,p} = & \left\{(\bm{a}_{2:n},\bm{b}_{2:n},p_1) \in Q^{2n-1}_{MNL23,p}: \exists \left(\bm{a}_{2:n}^{\#},\bm{b}^{\#}_{2:n},p^{\#}\right) \in Q^{2n-1}_{MNL23,p}, \text{ s.t. }  \right. \\
    & \left. (\bm{a}^{\#}_{2:n} \neq \bm{a}_{2:n} \vee \bm{b}^{\#}_{2:n} \neq \bm{b}_{2:n} \vee p^{\#} \neq p_1) \quad \wedge \right. \\
    & \left. \left(\forall i < j \in [n],\  \eta_{i,j}(\bm{a}^{\#},\bm{b}^{\#},p^{\#}) = \eta_{i,j}(\bm{a},\bm{b},p_1)\mbox{ for } a_1^{\#} =b_1^{\#} = a_1 = b_1 = 1\right) \quad \wedge \right. \\
    & \left. \left(\forall i,j,k\in [n], \ \eta_{i,j,k}(\bm{a}^{\#},\bm{b}^{\#},p^{\#}) = \eta_{i,j,k}(\bm{a},\bm{b},p_1)\mbox{ for } a_1^{\#} =b_1^{\#} = a_1 = b_1 = 1\right) \right\}.
\end{split}
\end{align}

We will consider two cases, depending on whether the mixing probabilities are known. For the case where the mixing probabilities are known, ~\cite[Theorem 13]{chierichetti2018learning} solved the identifiability issue for a uniform mixture ($p_1=0.5$), using a different technique.

\subsubsection{Unknown mixing probabilities}
In this case of MNL model with 2-slate and 3-slate in parameter space $(\bm{a}_{2:n},\bm{b}_{2:n},p_1)$, we study the following equation system in $(\bm{x},\bm{y},p)$:
\begin{align}\label{ex:mnl23_2mixture_p}
\begin{cases}
x_1 = y_1 = 1. \\
p \dfrac{x_i}{x_i + x_j} + (1-p) \dfrac{y_i}{y_i + y_j} = \eta_{i,j}, & \forall i \neq j \in [n], \\
p \dfrac{x_i}{x_i + x_j + x_k} + (1-p) \dfrac{y_i}{y_i + y_j + y_k} = \eta_{i,j,k}, & \forall i \neq j \neq k \in [n].
\end{cases}
\end{align}
\begin{proposition}\label{prop:ex-mnl_p}
If $n \geq 4$, the system~\eqref{ex:mnl23_2mixture_p} has exactly two solutions in $\mathbb{C}$ (counted with multiplicity) for all $(\bm{a}_{2:n},\bm{b}_{2:n},p_1) \in Q^{2n-2}_{MNL23} \times (0,1)$ but a set of $\lambda_{2n-1}$-measure zero, given by $(\bm{x},\bm{y},p)=(\bm{a},\bm{b},p_1)$ and $(\bm{x},\bm{y},p)=(\bm{b},\bm{a},1-p_1)$.
\end{proposition}

\begin{proof}[Proof of Proposition~\ref{prop:ex-mnl_p}]
To apply Theorem~\ref{thm:general-case}, we first translate the system~\eqref{ex:mnl23_2mixture_p} into the following (equivalent) equation system by multiplying by $(x_i+x_j)(y_i+y_j)$ or $(x_i+x_j+x_k)(y_i+y_j+y_k)$ to both sides:
\begin{align}\label{eq:mnl23_2mixtures_poly_abp}
\begin{cases}
x_1=y_1=1, \\
px_i(y_i+y_j)+(1-p)y_i(x_i+x_j)-\eta_{i,j}(x_i+x_j)(y_i+y_j) = 0, & \forall i \neq j \in [n], \\
px_i(y_i+y_j+y_k)+(1-p)y_i(x_i+x_j+x_k)-\eta_{i,j,k}(x_i+x_j+x_k)(y_i+y_j+y_k) = 0, &\forall i \neq j \neq k \in [n], \\
t_{i,j}(x_i+x_j) = 1, & \forall i \neq j \in [n],\\
h_{i,j}(y_i+y_j) = 1, & \forall i \neq j \in [n],\\
t_{i,j,k}(x_i+x_j+x_k) = 1, & \forall i \neq j \neq k \in [n],\\
h_{i,j,k}(y_i+y_j+y_k) = 1, & \forall i \neq j \neq k \in [n]. \\
\end{cases}
\end{align}
Similar to the case in Section \ref{apn-sub:examples_btl_proofs_p}, we translate the system~\eqref{eq:mnl23_2mixtures_poly_abp} into a system, denoted by $\mathcal{P}_{MNL23,abp}$, whose coefficients are given by polynomials in $(\bm{a},\bm{b},p_1)$, by multiplying by $(a_i+a_j)(b_i+b_j)$ or $(a_i+a_j+a_k)(b_i+b_j+b_k)$ to both sides. Note that the parameter space $Q_{MNL23}^{2n-1} \subseteq \mathbb{R}_+^{2n-2}$ guarantees that neither $(a_i+a_j)(b_i+b_j)$ nor $(a_i+a_j+a_k)(b_i+b_j+b_k)$ is zero, so  generic identifiability of $\mathcal{P}_{MNL23,abp}$ is equivalent to that of the system~\eqref{eq:mnl23_2mixtures_poly_abp}, hence also equivalent to that of the system~\eqref{ex:mnl23_2mixture_p}.

The arguments are similar as before, by checking Assumptions~\ref{assum:key-assumption1} and~\ref{assum:key-assumption2} in Theorem~\ref{thm:general-case} to conclude that the polynomial system has exactly two solutions in $\mathbb{C}$ (counted with multiplicity), implying that the MNL model with 2-slate and 3-slate is generically identifiable up to reordering. We omit the details and simply state the corresponding chunks of Magma code here for computing the Gr\"{o}bner basis and the number of solutions for some specific values of $(\bm{a'}_{1:4},\bm{b'}_{1:4})=(1,2,3,4;1,5,4,2)$ below.

\begin{lstlisting}[language=Magma,caption=Gr\"{o}bner basis of MNL models involving 2-\&3-slate with unknown mixing probabilities,label={lst:23mix_p},frame=single,basicstyle=\scriptsize]
P<p,x2,x3,x4,y2,y3,y4,a2,a3,a4,b2,b3,b4,p1>:=FreeAlgebra(Rationals(),14,"lex");

I:=ideal<P|
(1+a2)*(1+b2)*(p*1*(1+y2)+(1-p)*1*(1+x2))-
(p1*1*(1+b2)+(1-p1)*1*(1+a2))*(1+x2)*(1+y2),
(1+a3)*(1+b3)*(p*1*(1+y3)+(1-p)*1*(1+x3))-
(p1*1*(1+b3)+(1-p1)*1*(1+a3))*(1+x3)*(1+y3),
(1+a4)*(1+b4)*(p*1*(1+y4)+(1-p)*1*(1+x4))-
(p1*1*(1+b4)+(1-p1)*1*(1+a4))*(1+x4)*(1+y4),
(a2+a3)*(b2+b3)*(p*x2*(y2+y3)+(1-p)*y2*(x2+x3))-
(p1*a2*(b2+b3)+(1-p1)*b2*(a2+a3))*(x2+x3)*(y2+y3),
(a2+a4)*(b2+b4)*(p*x2*(y2+y4)+(1-p)*y2*(x2+x4))-
(p1*a2*(b2+b4)+(1-p1)*b2*(a2+a4))*(x2+x4)*(y2+y4),
(a3+a4)*(b3+b4)*(p*x3*(y3+y4)+(1-p)*y3*(x3+x4))-
(p1*a3*(b3+b4)+(1-p1)*b3*(a3+a4))*(x3+x4)*(y3+y4),
(1+a2+a3)*(1+b2+b3)*(p*1*(1+y2+y3)+(1-p)*1*(1+x2+x3))-
(p1*1*(1+b2+b3)+(1-p1)*1*(1+a2+a3))*(1+x2+x3)*(1+y2+y3),
(1+a2+a4)*(1+b2+b4)*(p*1*(1+y2+y4)+(1-p)*1*(1+x2+x4))-
(p1*1*(1+b2+b4)+(1-p1)*1*(1+a2+a4))*(1+x2+x4)*(1+y2+y4),
(1+a2+a3)*(1+b2+b3)*(p*x2*(1+y2+y3)+(1-p)*y2*(1+x2+x3))-
(p1*a2*(1+b2+b3)+(1-p1)*b2*(1+a2+a3))*(1+x2+x3)*(1+y2+y3),
(1+a2+a4)*(1+b2+b4)*(p*x2*(1+y2+y4)+(1-p)*y2*(1+x2+x4))-
(p1*a2*(1+b2+b4)+(1-p1)*b2*(1+a2+a4))*(1+x2+x4)*(1+y2+y4)>;

-> GroebnerBasis(I);
\end{lstlisting}

\begin{lstlisting}[language=Magma,caption=Dimension and degree computations of MNL models involving 2-\&3-slate with unknown mixing probabilities, label=mnl_model_abp,frame=single,basicstyle=\scriptsize]
a:=[1,2,3,4];
b:=[1,5,4,2];
p1:=7/10;

k:=Rationals();
A<x2,x3,x4,y2,y3,y4,p>:=AffineSpace(k,7);
P:=Scheme(A,
[
(1+a[2])*(1+b[2])*(p*1*(1+y2)+(1-p)*1*(1+x2))-
(p1*1*(1+b[2])+(1-p1)*1*(1+a[2]))*(1+x2)*(1+y2),
(1+a[3])*(1+b[3])*(p*1*(1+y3)+(1-p)*1*(1+x3))-
(p1*1*(1+b[3])+(1-p1)*1*(1+a[3]))*(1+x3)*(1+y3),
(1+a[4])*(1+b[4])*(p*1*(1+y4)+(1-p)*1*(1+x4))-
(p1*1*(1+b[4])+(1-p1)*1*(1+a[4]))*(1+x4)*(1+y4),
(a[2]+a[3])*(b[2]+b[3])*(p*x2*(y2+y3)+(1-p)*y2*(x2+x3))-
(p1*a[2]*(b[2]+b[3])+(1-p1)*b[2]*(a[2]+a[3]))*(x2+x3)*(y2+y3),
(a[2]+a[4])*(b[2]+b[4])*(p*x2*(y2+y4)+(1-p)*y2*(x2+x4))-
(p1*a[2]*(b[2]+b[4])+(1-p1)*b[2]*(a[2]+a[4]))*(x2+x4)*(y2+y4),
(a[3]+a[4])*(b[3]+b[4])*(p*x3*(y3+y4)+(1-p)*y3*(x3+x4))-
(p1*a[3]*(b[3]+b[4])+(1-p1)*b[3]*(a[3]+a[4]))*(x3+x4)*(y3+y4),
(1+a[2]+a[3])*(1+b[2]+b[3])*(p*1*(1+y2+y3)+(1-p)*1*(1+x2+x3))-
(p1*1*(1+b[2]+b[3])+(1-p1)*1*(1+a[2]+a[3]))*(1+x2+x3)*(1+y2+y3),
(1+a[2]+a[4])*(1+b[2]+b[4])*(p*1*(1+y2+y4)+(1-p)*1*(1+x2+x4))-
(p1*1*(1+b[2]+b[4])+(1-p1)*1*(1+a[2]+a[4]))*(1+x2+x4)*(1+y2+y4),
(1+a[2]+a[3])*(1+b[2]+b[3])*(p*x2*(1+y2+y3)+(1-p)*y2*(1+x2+x3))-
(p1*a[2]*(1+b[2]+b[3])+(1-p1)*b[2]*(1+a[2]+a[3]))*(1+x2+x3)*(1+y2+y3),
(1+a[2]+a[4])*(1+b[2]+b[4])*(p*x2*(1+y2+y4)+(1-p)*y2*(1+x2+x4))-
(p1*a[2]*(1+b[2]+b[4])+(1-p1)*b[2]*(1+a[2]+a[4]))*(1+x2+x4)*(1+y2+y4)
]);

-> Dimension(P);
-> @0@

-> Degree(P);
-> @2@
\end{lstlisting}
\end{proof}

\subsubsection{Known mixing probabilities}
We separately consider the two cases $p_1 \neq 0.5$ and $p_1 = 0.5$: For $p_1 \neq 0.5$, we show that the equation system achieves generic identifiability. For $p_1 = 0.5$, we show that the equation system achieves generic identifiability up to reordering. 
A proposition is written rigorously below:
\begin{proposition}\label{prop:ex-mnl23}
Suppose $n \geq 3$. If $(p_1,p_2) = (0.7,0.3)$, the system~\eqref{ex:mnl23_2mixture} has a unique solution in $\mathbb{C}$ (counted with multiplicity) for all $(\bm{a}_{2:n},\bm{b}_{2:n}) \in Q^{2n-2}_{MNL23}$ but a set of $\lambda_{2n-2}$-measure zero, given by $(\bm{x},\bm{y})=(\bm{a},\bm{b})$. 
Thus, we have generic identifiability of the two mixtures of MNL model with 2-\&3-slate.
\end{proposition}
\begin{proof}[Proof of Proposition~\ref{prop:ex-mnl23}]
To apply the results from Section \ref{sect:main-result}, we first translate the system~\eqref{ex:mnl23_2mixture} into the following (equivalent) equation system by multiplying by $(x_i+x_j)(y_i+y_j)$ or $(x_i+x_j+x_k)(y_i+y_j+y_k)$ to both sides:
\begin{equation}\label{eq:mnl23_2mixtures_poly}
\begin{cases}
    x_1=y_1 = 1, \\
    p_1x_i(y_i+y_j)+p_2y_i(x_i+x_j)-\eta_{i,j}(x_i+x_j)(y_i+y_j) = 0, & \forall i \neq j \in [n], \\
    p_1x_i(y_i+y_j+y_k)+p_2y_i(x_i+x_j+x_k)-\eta_{i,j,k}(x_i+x_j+x_k)(y_i+y_j+y_k) = 0, & \forall i \neq j \neq k \in [n], \\
    t_{i,j}(x_i+x_j) = 1, & \forall i \neq j \in [n],\\
    h_{i,j}(y_i+y_j) = 1, & \forall i \neq j \in [n],\\
    t_{i,j,k}(x_i+x_j+x_k) = 1, & \forall i \neq j \neq k \in [n],\\
    h_{i,j,k}(y_i+y_j+y_k) = 1, & \forall i \neq j \neq k \in [n].
\end{cases}
\end{equation}
Similar to the case in Section \ref{apn-sub:examples_btl_proofs_p}, we can translate the system~\eqref{eq:mnl23_2mixtures_poly} into the following equation system, with coefficients given by polynomials in $(\bm{a},\bm{b},p_1)$, by multiplying $(a_i+a_j)(b_i+b_j)$ or $(a_i+a_j+a_k)(b_i+b_j+b_k)$ to both sides:
\begin{equation}\label{eq:mnl23_2mixtures_poly'}
\begin{cases}
x_1=y_1 = 1, \\
(a_i+a_j)(b_i+b_j)(p_1x_i(y_i+y_j)+p_2y_i(x_i+x_j)) =\\
(p_1a_i(b_i+b_j)+p_2b_i(a_i+a_j))(x_i+x_j)(y_i+y_j), & \forall i \neq j \in [n], \\
(a_i+a_j+a_k)(b_i+b_j+b_k)(p_1x_i(y_i+y_j+y_k)+p_2y_i(x_i+x_j+x_k)) =\\
(p_1a_i(b_i+b_j+b_k)+p_2b_i(a_i+a_j+a_k))(x_i+x_j+x_k)(y_i+y_j+y_k), & \forall i \neq j \neq k \in [n], \\
t_{i,j}(x_i+x_j) = 1, h_{i,j}(y_i+y_j) = 1, & \forall i \neq j \in [n], \\
t_{i,j,k}(x_i+x_j+x_k) = 1, h_{i,j,k}(y_i+y_j+y_k) = 1, & \forall i \neq j \neq k \in [n],
\end{cases}
\end{equation}
Note that the parameter space $Q_{MNL23}^{2n-2} \subseteq \mathbb{R}_+^{2n-2}$ guarantees that both $(a_i+a_j)(b_i+b_j)$ and $(a_i+a_j+a_k)(b_i+b_j+b_k)$ are non-zero, so generic identifiability of the new polynomial system~\eqref{eq:mnl23_2mixtures_poly'} is equivalent to that of the system~\eqref{eq:mnl23_2mixtures_poly},  hence also the system~\eqref{ex:mnl23_2mixture}.

We now consider the case when $n=3$, and later make use of this result to prove the cases where $n \geq 4$.

\paragraph{Case $n=3$.}
In this case, we consider the following subset of the system~\eqref{eq:mnl23_2mixtures_poly'}:
\begin{equation}\label{eq:MNL23_poly_13}
\begin{cases}
x_1=y_1 = 1, \\
(a_i+a_j)(b_i+b_j)(p_1x_i(y_i+y_j)+p_2y_i(x_i+x_j))-\\
(p_1a_i(b_i+b_j)+p_2b_i(a_i+a_j))(x_i+x_j)(y_i+y_j) = 0, &  \forall (i,j) \in \{(1,2),(1,3),(2,3)\}, \\
(a_i+a_j+a_k)(b_i+b_j+b_k)(p_1x_i(y_i+y_j+y_k)+p_2y_i(x_i+x_j+x_k))= \\
(p_1a_i(b_i+b_j+b_k)+p_2b_i(a_i+a_j+a_k))(x_i+x_j+x_k)(y_i+y_j+y_k), & \forall (i,j,k) \in \{(1,2,3),(2,1,3)\}.
\end{cases}
\end{equation}
We now compute its Gr\"{o}bner basis via Magma.
\begin{lstlisting}[language=Magma,caption=Gr\"{o}bner basis of MNL models involving 2-\&3-slate with known mixing probabilities,label={lst:23mix_p07},frame=single,basicstyle=\scriptsize]
P<x1,x2,x3,y1,y2,y3,a1,a2,a3,b1,b2,b3>:=FreeAlgebra(Rationals(),12,"lex");

I:=ideal<P|x1-1,y1-1,
(a1+a2)*(b1+b2)*(7/10*x1*(y1+y2)+(1-7/10)*y1*(x1+x2))-
(7/10*a1*(b1+b2)+(1-7/10)*b1*(a1+a2))*(x1+x2)*(y1+y2),
(a1+a3)*(b1+b3)*(7/10*x1*(y1+y3)+(1-7/10)*y1*(x1+x3))-
(7/10*a1*(b1+b3)+(1-7/10)*b1*(a1+a3))*(x1+x3)*(y1+y3),
(a2+a3)*(b2+b3)*(7/10*x2*(y2+y3)+(1-7/10)*y2*(x2+x3))-
(7/10*a2*(b2+b3)+(1-7/10)*b2*(a2+a3))*(x2+x3)*(y2+y3),
(a1+a2+a3)*(b1+b2+b3)*(7/10*x1*(y1+y2+y3)+(1-7/10)*y1*(x1+x2+x3))-
(7/10*a1*(b1+b2+b3)+(1-7/10)*b1*(a1+a2+a3))*(x1+x2+x3)*(y1+y2+y3),
(a1+a2+a3)*(b1+b2+b3)*(7/10*x2*(y1+y2+y3)+(1-7/10)*y2*(x1+x2+x3))-
(7/10*a2*(b1+b2+b3)+(1-7/10)*b2*(a1+a2+a3))*(x1+x2+x3)*(y1+y2+y3)
>;

-> GroebnerBasis(I);
\end{lstlisting}
This gives
\begin{equation}\label{MNL_bad'}
\mathrm{Bad}(\bm{a}_{2:3},\bm{b}_{2:3})=
\begin{cases}
10 + 3 a_i + 7 b_i, (4 - 3 a_i) b_i, 3 - a_i (4 + 7 b_i),3 b_i+a_i(7 + 10 b_i), &\forall i \in \{2,3\}\\
7a_i+3b_i,a_i+b_i, &\forall i \in \{2,3\} \\
3 a_ib_j+7 a_jb_i,7+7a_j+4 b_i+7b_j,7 b_i+a_i(7+4 b_j + 7 b_i),&\forall i\neq j \in \{2,3\} \\
(3-4a_i)b_j+3a_j(1+b_j),10a_ib_i+3a_jb_i+7a_ib_j,&\forall i\neq j \in \{2,3\} \\
3a_ib_i-4a_jb_i-7ajb_j,3(1+a_j+b_j)-4a_i, &\forall i\neq j \in \{2,3\} \\
-3(1+a_2)b_3-a_3(7+7b_2+10b_3), \\
7-3(1+a_3)b_2-a_2(7+10b_2+7b_3), \\
\end{cases}
\end{equation}
Based on Theorem~\ref{thm:general-case}, by checking Assumptions~\ref{assum:key-assumption1} and~\ref{assum:key-assumption2}, we can conclude that the system~\eqref{eq:MNL23_poly_13} has a unique solution in $\mathbb{C}$ (counted with multiplicity) for all $(\bm{a}_{2:3},\bm{b}_{2:3}) \in Q_{MNL23}^{4}$ but a set of $\lambda_{4}$-measure zero.

Assumption~\ref{assum:key-assumption1}: This is clear, since $(\bm{x}_{1:3},\bm{y}_{1:3})=\left(\bm{a}_{1:3},\bm{b}_{1:3}\right)$ is a solution for all $(\bm{a}_{2:3},\bm{b}_{2:3}) \in \mathbb{C}^{4}$.

Assumption~\ref{assum:key-assumption2}: We choose $(\bm{a'}_{1:3},\bm{b'}_{1:3})=(1,2,3;1,5,4)$. It is routine to check that $(\bm{a'}_{2:3},\bm{b'}_{2:3}) \notin Z(\bm{a}_{2:3},\bm{b}_{2:3})$ using~\eqref{MNL_bad'}. Furthermore, we can use the Magma code in Listing~\ref{lst:23mix_p07} to check that the system~\eqref{eq:MNL23_poly_13} has exactly one solution in $\mathbb{C}$ (counted with multiplicity) for this $(\bm{a'}_{1:3},\bm{b'}_{1:3})$.  
\begin{lstlisting}[language=Magma,caption=Dimension and degree computations of MNL models involving 2-\&3-slate with known probabilities,label={lst:23mix_p07'},frame=single,basicstyle=\scriptsize]
a:=[1,2,3];
b:=[1,5,4];
p1:=7/10;

k:=Rationals();
A<x1,x2,x3,y1,y2,y3>:=AffineSpace(k,6);
P:=Scheme(A,[x1-1,y1-1,
(a[1]+a[2])*(b[1]+b[2])*(p1*x1*(y1+y2)+(1-p1)*y1*(x1+x2))-
(p1*a[1]*(b[1]+b[2])+(1-p1)*b[1]*(a[1]+a[2]))*(x1+x2)*(y1+y2),
(a[1]+a[3])*(b[1]+b[3])*(p1*x1*(y1+y3)+(1-p1)*y1*(x1+x3))-
(p1*a[1]*(b[1]+b[3])+(1-p1)*b[1]*(a[1]+a[3]))*(x1+x3)*(y1+y3),
(a[2]+a[3])*(b[2]+b[3])*(p1*x2*(y2+y3)+(1-p1)*y2*(x2+x3))-
(p1*a[2]*(b[2]+b[3])+(1-p1)*b[2]*(a[2]+a[3]))*(x2+x3)*(y2+y3),
(a[1]+a[2]+a[3])*(b[1]+b[2]+b[3])*(p1*x1*(y1+y2+y3)+(1-p1)*y1*(x1+x2+x3))-
(p1*a[1]*(b[1]+b[2]+b[3])+(1-p1)*b[1]*(a[1]+a[2]+a[3]))*(x1+x2+x3)*(y1+y2+y3),
(a[1]+a[2]+a[3])*(b[1]+b[2]+b[3])*(p1*x2*(y1+y2+y3)+(1-p1)*y2*(x1+x2+x3))-
(p1*a[2]*(b[1]+b[2]+b[3])+(1-p1)*b[2]*(a[1]+a[2]+a[3]))*(x1+x2+x3)*(y1+y2+y3)
]);

-> Dimension(P);
-> @0@

-> Degree(P);
-> @1@ 
\end{lstlisting}
Therefore, applying Theorem~\ref{thm:general-case} concludes the proof for $n=3$.

\paragraph{Case $n \geq 4$.}
For the case $n \geq 4$, we first consider the subset of equations~\eqref{eq:MNL23_poly_13} of~\eqref{eq:mnl23_2mixtures_poly'}. From the case $n=3$, we know there exists a $\lambda_4$-measure zero set $V_3$ such that for all $(\bm{a}_{2:3}, \bm{b}_{2:3}) \in Q_{MNL23}^4 \setminus V_3$, the system~\eqref{eq:mnl23_2mixtures_poly'} has a unique solution in $\mathbb{C}$ (counted with multiplicity) for $(\bm{x}_{1:3},\bm{y}_{1:3})$, given by $(\bm{x}_{1:3},\bm{y}_{1:3})=(\bm{a}_{1:3},\bm{b}_{1:3})$. So we can restrict the domain to be the set 
\[
N_{3,n} = \{(\bm{a}_{2:n},\bm{b}_{2:n}) \in Q^{2n-2}_{MNL23} : (\bm{a}_{2:3},\bm{b}_{2:3}) \notin V_3\}.
\]
To proceed, we determine $(x_i,y_i)$ for each $i \geq 4$. Consider the following part of the system~\eqref{eq:mnl23_2mixtures_poly'}, which contains $(x_i,y_i)$ as the only undetermined variables:
\begin{equation}\label{eq:btl_2_3_123i}
\begin{cases}
(1-\eta_{1,i})x_1y_1 + (p_2 - \eta_{1,i})y_1x_i + (p_1 - \eta_{1,i})x_1y_i - \eta_{1,i}x_iy_i = 0, \\
(1-\eta_{2,i})x_2y_2 + (p_2 - \eta_{2,i})y_2x_i + (p_1 - \eta_{2,i})x_2y_i - \eta_{2,i}x_iy_i = 0, \\
(1-\eta_{3,i})x_3y_3 + (p_2 - \eta_{3,i})y_3x_i + (p_1 - \eta_{3,i})x_3y_i - \eta_{3,i}x_iy_i = 0. \\
\end{cases}
\end{equation}

We first plug in $(\bm{x}_{1:3},\bm{y}_{1:3})=(\bm{a}_{1:3},\bm{b}_{1:3})$, and then eliminate the $x_iy_i$ terms in the last two equations using the first one. This yields a linear system of equations in $(x_i,y_i)$:
\[
\begin{cases}
\left(\eta_{2,i}-\eta_{1,i}\right)p_2 x_i +  \left(\eta_{2,i}-\eta_{1,i}\right)p_1y_i = \eta_{2,i}(\eta_{1,i}-1) + (1-\eta_{2,i})a_2b_2, \\
\left(\eta_{3,i}(p_2 - \eta_{2,i})b_2-\eta_{2,i}(p_2 - \eta_{3,i})b_3\right)x_i + \left(\eta_{3,i}(p_1 - \eta_{2,i})a_2 - \eta_{2,i}(p_1 - \eta_{3,i})a_3\right)y_i= \\
\quad \eta_{3,i}(\eta_{2,i}-1)a_2b_2-\eta_{2,i}(\eta_{3,i}-1)a_3b_3. 
\end{cases}
\]
As long as the coefficient matrix:
\begin{align*}
A =
\begin{bmatrix}
\left(\eta_{2,i}-\eta_{1,i}\right)p_2 & \left(\eta_{2,i}-\eta_{1,i}\right)p_1 \\
\eta_{3,i}(p_2 - \eta_{2,i})b_2-\eta_{2,i}(p_2 - \eta_{3,i})b_3 & \eta_{3,i}(p_1 - \eta_{2,i})a_2 - \eta_{2,i}(p_1 - \eta_{3,i})a_3
\end{bmatrix}.
\end{align*}
of this linear system has rank two, we have a unique solution for $(x_i,y_i)$ in $\mathbb{C}$ (counted with multiplicity). Note that we can plug in the values for $(\bm{a}_{2:n},\bm{b}_{2:n})=(2,3,0,\ldots,0;5,4,0,\ldots,0)$ to see that $\mathrm{rank}(A) = 2$. Therefore, there exists a measure zero subset $N_i$ of parameters such that $\mathrm{rank}(A)<2$. Since there are finitely many such subsets, we conclude that generic uniqueness holds for the cases where $n \geq 4$.
\end{proof}

\end{document}